\documentclass[11pt,reqno]{amsart}
\usepackage{graphicx,enumerate,amssymb,bbold}
\usepackage[notrig]{physics}
\usepackage[font=footnotesize,labelfont=small,position=top]{subcaption}
\usepackage[export]{adjustbox}
\usepackage[ruled,vlined]{algorithm2e}
\usepackage{booktabs,multirow}
\usepackage{pgfplots}
\usepackage{color}
\usepackage{colortbl}
\usepackage{tikz}
\usepackage{siunitx}
\usepackage{mathtools}
\usetikzlibrary{spy}
\usepackage[colorlinks=true, pdfstartview=FitV, linkcolor=blue, 
citecolor=blue, urlcolor=blue]{hyperref}
\definecolor{Mea_Deep_Color}{rgb}{0.3, 0.35, 0.45}
\definecolor{Mea_Cov_Color}{rgb}{0.5, 0.58, 0.7}
\pgfplotsset{width=7cm,compat=1.13}
\usepackage[left=2.5cm,right=2.5cm,top=2.5cm,bottom=2.5cm,includeheadfoot]{geometry}

\usepackage{siunitx}
\newcommand{\mum}[1]{\SI[multi-part-units = single]{#1}{\micro\meter}}

\numberwithin{equation}{section}
\numberwithin{figure}{section}
\theoremstyle{plain}
\newtheorem{theorem}{Theorem}[section]
\newtheorem{lemma}[theorem]{Lemma}
\newtheorem{proposition}[theorem]{Proposition}

\theoremstyle{definition}
\newtheorem{definition}[theorem]{Definition}

\newtheorem{remark}[theorem]{Remark}

\newcommand{\bitem}{\begin{itemize}}
\newcommand{\eitem}{\end{itemize}}
\newcommand{\mc}[1]{\mathcal{#1}}

\newcommand{\II}{\mathbb{I}}

\newcommand{\R}{\mathbb{R}}

\newcommand{\EE}{\mathbb{E}}

\newcommand{\Q}{\mathbb{Q}}

\newcommand{\bpm}{\begin{pmatrix}}
\newcommand{\epm}{\end{pmatrix}}
\newcommand{\bvm}{\begin{vmatrix}}
\newcommand{\evm}{\end{vmatrix}}
\newcommand{\bsm}{\left(\begin{smallmatrix}}
\newcommand{\esm}{\end{smallmatrix}\right)}
\newcommand{\T}{\top}

\newcommand{\ol}[1]{\overline{#1}}
\newcommand{\la}{\langle}
\newcommand{\ra}{\rangle}

\newcommand{\w}{\omega}

\newcommand{\gdw}{\Leftrightarrow}

\newcommand{\eins}{\mathbb{1}}

\DeclareMathSymbol{\mydiv}{\mathbin}{symbols}{"04}

\DeclareMathOperator{\Diag}{Diag}

\DeclareMathOperator{\dom}{dom}

\DeclareMathOperator{\argmin}{arg min}

\DeclareMathOperator{\ggrad}{grad}

\DeclareMathOperator{\Exp}{Exp}
\DeclareMathOperator{\expm}{expm}

\DeclareMathOperator{\DSC}{DSC}
\DeclareMathOperator{\MAE}{MAE}
\usepackage{lipsum}                     % Dummytext
\usepackage{xargs}                      % Use more than one optional parameter 
\usepackage{changepage}
\usepackage[colorinlistoftodos,prependcaption,textsize=tiny]{todonotes}
\usetikzlibrary{positioning}

\newcommand{\Depth}{4}
\newcommand{\Height}{6}
\newcommand{\Width}{2}

\newcommand{\yy}{1}

\newcommand{\xa}{5}
\newcommand{\ya}{-0.2}
\newcommand{\za}{1}
\newcommand{\xb}{6.4}

\newcommand*\circled[1]{\tikz[baseline=(char.base)]{
		\node[shape=circle,draw,inner sep=2pt] (char) {#1};}}

\usepackage{xcolor}
\definecolor{Layer_1}{rgb}{0.12156863, 0.46666667, 0.70588235}
\definecolor{Layer_2}{rgb}{0.68235294, 0.78039216, 0.90980392}
\definecolor{Layer_3}{rgb}{1.        , 0.49803922, 0.05490196}
\definecolor{Layer_4}{rgb}{0.17254902, 0.62745098, 0.17254902}
\definecolor{Layer_5}{rgb}{0.59607843, 0.8745098 , 0.54117647}
\definecolor{Layer_6}{rgb}{0.83921569, 0.15294118, 0.15686275}
\definecolor{Layer_7}{rgb}{0.58039216, 0.40392157, 0.74117647}
\definecolor{Layer_8}{rgb}{0.77254902, 0.69019608, 0.83529412}
\definecolor{Layer_9}{rgb}{0.54901961, 0.3372549 , 0.29411765}
\definecolor{Layer_10}{rgb}{0.89019608, 0.46666667, 0.76078431}
\definecolor{Layer_11}{rgb}{0.96862745, 0.71372549, 0.82352941}
\definecolor{Layer_12}{rgb}{0.49803922, 0.49803922, 0.49803922}
\definecolor{Layer_13}{rgb}{0.7372549 , 0.74117647, 0.13333333}
\definecolor{Layer_14}{rgb}{0.85882353, 0.85882353, 0.55294118}
\definecolor{Stein_line}{rgb}{0,0.5,0,1}
\definecolor{Euc_line}{rgb}{1,0.3,0,0}
\definecolor{Riemann}{rgb}{0.58,0.08,0}
\usetikzlibrary{shapes}
\usepackage{color}
\newcommand{\markerone}{\raisebox{0.5pt}{\tikz{\node[fill = 
green,circle,minimum size=1pt,scale=0.5, label = 
			{[below = 
				0.5ex]}] () {$1$};}}}

\newcommand{\markersix}{\raisebox{0.5pt}{\tikz{\node[fill = 
lime,circle,minimum size=1pt,scale=0.5, label = 
									{[below = 
										0.5ex]}] () {$6$};}}}
\newcommand{\markerfour}{\raisebox{0.5pt}{\tikz{\node[fill = 
yellow,circle,minimum size=1pt,scale=0.5, label = 
												{[below = 
													0.5ex]}] () {$4$};}}}
\newcommand{\markerfive}{\raisebox{0.5pt}{\tikz{\node[fill = 
cyan,circle,minimum size=1pt,scale=0.5, label = {[below = 0.5ex]}] () 
{$5$};}}}
\newcommand{\markerthree}{\raisebox{0.5pt}{\tikz{\node[fill = 
orange,circle,minimum size=1pt,scale=0.5, label = {[below = 0.5ex]}]
(){$3$};}}}

\DeclareMathOperator{\logm}{logm}

\title[Assignment Flow For Order-Constrained OCT Segmentation]{Assignment Flow For Order-Constrained OCT Segmentation}
\author[D.~Sitenko, B.~Boll, C.~Schn\"{o}rr]{Dmitrij Sitenko, Bastian Boll, 
Christoph Schn\"{o}rr}
\address[D.~Sitenko]{Image and Pattern Analysis Group (IPA) and Heidelberg 
Collaboratory for Image Processing (HCI), Heidelberg 
	University, Germany}
\email{dmitrij.sitenko@iwr.uni-heidelberg.de}
\address[B.~Boll]{Image and Pattern Analysis Group (IPA) and Heidelberg 
Collaboratory for Image Processing (HCI), Heidelberg 
	University, Germany}
\email{bastian.boll@iwr.uni-heidelberg.de}
\address[C.~Schn\"{o}rr]{Image and Pattern Analysis Group, Heidelberg 
University, Germany} 
\email{schnoerr@math.uni-heidelberg.de}
\urladdr{\url{https://ipa.math.uni-heidelberg.de}}
\date{} 

\newbox{\myorcidaffilbox}
\sbox{\myorcidaffilbox}{\large\includegraphics[height=1.7ex]{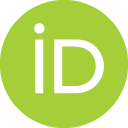}}

\newcommand{\orcidaffil}[1]{%
	\href{https://orcid.org/#1}{\usebox{\myorcidaffilbox}\,#1}}

\pdfoutput=1
\begin{document}

\begin{abstract}
At the present time Optical Coherence Tomography (OCT) is among the most 
commonly used non-invasive imaging methods for the acquisition 
of large volumetric scans of human retinal tissues and vasculature. The 
substantial increase of accessible highly resolved 3D samples at the optic 
nerve head and the macula is directly linked to medical advancements in early 
detection of eye diseases. To resolve decisive information from 
extracted OCT volumes and to make it applicable for further diagnostic 
analysis, the exact identification of retinal layer
thicknesses serves as an 
essential task be done for each patient separately. However, the 
manual examination of multiple OCT scans in a row is a  demanding and time 
consuming task, which results in a lengthy qualification process and is 
frequently confounded in the presence of tissue-dependent speckle noise. 
Therefore, the elaboration of automated segmentation models has become an 
important task in the field of medical image processing.
 
We propose a novel, purely data driven \textit{geometric approach to 
order-constrained $3$D OCT retinal cell layer segmentation} which takes as 
input data in any metric space and comes along with basic operations that can 
be effectively computed in parallel.  As opposed to many established retina detection 
methods, our presented formulation avoids the use of any shape prior and 
accomplishes the natural order of the retina in a purely geometric way, while 
maintaining the high level of accuracy.   
This makes the approach unbiased and hence suited for the detection of local 
anatomical changes of retinal tissue structure. To demonstrate robustness of 
the proposed approach, we compare 
two different choices of features on a data set of manually annotated $3$D OCT 
volumes of healthy human retina.
The quality of computed segmentations is compared to the state of the art in 
terms of mean absolute error and the Dice similarity coefficient. Visualizations 
of segmented volumes are also provided. The  
results indicate a great potential for applying our method to the classification of 
diseased retina and opens a new research direction regarding the joint segmentation of 
retinal cell layers and blood vessel structures.      
\end{abstract}

\maketitle
\tableofcontents

%%%
% !TEX root =  ../OCT-3D-AFlow.tex
%%%%%%%%%%%%%%%%%%%%%%%%%%%%%%%%%%%
\section{Introduction}\label{sec:Introduction}
\subsection{Overview, Motivation}
Optical Coherence Tomography (OCT) is a non-invasive imaging technique which 
measures the intensity response of back scattered light from millimeter 
penetration depth and provides information about retinal tissue structure in 
vivo to understand human eye functionalities, see Figure \ref{fig:Vis_Anotomy} 
for a more descriptive anatomical explanation. OCT devices record multiple 
two-dimensional B-scans in rapid succession and combine them to a single volume 
in a subsequent alignment step. We focus specifically on the application of OCT 
in ophthalmology for the aquisition of high-resolution volume scans of the human 
retina. Taking an OCT scan only takes minutes and can help detect symptoms of 
pathological conditions such as glaucoma, diabetes, multiple sclerosis or 
age-related macular degeneration. The relative ease of data acquisition also 
enables to use multiple OCT volume scans of a single patient over time to track 
the progression of a pathology or quantify the success of therapeutic treatment.
As a consequence of the technological progress in OCT imaging which was made 
over past few decades since its invention in 1991 \cite{Huang:1991aa}, more 
expertise for extraction of manual annotations is required which in the 
presence of big volumetric data sets is difficult to access.
 
To better leverage the availability of retinal OCT data in both clinical 
settings 
and empirical studies, much work is focused on the analysis of appropriate 
automatic feature extraction techniques. In particular, the access to such 
methods is especially crucial for achieving enhanced effectiveness of existing 
quantitative retinal multi cell layer 
segmentation approaches, and for increasing their clinical potential in real life 
applications, such as detection of fluid regions and reconstruction of vascular 
structures. The 
difficulty of 
these tasks lies in the challenging signal-to-noise ratio which is influenced by 
multiple factors including mechanical eye movement during registration and the 
presence of speckle noise.

In this paper, we extend the approach \cite{Astrom:2017ac} for labeling data on 
graphs to automatic cell layer segmentation in OCT data. After a feature 
extraction step, each voxel is labeled by smoothing local layer decisions and 
jointly leveraging a global geometric invariant -- the natural order of cell 
layers along the vertical axis of each B-scan, as shown in the third row of Figure 
\ref{fig:Volume_Vis_Introdunction}. We are able to produce high-quality 
segmentations 
of OCT volumes by using \emph{local} features as input for a purpose-built 
assignment flow variant which serves to incorporate global context in a 
controlled way. 
This is in contrast to common machine learning approaches which use essentially full B-scans as input. By incorporating global context into the feature extraction process, the latter methods are at increased risk of overfitting training data and potentially missinterpreting unseen pathologies.

Our segmentation approach is a smooth image labeling algorithm based on 
geometric numerical integration on an elementary statistical manifold. It can 
work with input data from any metric space, making it agnostic to the choice of 
feature extraction and suitable as plug-in replacement in diverse pipelines. In 
addition to respecting the natural order of cell layers, the proposed 
segmentation process has a high amount of built-in parallelism such that modern 
graphics acceleration hardware can easily be leveraged. 
We compare the effectiveness of our novel approach between a selection of input features ranging from traditional covariance descriptors to convolutional neural networks.
Figure \ref{fig:Volume_Vis_Introdunction} shows one specific example of an 
segmented OCT-volume with clearly visible difficulties of speckle noise and 
jagged patterns followed by an labeled (2D) cutout which is a typical result 
our 
novel approach.
\begin{figure}[t]
	\begin{tikzpicture}
	\node(D)[anchor=south west,inner 
	sep=0] at (18.6,7.6){}; 
	%\node(A)[anchor=south west] at (17,5){};
	\node(A)[anchor=south west,inner 
	sep=0] at (17,5) 
	{\includegraphics[width=5cm,height=5cm]{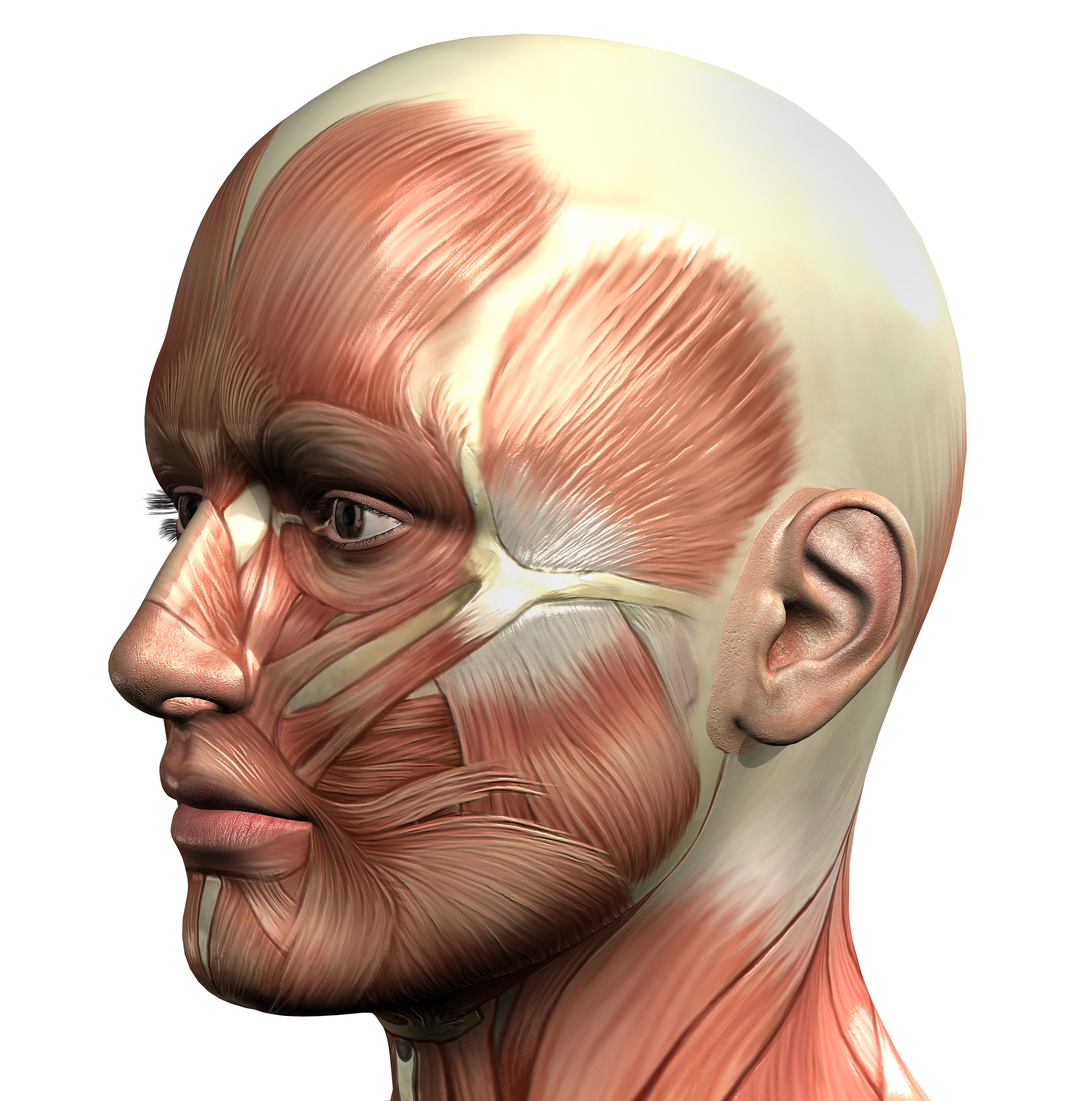}};
	\node(B)[inner sep=0pt,left = 4cm of A] 
	{\phantom{\includegraphics[width=.31\textwidth]{Graphics/4902.jpg}}};
	\node(C)[draw,very thick, blue!60, minimum size=0.5cm, at=(D)]{};
	\begin{scope}[thick,blue!40]
	\draw(C) -- (11,9.5);
	\draw(C) -- (11,5);
	\end{scope}
	\node[draw=blue!80,very thick,at=(B),inner sep=2pt,left = 4cm of 
	A]{\includegraphics[width=.30\textwidth]{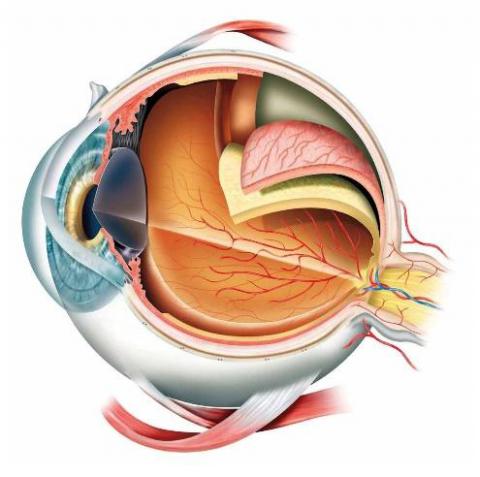}};
	\node(Fovea)[fill = green,circle,minimum size=1pt,scale=0.5, label = 
	{[below = 
		0.5ex]}] at (12.7,8) {$1$};
	\node(Sclera)[fill = red,circle,minimum size=1pt,scale=0.5, label = {[below 
	= 
		0.5ex]}] at (12.7,8.35) {$2$};
	\node(Vitreous humour)[fill = lime,circle,minimum size=1pt,scale=0.5, label 
	= 
	{[below = 0.5ex]}] at (12.7,7.65) {$6$};
	\node(Retina)[fill = yellow,circle,minimum size=1pt,scale=0.5, label = 
	{[below 
		= 0.5ex]}] at (12.7,9.05) {$4$};
	\node(Cornea)[fill = cyan,circle,minimum size=1pt,scale=0.5, label = 
	{[below = 
		0.5ex]}] at (12.7,9.4) {$5$};
	\node(Choroid)[fill = orange,circle,minimum size=1pt,scale=0.5, label = 
	{[below 
		= 0.5ex]}] at (12.7,8.7) {$3$};
	\path[->, draw,thick, color = green] (11.8, 7.4) -- (Fovea);
	\path[->, draw,thick, color = lime] (10.4, 7.3) -- (Vitreous humour);
	\path[->, draw,thick, color = red] (11.6, 8.7) -- (Sclera);
	\path[->, draw,thick, color = orange] (11.3,8.7) -- (Choroid);
	\path[->, draw,thick, color = yellow] (11,8) -- (Retina);
	\path[->, draw,thick, color = cyan] (8.35, 8.3) -- (Cornea);
	%\draw[red,ultra thick,rounded corners] (27.7,8) rectangle (28.2,7.3);
	\end{tikzpicture}
	\captionsetup{font=footnotesize}
	\caption{Schematic illustration designed by \cite{kjpargeter}
	of human eye functionality: The light 
	enters 
		the Cornea \protect\markerfive \hspace{2pt}though the vitreous humour 
		\protect\markersix \hspace{2pt}towards 
		retina 
		\protect\markerfour \hspace{2pt}and 
		choroid \protect\markerthree \hspace{2pt}which are located around the 
		fovea 
		\protect\markerone.}
	\label{fig:Vis_Anotomy}
\end{figure}
\subsection{Related Work}
Effective segmentation of OCT volumes is a very active area of research. Here, 
we briefly review the current state of the art approaches originating from 
the broad research fields of graphical models, variational methods and machine 
learning.
\subsubsection{Graphical Models}

The first mathematical access to the problem is provided by the theory of graphical 
models which transforms the segmentation task into an optimization problem with hard
pairwise interaction constraints between voxels. Starting with Li et. al. 
\cite{Li:2006aa} and Haeker \cite{Haeker:2007aa}, simultaneous retina layer 
detection attempts were made by finding an $s$-$t$ minimum graph cut.
Garvin et. al. \cite{Garvin:2008aa} further extended this approach with a shape prior 
modeling layer boundaries. The methods benefit from low computational 
complexity, but are lacking of robustness in the presence of speckle and 
therefore require additional preprocessing steps. 
Along this line of reasoning, B.J. Anthony et. al \cite{Anthony:2010aa} used a two 
stage segmentation process by applying anisotropic diffusion in a preprocessing
step and consequently segmenting outer retina layers using graphical models.
Similarly, Kafieh et. al. \cite{Kafie:2013aa} proposed to use specific 
distances based on diffusion maps which are computed by coarse graining the 
original graph. However, increased performance for noisy OCT data gained by 
regularizing in this way comes at the cost of introducing bias in the 
preprocessing step which in turn inpairs robustness in settings with medical 
pathologies. 

Motivated by \cite{Song:2012aa}, Dufour et. al. \cite{Dufour:2013aa} comes up 
with a circular shape prior for segmentation of 6 retinal layers by 
incorporating soft constraints which are more 
suitable for the robust detection of pathological retina structures. 
Chiu. et al. \cite{Chiu:2015aa} relies on a graphical model approach as a 
postprocessing step after applying a supervised kernel regression 
classification with features extracted according to \cite{Quellec:2010aa}.
Rathke et. al. \cite{Rathke2014} reduced the overall complexity by a 
parallelizable segmentation approach based on probabilistic graphical models 
with global low-rank shape prior representing interacting retina tissues 
surfaces. While the \textit{global} shape prior works well for non-pathological 
OCT data, it cannot be adapted to the broad range of variations caused by 
\textit{local pathological} structure resulting in a inherent limitation of 
this approach.

\begin{figure}[ht!] 
	{\raisebox{10mm}{\begin{subfigure}[t]{1.6em}
				\caption[singlelinecheck=off]{}
		\end{subfigure}}\ignorespaces}
	\begin{subfigure}[t]{0.47\textwidth}
		\centering
		\includegraphics[width=6.5cm,height=3.2cm]{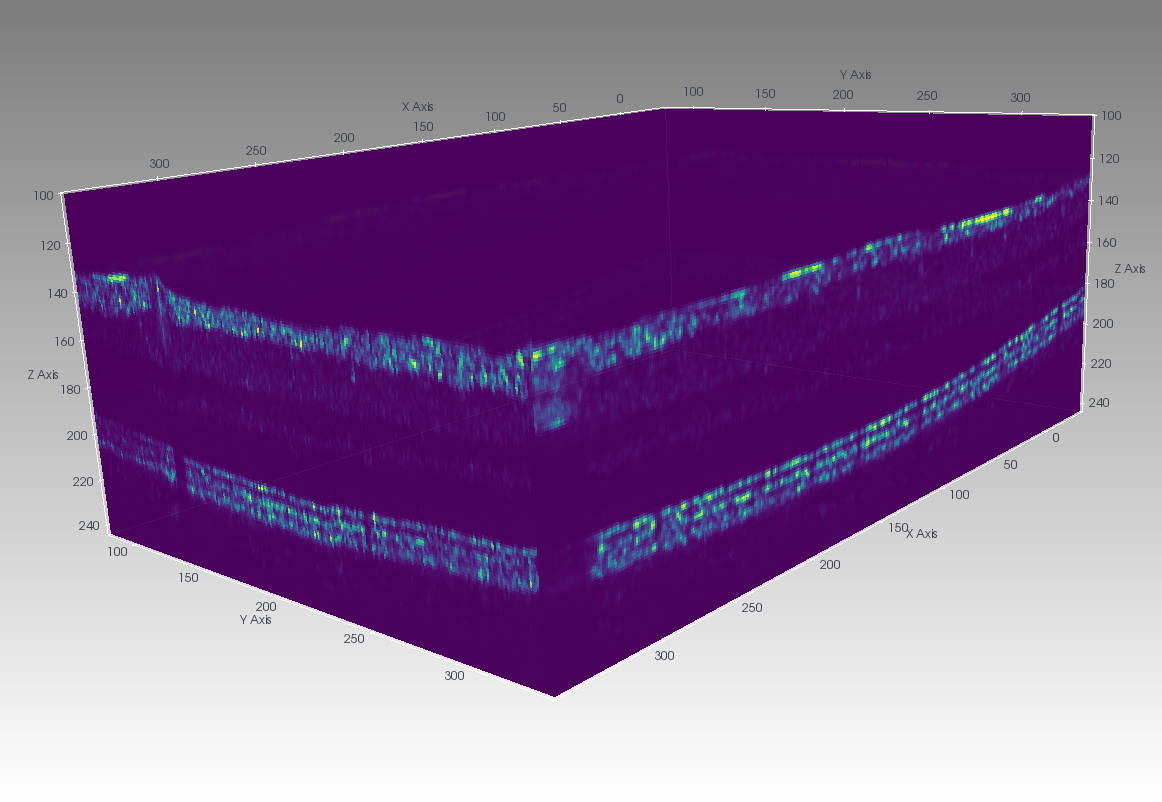}
	\end{subfigure}
	\begin{subfigure}[t]{0.47\textwidth}
		\centering
		\includegraphics[width=6.5cm,height=3.2cm]{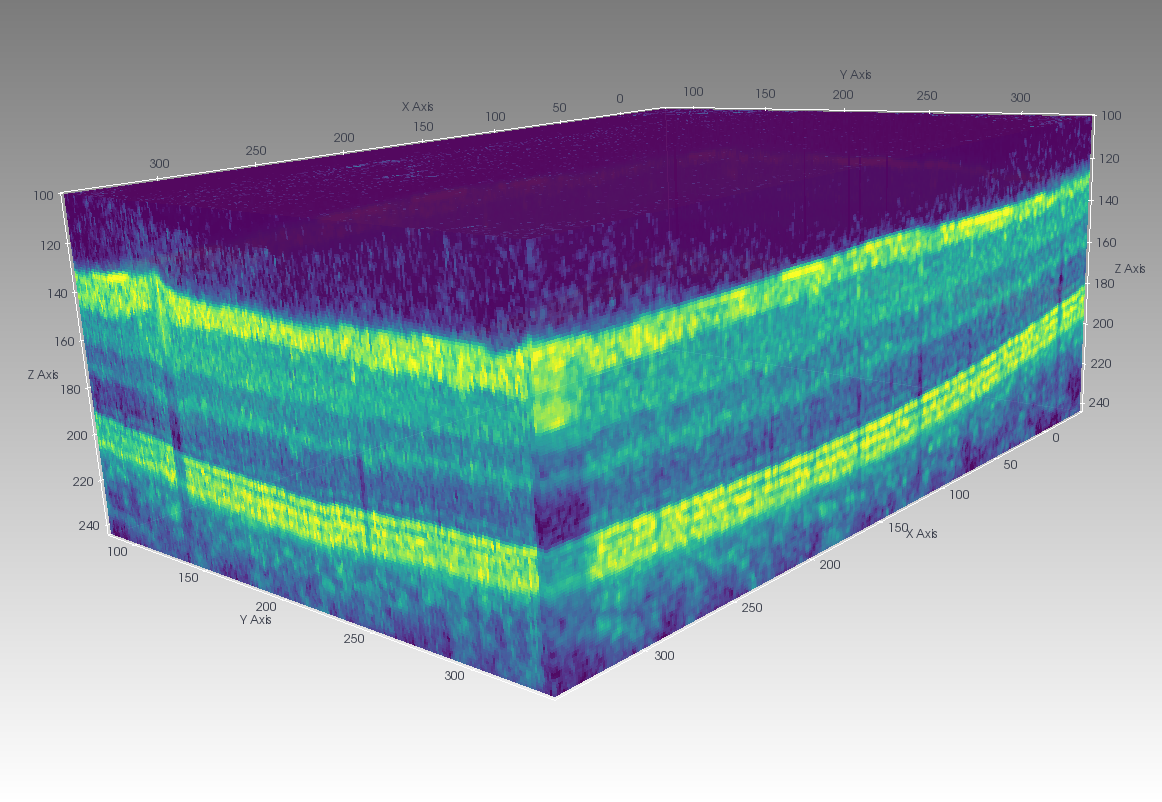}
	\end{subfigure}\\
\vspace{5pt}
	{\raisebox{10mm}{\begin{subfigure}[t]{1.6em}
				\caption[singlelinecheck=off]{}
		\end{subfigure}}\ignorespaces}
	\begin{subfigure}[t]{0.47\textwidth}
		\centering
		\includegraphics[width=6.5cm,height=3.2cm]{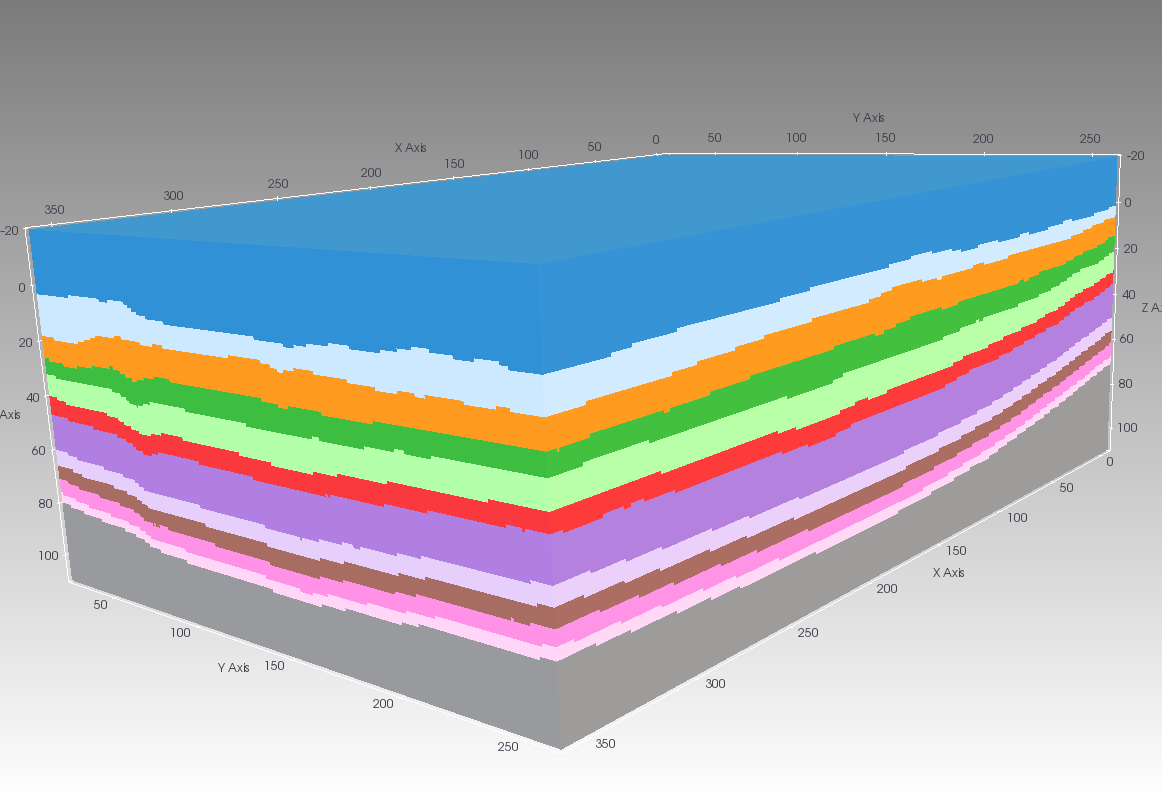}
		%\caption{Ground truth segmentation}	
	\end{subfigure}
	\begin{subfigure}[t]{0.47\textwidth}
		\centering
		\includegraphics[width=6.5cm,height=3.2cm]{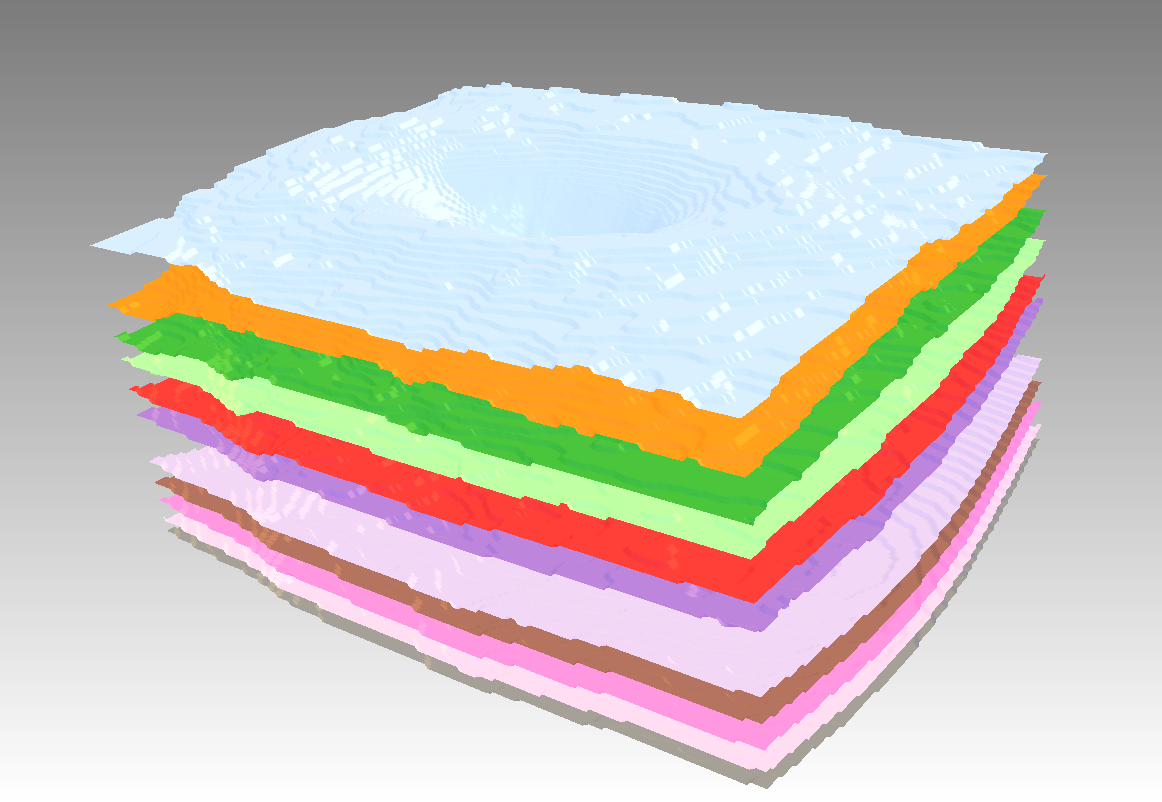}
	\end{subfigure}\\
\vspace{10pt}
{\raisebox{10mm}{\begin{subfigure}[t]{1.6em}
			\caption[singlelinecheck=off]{}
	\end{subfigure}}\ignorespaces}
		\begin{subfigure}[t]{0.945\textwidth}
		\centering
		\includegraphics[width=14.5cm,height=2cm]{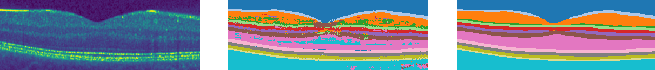}
	\end{subfigure}
	\captionsetup{font=footnotesize}
	\caption{\textbf{(a)}: $3$D OCT volume scan dimension $512 \times 512 
		\times 256$  of healthy 
		human retina with ambiguous locations of layer boundaries with 
		normalized view on the right. \textbf{(b)}: The 
		resulting segmentation of 11 layers displaying the order preserving 
		labeling of the proposed approach. Boundary surfaces 
		between different segmented cell layers are illustrated. \textbf{(c)}: Typical result 
		of 
		the proposed segmentation approach for a single 
		B-scan of healthy retina. \emph{Left}: raw OCT input data. 
		\emph{Middle}: 
		segmentation by locally selecting the label with maximum score for each voxel 
		after 
		feature extraction. \emph{Right}: segmentation by the proposed 
		assignment flow 
		approach using the same extracted features.}
	\label{fig:Volume_Vis_Introdunction}
\end{figure}

\subsubsection{Variational Methods}

Another category of layer detection methods focus on minimizing of energy 
functional to express the quantity of interest as the solution to an 
optimization problem. To this class of methods for retina detection level set 
approaches have proven to be particularly suitable by encoding each retina 
layer as the zero level sets of a certain functional. Yazdanpanah et. al. 
\cite{Yazdanpanah:2011aa} introduces a level set method for minimizing an 
active contour functional supported by a multiphase Chan Vese model 
\cite{Chan:2001} as circular shape prior, to avoid limitations of hard 
constraints as opposed to graphical model proposed by \cite{Garvin:2008aa}. 
Duan et. al. \cite{Duan:2015aa} suggests the approach to model layer boundaries 
with a mixture of Mumford Shah and Vese and Osher functionals by first 
preprocessing the data in the Fourier domain. A capable level set approach for 
joint segmentation of pathological retina tissues was reported in 
\cite{Novosel:2017aa}.
However, due to the involved hierarchical optimization, their method is computationally expensive. 
One common downside of the above algorithms are their inherent limitations to only 
include 
local notions of layer ordering, making their 
extension to cases with pathologically caused retina degeneracy a difficult 
task.

\subsubsection{Machine Learning}

Much recent work has focused on the use of deep learning to address the task of cell layer segmentation in a purely data driven way. 
The U-net architecture \cite{ronneberger2015unet} has proven influential in this domain because of its good predictive performance in settings with limited availability of training data. Multiple modifications of U-net have been proposed to specifically increase its performance in OCT applications \cite{roy2017relaynet,liu2018semi}. The common methods largely rely on convolutional neural networks to predict layer segmentations for individual B-scans which are subsequently combined to full volumes.
These methods have also been used as part of a two-stage pipeline where additional prior knowledge such as local regularity and global order of cell layers along a spatial axis is incorporated through graph-based methods \cite{fang2017automatic} or a second machine learning component \cite{he19order}. However, because global context is already used in feature extraction, the risk of overfitting remains and unseen pathologies may result in unpredictable behavior.

\subsection{Contribution, Organization}

We propose a geometric assignment approach to retinal layer segmentation. By leveraging a continuous characterization of layer ordering, our method is able to simultaneously perform local regularization and incorporate the global topological ordering constraint in a \textit{single smooth} labeling process. 
The segmentation is computed from a distance matrix containing pairwise distances between data for each voxel and prototypical data for each layer in some feature space. This highlights the ability to extract features from raw OCT data in a variety of different ways and to use the proposed segmentation as a plug-in replacement for other graph-based methods.

As a result of the proposed method, it becomes possible to compute high-quality cell layer segmentations of OCT volumes by using only local features for each voxel. This is in contrast to competing deep learning approaches which commonly use information from an entire B-scan as input. In addition, the exclusive use of local features combats bias introduced through limited data availability in training and makes incorporation of three-dimensional information easily possible without limiting runtime scalability. 
To demonstrate this, we implement two feature extraction approaches. The first is based on identifying each datum with a covariance descriptor and finding prototypical descriptors as cluster centers. For each voxel, Riemannian distances to the prototypical descriptors are used as input for subsequent segmentation.
The second is based on training a convolutional neural network to classify small voxel patches of raw OCT data. Predicted class scores for each voxel are subsequently used as input for the proposed segmentation method.

The final pipeline enables robust cell layer segmentation for raw OCT volumes 
at scale, labeling an entire OCT volume in the time frame between 30 seconds 
and several minutes on a single GPU and in general leads to increased  
performance in the case of more informative features. This is without using any 
prior knowledge other than local regularity and order of cell layers. In 
particular, no global shape prior is used thus making our proposed approach suited for 
retina detection in OCT volumes with observable pathological patterns.

Our paper considerably elaborates the conference version \cite{Sitenko:2020aa} and is organized as follows. The assignment flow approach is summarized 
in Section \ref{sec:Assignment-Flow} and extended in Section 
\ref{sec:Ordered-Segmentation} in order to take into account the order of 
layers as a global constraint. In Section \ref{sec:OCT-Segmentation}, we 
consider the Riemannian manifold $\mc{P}_{d}$ of positive definite matrices as 
a suitable feature space for local OCT data descriptors. Various Riemannian 
metrics are discussed with regard to computational efficiency of clustering.
The resulting features are subsequently compared to local features extracted by 
a convolutional network in Section \ref{sec:Experimental-Results}. Performance 
measures for OCT segmentation will be reported for our novel approach and for two 
other state-of-the-art methods with available standalone software, that were evaluated in detail 
 as summarized in Section 
\ref{sec:Experimental-Results}. In Section \ref{sec:discussion}, we shortly 
discuss the access to appropriate ground truth data and the impact of feature locality underlying our 
approach.
% !TEX root =  ../OCT-3D-AFlow.tex
%%%%%%%%%%%%%%%%%%%%%%%%%%%%%%%%%%%

\section{Assignment Flow}\label{sec:Assignment-Flow}

We summarize the assignment flow approach introduced by \cite{Astrom:2017ac} and refer to the recent survey \cite{Schnorr:2019aa} for more background and a review of recent related work.

\subsection{Assignment Manifold}\label{sec:assignmentManifold}
Let $(\mc{F},d_{\mc{F}})$ be a metric space and
\begin{subequations}\label{eq:def-mc-F-Fast}
\begin{align}\label{eq:def-mcF-n}
\mc{F}_{n} 
&= \{f_{i} \in \mc{F} \colon i \in \mc{I}\},\qquad |\mc{I}|=n
\intertext{
given data. Assume that a predefined set of prototypes
} \label{eq:def-mcF-ast}
\mc{F}_{\ast} 
&= \{f^{\ast}_{j} \in \mc{F} \colon j \in \mc{J}\},\qquad |\mc{J}|=c
\end{align}
\end{subequations}
is given. \textit{Data labeling} denotes the assignments
\begin{equation}\label{eq:fj-fi}
j \to i,\qquad f_{j}^{\ast} \to f_{i}
\end{equation}
of a single prototype $f_{j}^{\ast} \in \mc{F}_{\ast}$ to each data point $f_{i} \in \mc{F}_{n}$.
The set $\mc{I}$ is assumed to form the vertex set of an  undirected graph $\mc{G}=(\mc{I},\mc{E})$ which defines a relation $\mc{E} \subset \mc{I} \times \mc{I}$ and neighborhoods
\begin{equation}\label{eq:def-Ni}
\mc{N}_{i} = \{k \in \mc{I} \colon ik \in \mc{E}\} \cup \{i\},
\end{equation}
where $ik$ is a shorthand for the unordered pair (edge) $(i,k)=(k,i)$. We require these neighborhoods to satisfy the symmetry relation 
\begin{equation}\label{eq:Ni-Nk}
k \in \mc{N}_{i} \quad\gdw\quad i \in \mc{N}_{k},\qquad\forall i,k \in \mc{I}.
\end{equation}

The assignments (labeling) \eqref{eq:fj-fi} are represented by matrices in the set
\begin{equation}\label{eq:def-W-ast}
\mc{W}_{\ast} = \{W \in \{0,1\}^{n \times c} \colon W\eins_{c}=\eins_{n}\}
\end{equation}
with unit vectors $W_{i},\,i \in \mc{I}$, called \textit{assignment vectors}, as row vectors. These assignment vectors are computed by numerically integrating the assignment flow below \eqref{eq:assignment-flow} in the following  geometric setting. The integrality constraint of \eqref{eq:def-W-ast} is relaxed and vectors
\begin{equation}\label{eq:def-Wi}
W_{i} = (W_{i1},\dotsc,W_{ic})^{\T} \in \mc{S},\quad i \in \mc{I},
\end{equation}
that we still call \textit{assignment vectors},
are considered
on the elementary Riemannian manifold
\begin{equation}\label{eq:def-S}
(\mc{S},g),\qquad
\mc{S} = \{p \in \Delta_{c} \colon p > 0\}
\end{equation}
with the probability simplex $\Delta_{c}=\big\{p \in \R_{+}^{c}\colon \sum_{i=1}^{c} = \la\eins,p\ra=1\big\}$, the barycenter
\begin{equation}\label{eq:barycenter-S}
\eins_{\mc{S}} = \frac{1}{c}\eins_{c} \in \mc{S},
\qquad(\text{barycenter})
\end{equation}
tangent space
\begin{equation}\label{eq:def-T0}
T_{0}
= \{v \in \R^{c} \colon \la\eins_{c},v \ra=0\}
\end{equation}
and tangent bundle $T\mc{S} = \mc{S} \times T_{0}$,
the orthogonal projection
\begin{equation}\label{eq:def-Pi0}
\Pi_{0} \colon \R^{c} \to T_{0},\qquad
\Pi_{0} = I - \eins_{\mc{S}}\eins^{\T}
\end{equation}
and the Fisher-Rao metric
\begin{equation}\label{schnoerr-eq:FR-metric-S}
g_{p}(u,v) = \sum_{j \in \mc{J}} \frac{u^{j} {v}^{j}}{p^{j}},\quad p \in \mc{S},\quad
u,v \in T_{0}.
\end{equation}
Based on the linear map
\begin{equation}\label{eq:def-Rp}
R_{p} \colon \R^{c} \to T_{0},\qquad
R_{p} = \Diag(p)-p p^{\T},\qquad p \in \mc{S}
\end{equation}
that satisfies
\begin{equation}\label{eq:Rp-Pi0}
R_{p} = R_{p} \Pi_{0} = \Pi_{0} R_{p},
\end{equation}
exponential maps and their inverses are defined by
\begin{subequations}\label{eq:schnoerr-eq:exp-maps}
\begin{align}
\Exp &\colon \mc{S} \times T_{0} \to \mc{S}, &
(p,v) &\mapsto
\Exp_{p}(v) = \frac{p e^{\frac{v}{p}}}{\la p,e^{\frac{v}{p}}\ra},
\label{eq:Exp0} \\ \label{eq:IExp0}
\Exp_{p}^{-1} &\colon \mc{S} \to T_{0}, &
q &\mapsto  \Exp_{p}^{-1}(q) = R_{p} \log\frac{q}{p},
\\
\exp_{p} &\colon T_{0} \to \mc{S}, &
\exp_{p} &= \Exp_{p} \circ R_{p},
\label{schnoerr-eq:def-exp-p} \\ \label{schnoerr-eq:def-exp-p-inverse}
\exp_{p}^{-1} &\colon \mc{S} \to T_{0}, &
\exp_{p}^{-1}(q) &= \Pi_{0}\log\frac{q}{p}.
\end{align}
\end{subequations}
Applying the map $\exp_{p}$ to a vector in $\R^{c} = T_{0} \oplus \R\eins$ does not depend on the constant component of the argument, due to \eqref{eq:Rp-Pi0}.

\begin{remark}\label{rem:Exp-map}
The map $\Exp$ corresponds to the e-connection of information geometry \cite{Amari:2000aa}, rather than to the exponential map of the Riemannian connection. Accordingly, the affine geodesics \eqref{eq:Exp0} are not length-minimizing. But they provide a close approximation \cite[Prop.~3]{Astrom:2017ac} and are more convenient for numerical computations.
\end{remark}

The \textit{assignment manifold} is defined as
\begin{equation}\label{schnoerr-eq:def-mcW}
(\mc{W},g),\qquad \mc{W} = \mc{S} \times\dotsb\times \mc{S}.\qquad (n = |\mc{I}|\;\text{factors})
\end{equation}
We identify $\mc{W}$ with the embedding into $\R^{n\times c}$ 
\begin{equation}\label{eq:mcW-matrix-embed}
  \mc{W} = \{ W \in \R^{n\times c} \colon W\eins_c = \eins_n\text{ and } W_{ij} > 0 \text{ for all } i\in[n], j\in [c]\}.
\end{equation}
Thus, points $W \in \mc{W}$ are row-stochastic matrices $W \in \R^{n \times c}$ with row vectors $W_{i} \in \mc{S},\; i \in \mc{I}$ that represent the assignments \eqref{eq:fj-fi} for every $i \in \mc{I}$. We set
\begin{equation}\label{schnoerr-eq:TmcW}
\mc{T}_{0} := T_{0} \times\dotsb\times T_{0}
\qquad (n = |\mc{I}|\;\text{factors}).
\end{equation}
Due to \eqref{eq:mcW-matrix-embed}, the tangent space $\mc{T}_0$ can be identified with
\begin{equation}\label{eq:mcT-matrix-embed}
  \mc{T}_0 = \{ V \in \R^{n\times c} \colon V\eins_c = 0\}.
\end{equation}
Thus, $V_i \in T_{0}$ for all row vectors of $V \in \R^{n \times c}$ and $i \in \mc{I}$. All mappings defined above factorize in a natural way and apply row-wise, e.g.~$\Exp_{W} = (\Exp_{W_{1}},\dotsc,\Exp_{W_{n}})$ etc.

\subsection{Assignment Flow}

Based on \eqref{eq:def-mcF-n} and \eqref{eq:def-mcF-ast}, the distance vector field
\begin{equation}\label{eq:def-distance-vector}
D_{\mc{F};i} = \big(d_{\mc{F}}(f_{i},f_{1}^{\ast}),\dotsc,d_{\mc{F}}(f_{i},f_{c}^{\ast})\big)^{\T},\qquad i \in \mc{I}
\end{equation}
is well-defined. These vectors are collected as row vectors of the \textit{distance matrix}
\begin{equation}\label{eq:def-distance-matrix}
D_{\mc{F}} \in S_{+}^{n},
\end{equation}
where $S_{+}^{n}$ denotes the set of symmetric and entrywise nonnegative matrices.

\begin{remark}\label{rem:Feature_Choice}
In this paper, we build upon two different types of features to determine 
vectors \eqref{eq:def-distance-vector} which are serving as input before 
mapping the assembled matrix \eqref{eq:def-distance-matrix} onto the assignment 
manifold as explained below. Hereby, the first class of features access our 
model by calculating distance to prototypes \eqref{eq:def-mc-F-Fast} with 
metric introduced in section \eqref{sec:Rmeans} while the second feature class  
directly possess the form of \eqref{eq:def-distance-matrix} as argued in 
section \eqref{sec:experiments_cnn}.
\end{remark}

The \textit{likelihood map} and the \textit{likelihood vectors}, respectively, are defined as
\begin{equation}\label{schnoerr-eq:def-Li}
L_{i} \colon \mc{S} \to \mc{S},\qquad
L_{i}(W_{i})
= \exp_{W_{i}}\Big(-\frac{1}{\rho}D_{\mc{F};i}\Big)
= \frac{W_{i} e^{-\frac{1}{\rho} D_{\mc{F};i}}}{\la W_{i},e^{-\frac{1}{\rho} D_{\mc{F};i}} \ra},\qquad i \in \mc{I},
\end{equation}
where the scaling parameter $\rho > 0$ is used for normalizing the a-prior unknown scale of the components of $D_{\mc{F};i}$ that depends on the specific application at hand.

A key component of the assignment flow is the interaction of the likelihood vectors through \textit{geometric} averaging within the local neighborhoods \eqref{eq:def-Ni}. Specifically, using  weights
\begin{equation}\label{eq:weights-Omega-i}
  \omega_{ik} > 0\quad \text{for all}\; k \in \mc{N}_{i},\;i \in \mc{I}\quad\text{with}\quad \sum_{k \in \mc{N}_{i}} w_{ik}=1,
\end{equation}
the \textit{similarity map} and the \textit{similarity vectors}, respectively, are defined as
\begin{equation}\label{eq:def-Si}
S_{i} \colon \mc{W} \to \mc{S},\qquad
S_{i}(W) = \Exp_{W_{i}}\Big(\sum_{k \in \mc{N}_{i}} w_{ik} \Exp_{W_{i}}^{-1}\big(L_{k}(W_{k})\big)\Big),\qquad i \in \mc{I}.
\end{equation}
If $\Exp_{W_{i}}$ were the exponential map of the Riemannian (Levi-Civita) connection, then the argument inside the brackets of the right-hand side would just be the negative Riemannian gradient with respect to $W_{i}$ of center of mass objective function comprising the points $L_{k},\,k \in \mc{N}_{i}$, i.e.~the weighted sum of the squared Riemannian distances between $W_{i}$ and  $L_{k}$   \cite[Lemma 6.9.4]{Jost:2017aa}. In view of Remark \ref{rem:Exp-map}, this interpretation is only approximately true mathematically, but still correct informally: $S_{i}(W)$ moves $W_{i}$ towards the geometric mean of the likelihood vectors $L_{k},\,k \in \mc{N}_{i}$. Since $\Exp_{W_{i}}(0)=W_{i}$, this mean precisely is $W_{i}$ if the aforementioned gradient vanishes.

The \textit{assignment flow} is induced on the assignment manifold $\mc{W}$ by the locally coupled system of nonlinear ODEs
\begin{subequations}\label{eq:assignment-flow}
\begin{align}
\dot W &= R_{W}S(W),\qquad W(0)=\eins_{\mc{W}},
\label{eq:assignment-flow-a}
\\
\label{eq:assignment-flow-b}
\dot W_{i} &= R_{W_{i}} S_{i}(W),\qquad W_{i}(0)=\eins_{\mc{S}},\quad i \in \mc{I},
\end{align}
\end{subequations}
where $\eins_{\mc{W}} \in \mc{W}$ denotes the barycenter of the assignment manifold \eqref{schnoerr-eq:def-mcW}. The solution $W(t)\in\mc{W}$ is numerically computed by geometric integration \cite{Zeilmann:2020aa} and determines a labeling $W(T) \in \mc{W}_{\ast}$ for sufficiently large $T$ after a trivial rounding operation. Convergence and stability of the assignment flow have  been studied by \cite{Zern:2020aa}.
% !TEX root =  ../OCT-3D-AFlow.tex
%%%%%%%%%%%%%%%%%%%%%%%%%%%%%%%%%%%

\section{OCT Data Representation by Covariance Desciptors}\label{sec:OCT-Segmentation}

In this section, we work out the basic geometric notation for representation of 
OCT data by means of covariance descriptors \cite{Tuzel:2006aa}.
Specifically, the metric data space $(\mc{F},d_{\mc{F}})$ underlying 
\eqref{eq:def-mc-F-Fast} will be identified with the Riemannian manifold 
$(\mc{P}_{d},d_{g})$ of positive definite matrices of dimension $d \times d$, 
with Riemannian metric $g$ and Riemannian distance $d_{g}$ as specified in 
section \ref{sec:Experimental-Results}. In particular regarding the computation 
of corresponding prototypes \eqref{eq:def-mcF-ast}, an important aspect 
concerns the trade-off between respecting the Riemannian distance $d_{g}$ of 
the matrix manifold $\mc{P}_{d}$ and approximating surrogate distance 
functions, that enable to compute more efficiently Riemannian means of 
covariance descriptors while adopting their natural geometry. We review and 
discuss various choices in Section \ref{sec:Rmeans} after reviewing few 
required concepts of Riemannian geometry in Section \ref{sec:Pd}.

\subsection{The Manifold $\mc{P}_{d}$}
\label{sec:Pd}

We collect few concepts related to data $p \in \mc{M}$ taking values on a general Riemannian manifold $(\mc{M},g)$ with Riemannian metric $g$; see, e.g., \cite{Lee:2013aa,Jost:2017aa} for background reading. Then we apply these concepts to the specific manifold $(\mc{P}_{d},g)$ and the corresponding distance $d_{g}$, keeping the symbol $g$ for the metric for simplicity. We refer to, e.g., \cite{Bhatia:2007aa,Bhatia:2013aa,Pennec:2004aa,Moakher:2006aa} for further reading and to the references in Section \ref{sec:Rmeans}.

Let $\gamma\colon [0,1]\to\mc{M}$ a smooth curve connecting two points $p = \gamma(0)$ and $q = \gamma(1)$. The \textit{Riemannian distance} between $p$ and $q$ is given by 
\begin{subequations}
\begin{align}\label{eq:R_Distance}
 d_g(p,q) &= \min_{\gamma\colon \gamma(0) = p,\gamma(1) = q} L(\gamma)
\intertext{with}
L(\gamma) &= \int_{0}^{1}\|\dot\gamma(t)\|_{\gamma(t)}\dd{t}
= \int_{0}^1 
\sqrt{g_{\gamma(t)}\big(\dot{\gamma}(t),\dot{\gamma}(t)\big)} \dd{t}.
\end{align}
\end{subequations}
Assume the minimum of the right-hand side of \eqref{eq:R_Distance} is attained at $\ol{\gamma}$. Then the \textit{exponential map} at $p$ is defined on some neighborhood $V_{p} \subseteq T_{p}\mc{M}$ of $0$ in the tangent space to $\mc{M}$ at $p$ by
\begin{equation}\label{eq:def-exp-M}
\exp_{p} \colon V_{p} \supseteq T_{p}\mc{M} \;\to\; U_{p} \subseteq\mc{M},\qquad
v\mapsto \exp_{p}(v) := \ol{\gamma}(1).
\end{equation}
This mapping is a diffeomorphism of $V_{p}$ and its inverse $\exp_{p}^{-1} 
\colon U_{p} \to V_{p}$ exists on a corresponding open neighborhood $U_{p}$. 
Let $\mc{X}(\mc{M})$ denote the set of all smooth vector fields on $\mc{M}$, 
i.e.~$X \in \mc{X}(\mc{M})$ evaluates to a tangent vector $X_{p} \in 
T_{p}\mc{M}$ smoothly depending on $p$. The set of all smooth covector fields 
(one-forms) is denoted by $\mc{X}^{\ast}(\mc{M})$, and $df(X)$ denotes the 
action of the differential $df \in \mc{X}^{\ast}(\mc{M})$ of a smooth function 
$f \colon \mc{M}\to\R$ on a vector field $X$. The \textit{Riemannian gradient} 
of $f$ is the vector field $\ggrad f \in \mc{X}(\mc{M})$ defined by
\begin{equation}\label{Riemann_Gradient_Def}
g(\ggrad f,X) = df(X) = X f,\qquad
\forall X \in \mc{X}(\mc{M}).
\end{equation}

We now focus on the following problem: Given a set of points $\{p_{i}\}_{i \in [N]}\subset \mc{M}$, compute the \textit{weighted Riemannian mean} as minimizer of the objective function
\begin{equation}\label{eq:w-R-mean}
\ol{p} = \arg\min_{q \in \mc{M}}J(q),\qquad
J(q) = \sum_{i \in [N]} \w_{i}d^{2}_{g}(q,p_{i}),\qquad
\w_{i}>0,\;\forall i,\quad \sum_{i\in [N]}\w_{i}=1.
\end{equation}
The Riemannian gradient of this objective function is given by \cite[Lemma 6.9.4]{Jost:2017aa}
\begin{equation}
\ggrad J(p) = -\sum_{i\in [N]}\w_{i} \exp_{p}^{-1}(p_{i}).
\end{equation}
Hence the Riemannian mean $\ol{p}$ is determined by the optimality condition
\begin{equation}\label{eq:optCond-Rmean-general}
\sum_{i\in [N]} \w_{i}\exp^{-1}_{\ol{p}}(p_{i}) = 0.
\end{equation}
A basic numerical method for computing $\ol{p}$ is the fixed point iteration
\begin{equation}\label{eq:Rmean-FPiteration-general}
q_{(t+1)} = \exp_{q_{(t)}}\Big(\sum_{i\in [N]} \w_{i}\exp^{-1}_{q_{(t)}}(p_{i})\Big),\qquad
t=1,2,\dotsc
\end{equation}
that may converge for a suitable initialization $q_{(0)}$ to $\ol{p}$.

We now focus on the specific manifold $(\mc{P}_{d},g)$ 
\begin{equation}\label{eq:def-mcPd}
\mc{P}_{d} = \{S\in\R^{d\times d}\colon S=S^{\T},\, S\;\text{is positive definite}\}
\end{equation}
equipped with the Riemannian metric
\begin{equation}
\label{AIRM}
g_{S}(U,V) = \tr(S^{-1} U S^{-1} V),\qquad U,V\in T_{S}\mc{P}_{d} = \{S \in \R^{d\times d}\colon S^{\T}=S\}.
\end{equation}
The Riemannian distance \eqref{eq:R_Distance} is given by
\begin{equation}\label{eq:def-g-mcPd}
d_{\mc{P}_{d}}(S,T) = \Big(\sum_{i \in [d]}\big(\log\lambda_{i}(S,T)\big)^{2}\Big)^{1/2},
\end{equation}
whereas the exponential map \eqref{eq:def-exp-M} reads
\begin{equation}\label{eq:def-exp-mcPd}
\exp_{S}(U) = S^{\frac{1}{2}}\expm(S^{-\frac{1}{2}}US^{-\frac{1}{2}})S^{\frac{1}{2}},
\end{equation}
and $\expm(\cdot)$ denotes the matrix exponential. 
Finally, given a smooth objective function $J \colon \mc{P}_{d} \to \R$, the Riemannian gradient is given by
\begin{equation}\label{eq:R-grad-Pd}
\ggrad J(S) = S\big(\partial J(S)\big)S \in T_{S}\mc{P}_{d},
\end{equation}
where the symmetric matrix $\partial J(S)$ denotes the Euclidean gradient of $J$ at $S$. Since $\mc{P}_{d}$ is a simply connected, complete and nonpositively curved Riemannian manifold \cite[Section 10]{Bridson:1999aa}, the exponential map \eqref{eq:def-exp-mcPd} is globally defined and bijective, and the Riemannian mean always exists and is uniquely defined as minimizer of the objective function \eqref{eq:w-R-mean}, after substituting the Riemannian distance \eqref{eq:def-g-mcPd}.

\subsection{Computing Prototypical Covariance Descriptors}
\label{sec:Rmeans}
In this section, we focus on the computational differential geometric framework 
required for extraction of prototypes \eqref{eq:def-mcF-ast} as Riemannian 
means from a set of covariance descriptors assembled from OCT data. Application 
details are reported in Section \ref{sec:Experimental-Results}. Particularly 
with regard to more efficient handling present volumetric data and to reduce 
the computational costs, a surrogate metrics and distances are reviewed in 
Sections \ref{sec:Log_Euclid} and \ref{sec:S-distance}. Their qualitative 
comparison is reported in Section \ref{sec:Experimental-Results}.

\subsubsection{Computing Riemannian Means}\label{sec:Rmeans-approximation} 
Given a set of covariance descriptors 
\begin{equation}\label{eq:def-mcSN}
\mc{S}_{N} = \{(S_1,\w_{1}),\dots , (S_{N},\w_{N})\} \subset \mathcal{P}_d
\end{equation}
together with positive weights $\w_{i}$, we next focus on the solution of the  
problem \eqref{eq:w-R-mean} for specific geometry \eqref{eq:def-mcPd},
\begin{equation}\label{eq:def-Rmean}
\ol{S} = \arg\min_{S \in \mc{P}_{d}} J(S;\mc{S}_{N}),\qquad
J(S;\mc{S}_{N}) = \sum_{i \in [N]} \w_{i}d_{\mc{P}_{d}}^{2}(S,S_{i}),
\end{equation}
with the distance $d_{\mc{P}_{d}}$ given by \eqref{eq:def-g-mcPd}. From \eqref{eq:def-exp-mcPd}, we deduce
\begin{equation}
U = \exp_{S}^{-1} \circ \exp_{S}(U)
= S^{\frac{1}{2}} \logm\big(S^{-\frac{1}{2}}\exp_{S}(U)S^{-\frac{1}{2}}\big) S^{\frac{1}{2}}
\end{equation}
with the matrix logarithm $\logm=\expm^{-1}$ 
\cite[Section 11]{Higham:2008aa}. As a result, optimality condition \eqref{eq:optCond-Rmean-general} reads
\begin{equation}\label{eq:OC-Pd-mean}
\sum_{i \in [N]} \w_{i} \ol{S}^{\frac{1}{2}} \logm\big(\ol{S}^{-\frac{1}{2}}S_{i} \ol{S}^{-\frac{1}{2}}\big) \ol{S}^{\frac{1}{2}} = 0.
\end{equation}
Applying the corresponding basic fixed iteration 
\eqref{eq:Rmean-FPiteration-general} has two drawbacks, however 
\cite{Congedo:2015aa}: Convergence is not theoretically guaranteed and if the 
iteration converges, than at a linear rate only. Since each iterative step 
requires nontrivial numerical matrix decomposition that has to be applied 
multiple times to every voxel (vertex) of a 3D gridgraph, this results in an 
overall quite expensive approach, in particular when larger data sets are 
involved as is the case for highly resolved 3D OCT volumetric scans.

The following variant proposed by \cite{Bini:2013aa} is guaranteed to converge 
at a \textit{quadratic} rate assuming the matrices $\{S_1,\dots,S_N \}$ to 
pairwise commute. Using the 
parametrization 
\begin{equation}
S = L L^{\T}
\end{equation}
corresponding to the Cholesky decomposition replacing 
the map of fixed point iteration \eqref{eq:Rmean-FPiteration-general} with its 
linearization leads to the following fixed point iteration
\begin{equation}
F_{\tau}(L; \mc{S}_{N}) = L L^{\T}- \tau \sum_{i\in [N]} \w_{i} L^{\T}  \logm(L^{-\T}S_i^{-1}L^{-1})L ,\qquad 
\tau > 0, 
\end{equation}
with damping parameter $\tau$. 
Comparing to \eqref{eq:OC-Pd-mean} shows that the basic idea is to compute the Riemannian mean $\ol{S}$ as fixed point of the iteration
\begin{equation}\label{eq:Pd-mean-adaptive}
\ol{S} = \lim_{t \to \infty} S_{(t)},\qquad
S_{(t+1)} =  F(S_{(t)};\mc{S}_{N}).
\end{equation}
Algorithm \ref{Algorithm Riemannian Mean} provides a refined variant of this iteration including adaptive stepsize selection. See \cite{Congedo:2015aa} for alternative algorithms that determine the Riemannian mean.

\begin{algorithm}
	\textbf{Initialization} \\
	$\epsilon$ (termination threshold) \\
	$t = 0, \quad S_{(0)} =  L L^{\T}$,  
	with $S_{(0)}$ solving \eqref{eq:LE-mean}. \\
	$c_0 = \frac{\lambda_{\max}(S_{(0)})}{\lambda_{\min}(S_{(0)})
	)}$,
	$ \{\alpha_0, \beta_0 \} = 
	\big[\frac{\log(c_0)}{c_0-1},c_0\frac{\log(c_0)}{c_0-1}\big]$ (condition number 
		and step size selection parameters)
	\\
	$\tau_0 = \frac{2}{\alpha_0+\beta_0}$ \\
	$S_{(1)} = F_{\tau}(L;\mc{S}_{N})$   (iterative step) \\
	%$\gamma_k :=  r_k^TLr_k \text{ Line search on $\tilde{r}_k$}$\\$\\
	$	\epsilon_{1} = \big\|\sum_{i\in [N]} \omega_i 
	\logm(S^{\frac{1}{2}}_{(1)} S^{-1}_i  
	S^{\frac{1}{2}}_{(1)}\big\|_{F}, \quad t = 1$\\
	\While{$\epsilon_{t} > \epsilon$}{
		$S_{(t)} =  L L^{\T}$\\
		$c_{t} = \frac{\lambda_{\max}(S_{(t)})}{\lambda_{\min}(S_{(t)})}$\\
		\If{$c_{t} = 1$}{stop}$
		\{\alpha_{t}, \beta_{t} \} = \{\sum_{k = 
			0}^{t}\frac{\log(c_k)}{c_k-1},c_k\frac{\log(c_k)}{c_k-1}\}$ 
		\\
		$ \tau_{t} = \frac{2}{\alpha_{t}+\beta_{t}}$ \\
		$S_{(t+1)} = F_{\tau_{t}}(L;\mc{S}_{N})$ \\
		%$\gamma_k :=  r_k^TLr_k \text{ Line search on $\tilde{r}_k$}$\\
		$	\epsilon_{t+1} := \big\| \sum_{i\in [N]} \omega_i 
		\logm(S^{\frac{1}{2}}_{(t+1)} S^{-1}_i  
		S^{\frac{1}{2}}_{(t+1)})\big\|_{F} 
		, \quad t \leftarrow t+1$\\
	} 
	\caption{Fixed Point Iteration for Computing the Riemannian Matrix Mean.}
	\label{Algorithm Riemannian Mean}
\end{algorithm}

\subsubsection{Log-Euclidean Distance and Means}\label{sec:Log_Euclid}
A computationally cheap approach was proposed by \cite{Arsigny:2007aa} (among several other ones). Based on the operations
\begin{subequations}
\label{Log_Euclid}
\begin{align}
S_{1} \odot S_{2} 
&= \expm\big(\logm(S_{1}+\logm(S_{2})\big)),
\\
\lambda \cdot S &= \expm\big(\lambda\logm(S)\big),
\end{align}
\end{subequations}
the set $(\mc{P}_{s},\odot,\cdot)$ becomes isomorphic to the vector space where $\odot$ plays the role of addition. Consequently, the mean of the data $\mc{S}_{N}$ given by \eqref{eq:def-mcSN} is defined analogous to the arithmetic mean by
\begin{equation}\label{eq:LE-mean}
\ol{S} = \expm\Big(\sum_{i\in [N]}\w_{i}\logm(S_{i})\Big).
\end{equation}
While computing the mean is considerably cheaper than integrating the flow 
\eqref{eq:R-grad-Pd} using approximation Algorithm  \ref{Algorithm Riemannian 
Mean}, the critical drawback of relying on \eqref{eq:LE-mean} is not taking 
into account the (curved structure) of the manifold $\mc{P}_{d}$. Therefore, in 
the next section, we additionally consider another approximation of the 
Riemannian mean that better respects the underlying geometry but can still be 
evaluated more efficiently than the Riemannian mean of Section \ref{sec:Rmeans-approximation}.

\subsubsection{$S$-Divergence and Means}\label{sec:S-distance}
A general approach to the approximation of the objective function \eqref{eq:w-R-mean} is to replace the squared Riemannian $d_{g}^{2}(p,q)$ distance by a divergence function 
\begin{equation}\label{eq:D-approximates-dg2}
D(p,q) \approx \frac{1}{2} d_{g}^{2}(p,q)
\end{equation}
that satisfies 
\begin{subequations}
\begin{align}
D(p,q) &\geq 0
\qquad\text{and}\qquad
D(p,q)=0 \;\gdw\;p=q,
\\ \label{eq:D-Hessian}
\partial_{1}^{2} D(p,q) &\succ 0,\quad \forall p \in \dom D(\cdot,q).
\end{align}
\end{subequations}
We refer to, e.g., \cite{Censor:1997aa,Bauschke:1997aa} for a complete definition. Property \eqref{eq:D-Hessian} says that, for any feasible $p$, the Hessian with respect to the first argument is positive definite. In fact, suitable divergence functions $D$ recover in this way locally the metric tensor of the underlying manifold $\mc{M}$, in order to qualify as a surrogate for the squared Riemannian distance \eqref{eq:D-approximates-dg2}.

For the present case $\mc{M}=\mc{P}_{d}$ of interest, 
Sra \cite{Sra:2015aa} proposed the divergence function, called \textit{Stein divergence}
\begin{equation}
\label{Stein_div}
D_{s}(S_{1},S_{2}) = \log\det\Big(\frac{S_{1}+S_{2}}{2}\Big) - \frac{1}{2}\log\det (S_{1} S_{2}),\qquad
S, S_{1},S_{2} \in \mc{P}_{d}.
\end{equation}
Regarding the task of evaluating the Riemannian distance \eqref{eq:def-g-mcPd}, 
which is required for the second term of problem \eqref{eq:def-Rmean} for 
subsequential extraction of prototypes \eqref{eq:def-mcF-ast} in Section 
\eqref{sec:Experimental-Results},  while avoiding to solve the numerically 
involved numerical generalized eigenvalue problem, we replace 
\eqref{eq:def-Rmean} by  
\begin{equation}
\ol{S} = \arg\min_{S\in\mc{P}_{d}} J_{s}(S;\mc{S}_{N}),\qquad
J_{s}(S;\mc{S}_{N})
= \sum_{i\in [N]}\w_{i} D_{s}(S,S_{i}).
\end{equation}
The resulting Riemannian gradient flow reads
\begin{subequations}
\begin{align}
\dot S 
&= -\ggrad J_{s}(S;\mc{S}_{N})
\overset{\eqref{eq:R-grad-Pd}}{=} 
-S \partial J(S;\mc{S}_{N}) S
\\
&= -\frac{1}{2}\big(S R(S;\mc{S}_{N}) S - S\big),\qquad
R(S;\mc{S}_{N})
= \sum_{i\in [N]}\w_{i}\Big(\frac{S+S_{i}}{2}\Big)^{-1}.
\end{align}
\end{subequations}
Discretizing the flow using the geometric explicit Euler scheme with step size $h$, 
\begin{subequations}
\begin{align}
S_{(t+1)} 
&= \exp_{S_{(t)}}\big(-h \ggrad J_{s}(S_{(t)};\mc{S}_{N})\big)
\\
&\overset{\eqref{eq:def-exp-mcPd}}{=}
S_{(t)}^{\frac{1}{2}} \expm\Big(\frac{h}{2}\big(I-S_{(t)}^{\frac{1}{2}}R(S_{(t)};\mc{S}_{N})S_{(t)}^{\frac{1}{2}}\big)\Big) S_{(t)}^{\frac{1}{2}}
\end{align}
\end{subequations}
and using the log-Euclidean mean \eqref{eq:LE-mean} as initial point $S_{(0)}$, defines Algorithm \ref{Algorithm Riemannian Mean S} as listed below.

\begin{algorithm}
	\textbf{Initialization} \\
	$\epsilon$ (termination threshold) \\
	$t = 0, \quad S_{(0)}$ solves \eqref{eq:LE-mean} \\
	$ \epsilon_{0} > \epsilon $ (any value $\epsilon_{0}$) 
	\\
	\While{$\epsilon_{t} > \epsilon$}{
		$L L^{\T} = S_{(t)}$ \\
		$L_{i} L_{i}^{\T} = \frac{S_{(t)}+S_{i}}{2}$ for $i \in [N]$ \\
		$U = I - S_{(t)}^{\frac{1}{2}}\big(\sum_{i\in [N]}\w_{i}(L_{i}L_{i}^{\T})^{-1}\big) S_{(t)}^{\frac{1}{2}}$
		\\
		$S_{(t+1)} = S_{(t)}^{\frac{1}{2}}\expm(\frac{h}{2} U) S_{(t)}^{\frac{1}{2}}$ \\
		%$\gamma_k :=  r_k^TLr_k \text{ Line search on $\tilde{r}_k$}$\\
		$	\epsilon_{t+1} := 
		\|U\|_{F}
		, \quad t \leftarrow t+1$\\
	} 
	\caption{Computing the Geometric Matrix Mean Based on the $S$-divergence.}
	\label{Algorithm Riemannian Mean S}
\end{algorithm}

%%%%%%%%%%%%%%%%%%%%
\clearpage
%%%%%%%%%%%%%%%%%%%%

% !TEX root =  ../OCT-3D-AFlow.tex
%%%%%%%%%%%%%%%%%%%%%%%%%%%%%%%%%%

\section{Ordered Layer Segmentation}
\label{sec:Ordered-Segmentation}

In this section, we work out an extension of the assignment flow (Section \ref{sec:Assignment-Flow}) 
which is able to respect the order of cell layers as a global constraint while remaining in the same smooth geometric setting. In particular, existing schemes for numerical integration still apply to the novel variant.

\subsection{Ordering Constraint}

With regard to segmenting OCT data volumes, the order 
of cell layers is crucial prior knowledge. In this paper we focus on 
segmentation of the following 11 retina layers: Retinal Nerve 
Fiber Layer (RNFL), Ganglion Cell Layer (GCL), Inner Nuclear Layer (INL), Outer 
Plexiform Layer (OPL), Outer Nuclear Layer (ONL), two photoreceptor layers 
(PR1, PR2) separated by the External Limiting Membrane (ELM)
and the Retinal Pigment Epithelium (RPE) together with the Choroid Section (CS).
Figure \ref{OCT_Acquisition} also contains positions for the Internal Limiting 
Membrane (ILM) and Brunch Membrane (BM).

\begin{figure}[ht]
	\begin{adjustwidth*}{1cm}{1.5cm} 
		\scalebox{.85}{\begin{tikzpicture}
\centering
% Draw Layer Ordering Cube

\foreach \i/\Layer/\Name in 
{0/Layer_14/CS,1/Layer_13/CC,2/Layer_12/BM,3/Layer_11/RPE,4/Layer_10/PR2,5/Layer_9/PR1,
	6/Layer_8/ELM,7/Layer_7/ONL,8/Layer_6/OPL,
		9/Layer_5/INL,10/Layer_4/IPL,11/Layer_3/GCL,12/Layer_2/RNFL,13/Layer_1/ILM}{
	\coordinate (O) at (-15+\xb,-3.5+2*\i*0.25+2*\yy-0.2,0+\za);
	\coordinate (A) at (-15+\xb,-3.5+2*0.25\Width+2*\i*0.25+2*\yy-0.2,0+\za);
	\coordinate (B) at 
	(-15+\xb,-3.5+2*0.25\Width+2*\i*0.25+2*\yy-0.2,2*0.25\Height+\za);
	\coordinate (C) at (-15+\xb,-3.5+2*\i*0.25+2*\yy-0.2,2*0.25\Height+\za);
	\coordinate (D) at (-15+4*0.25\Depth+\xb,-3.5+2*\i*0.25+2*\yy-0.2,0+\za);
	\coordinate (E) at 
	(-15+4*0.25\Depth+\xb,-3.5+2*0.25\Width+2*\i*0.25+2*\yy-0.2,0+\za);
	\coordinate (F) at 
	(-15+4*0.25\Depth+\xb,-3.5+2*0.25\Width+2*\i*0.25+2*\yy-0.2,2*0.25\Height+\za);
	\coordinate (G) at 
	(-15+4*0.25\Depth+\xb,-3.5+2*\i*0.25+2*\yy-0.2,2*0.25\Height+\za);
	\draw[\Layer!80!black,fill=\Layer!80] (O) -- (C) -- (G) -- (D) -- 
	cycle;% 
	%Bottom 
%Face
	\draw[\Layer!80!black,fill=\Layer!80] (O) -- (A) -- (E) -- (D) -- 
	cycle;% 
	%Back 
%Face
	\draw[\Layer!80!black,fill=\Layer!80] (O) -- (A) -- (B) -- (C) -- 
	cycle;% 
	%Left 
%Face
	\draw[\Layer!40!black,fill=\Layer!80,opacity=1] (D) -- (E) -- (F) -- 
	(G) -- cycle;% Right Face
	\draw[\Layer!40!black,fill=\Layer!80,opacity=1] (C) -- (B) -- (F) -- 
	(G) -- 
	cycle;% Front Face
	\draw[\Layer!40!black,fill=\Layer!80,opacity=1] (A) -- (B) -- (F) -- 
	(E) -- 
	cycle;
	\node[text width=8em,color = \Layer](\Name) at (O) {\Name};}% T\op Face#

\draw[->,-latex, thin](-7.5,0) to[out=30,in=150] (-2,0);

\node[rounded corners=6pt, thin,align=center, fill=gray!20,
	text width=8em](C) at (-5.5,4) {Retina Layers};
	\draw[->] (C) |- (-6.5,2,2);

\coordinate (O) at (-5.7+\xa,-0.5+\ya,-1+\za);
\coordinate (A) at (-5.7+\xa,-0.5+\Width+\ya,-1+\za);
\coordinate (B) at (-5.7+\xa,-0.5+\Width+\ya,-1+0.25\Height+\za);
\coordinate (C) at (-5.7+\xa,-0.5+\ya,-1+0.25\Height+\za);
\coordinate (D) at (-2+0.25\Depth+\xa,-0.5+\ya,-1+\za);
\coordinate (E) at (-2+0.25\Depth+\xa,-0.5+\Width+\ya,-1+\za);
\coordinate (F) at (-2+0.25\Depth+\xa,-0.5+\Width+\ya,-1+0.25\Height+\za);
\coordinate (G) at (-2+0.25\Depth+\xa,-0.5+\ya,-1+0.25\Height+\za);
\draw[black,fill=blue!20,opacity=1] (O) -- (C) -- (G) -- (D) -- cycle;% 
%Bottom Face
\draw[black,fill=blue!20,opacity=1] (O) -- (A) -- (E) -- (D) -- cycle;% Back 
%Face
\draw[black,fill=blue!20,opacity=1] (O) -- (A) -- (B) -- (C) -- cycle;% Left 
%Face
\draw[black,fill=blue!20,opacity=1] (D) -- (E) -- (F) -- (G) -- cycle;% 
%Right Face
\draw[black,fill=blue!20,opacity=1] (C) -- (B) -- (F) -- (G) -- cycle;% 
%Front Face
\draw[black,fill=blue!20,opacity=1] (A) -- (B) -- (F) -- (E) -- cycle;% 
%Top Face

\coordinate (O) at (-5.15+\xa,0+\ya,2+\za);
\coordinate (A) at (-5.15+\xa,\Width+\ya,2+\za);
\coordinate (B) at (-5.15+\xa,\Width+\ya,2+0.25\Height+\za);
\coordinate (C) at (-5.15+\xa,0+\ya,2+0.25\Height+\za);
\coordinate (D) at (-3.5+0.25\Depth+\xa,\ya,2+\za);
\coordinate (E) at (-3.5+0.25\Depth+\xa,\Width+\ya,2+\za);
\coordinate (F) at (-3.5+0.25\Depth+\xa,\Width+\ya,2+0.25\Height+\za);
\coordinate (G) at (-3.5+0.25\Depth+\xa,0+\ya,2+0.25\Height+\za);
\draw[black,fill=blue!95,opacity=1] (O) -- (C) -- (G) -- (D) -- cycle;% 
%Bottom Face
\draw[black,fill=blue!95,opacity=1] (O) -- (A) -- (E) -- (D) -- cycle;% Back 
%Face
\draw[black,fill=blue!95,opacity=1] (O) -- (A) -- (B) -- (C) -- cycle;% Left 
%Face
\draw[black,fill=blue!95,opacity=1] (D) -- (E) -- (F) -- (G) -- cycle;% 
%Right Face
\draw[black,fill=blue!95,opacity=1] (C) -- (B) -- (F) -- (G) -- cycle;% 
%Front Face
\draw[black,fill=blue!95,opacity=1] (A) -- (B) -- (F) -- (E) -- cycle;% 
%Top Face

\coordinate (O) at (-3.5+\xa,0+\ya,2+\za);
\coordinate (A) at (-3.5+\xa,\Width+\ya,2+\za);
\coordinate (B) at (-3.5+\xa,\Width+\ya,2+0.25\Height+\za);
\coordinate (C) at (-3.5+\xa,0+\ya,2+0.25\Height+\za);
\coordinate (D) at (-3.5+0.25\Depth+\xa,\ya,2+\za);
\coordinate (E) at (-3.5+0.25\Depth+\xa,\Width+\ya,2+\za);
\coordinate (F) at (-3.5+0.25\Depth+\xa,\Width+\ya,2+0.25\Height+\za);
\coordinate (G) at (-3.5+0.25\Depth+\xa,0+\ya,2+0.25\Height+\za);
\draw[black,fill=yellow!85] (O) -- (C) -- (G) -- (D) -- cycle;% Bottom Face
\draw[black,fill=yellow!85] (O) -- (A) -- (E) -- (D) -- cycle;% Back Face
\draw[black,fill=yellow!85] (O) -- (A) -- (B) -- (C) -- cycle;% Left Face
\draw[black,fill=yellow!85,opacity=0.8] (D) -- (E) -- (F) -- (G) -- cycle;% 
%Right Face
\draw[black,fill=yellow!85,opacity=0.6] (C) -- (B) -- (F) -- (G) -- cycle;% 
%Front Face
\draw[black,fill=yellow!85,opacity=0.8] (A) -- (B) -- (F) -- (E) -- cycle;% 
%Top Face

%Draw Bscan 

\coordinate (O) at (-3.20+\xa,0+\ya,2+\za);
\coordinate (A) at (-3.20+\xa,\Width+\ya,2+\za);
\coordinate (B) at (-3.20++\xa,\Width+\ya,2+0.25\Height+\za);
\coordinate (C) at (-3.20++\xa,0+\ya,2+0.25\Height+\za);
\coordinate (D) at (-1.5+0.25\Depth+\xa,\ya,2+\za);
\coordinate (E) at (-1.5+0.25\Depth+\xa,\Width+\ya,2+\za);
\coordinate (F) at (-1.5+0.25\Depth+\xa,\Width+\ya,2+0.25\Height+\za);
\coordinate (G) at (-1.5+0.25\Depth+\xa,0+\ya,2+0.25\Height+\za);
\draw[black,fill=blue!95,opacity=1] (O) -- (C) -- (G) -- (D) -- cycle;% 
%Bottom Face
\draw[black,fill=blue!95,opacity=1] (O) -- (A) -- (E) -- (D) -- cycle;% Back 
%Face
\draw[black,fill=blue!95,opacity=1] (O) -- (A) -- (B) -- (C) -- cycle;% Left 
%Face
\draw[black,fill=blue!95,opacity=1] (D) -- (E) -- (F) -- (G) -- cycle;% 
%Right Face
\draw[black,fill=blue!95,opacity=1] (C) -- (B) -- (F) -- (G) -- cycle;% 
%Front Face
\draw[black,fill=blue!95,opacity=1] (A) -- (B) -- (F) -- (E) -- cycle;%

\coordinate (O) at (-7.15+\xa,-2+\ya,-2+\za);
\coordinate (A) at (-7.15+\xa,-2+\Width+\ya,-2+\za);
\coordinate (B) at (-7.15+\xa,-2+\Width+\ya,-2+0.25\Height+\za);
\coordinate (C) at (-7.15+\xa,-2+\ya,-2+0.25\Height+\za);
\coordinate (D) at (-3.45+0.25\Depth+\xa,-2+\ya,-2+\za);
\coordinate (E) at (-3.45+0.25\Depth+\xa,-2+\Width+\ya,-2+\za);
\coordinate (F) at (-3.45+0.25\Depth+\xa,-2+\Width+\ya,-2+0.25\Height+\za);
\coordinate (G) at (-3.45+0.25\Depth+\xa,-2+\ya,-2+0.25\Height+\za);
\draw[black,fill=blue!40,opacity=0.4] (O) -- (C) -- (G) -- (D) -- cycle;% 
%Bottom Face
\draw[black,fill=blue!40,opacity=0.4] (O) -- (A) -- (E) -- (D) -- cycle;% Back 
%Face
\draw[black,fill=blue!40,opacity=0.4] (O) -- (A) -- (B) -- (C) -- cycle;% Left 
%Face
\draw[black,fill=blue!40,opacity=0.4] (D) -- (E) -- (F) -- (G) -- cycle;% 
%Right Face
\draw[black,fill=blue!40,opacity=0.4] (C) -- (B) -- (F) -- (G) -- cycle;% 
%Front Face
\draw[black,fill=blue!40,opacity=0.4] (A) -- (B) -- (F) -- (E) -- cycle;% 
%Top Face

% Draw A Scan

%Top Face

%Draw OCT Cube 
\coordinate (O) at (0,0,0);
\coordinate (A) at (0,\Width,0);
\coordinate (B) at (0,\Width,\Height);
\coordinate (C) at (0,0,\Height);
\coordinate (D) at (\Depth,0,0);
\coordinate (E) at (\Depth,\Width,0);
\coordinate (F) at (\Depth,\Width,\Height);
\coordinate (G) at (\Depth,0,\Height);
\draw[red!60!black,fill=red!5,opacity=0.4] (O) -- (C) -- (G) -- (D) -- cycle;% 
%Bottom Face
\draw[red!60!black,fill=red!5,opacity=0.4] (O) -- (A) -- (E) -- (D) -- cycle;% 
%Back Face
\draw[red!60!black,fill=red!5,opacity=0.4] (O) -- (A) -- (B) -- (C) -- cycle;% 
%Left Face
\draw[red!60!black,fill=red!5,opacity=0.4] (D) -- (E) -- (F) -- (G) -- cycle;% 
%Right Face
\draw[red!60!black,fill=red!5,opacity=0.4] (C) -- (B) -- (F) -- (G) -- cycle;% 
%Front Face
\draw[red!60!black,fill=red!5,opacity=0.4] (A) -- (B) -- (F) -- (E) -- cycle;% 
%Top Face

\draw (2.8,3.5,2) node {\scriptsize{\color{black}$\circled{2}$}};
\draw (2.8,3,2) node[rotate = 180] 
{{\color{black}$\underbrace{\hspace{4cm}}$}};
\draw (3.9,-1.2,1) node[rotate = 225] {\scriptsize{\color{black}$\circled{3}$}};
\draw (5,1.8,2) node[rotate = 90] 
{{\color{black}$\underbrace{\hspace{2cm}}$}};
\draw (5,1.4,1) node[rotate = 270] {\scriptsize{\color{black}$\circled{1}$}};
\draw (3.95,-0.6,2) node[rotate = 45] 
{{\color{black}$\underbrace{\hspace{3.3cm}}$}};
\node[rounded corners=6pt, thin,align=center, fill=gray!20,
text width=8em](A_S) at (-1.85,4) {A-Scan};
\draw[->] (A_S) -| (1,1.3,1.5);
\node[rounded corners=6pt, thin,align=center, fill=gray!20,
text width=8em](B) at (3,4) {B-Scan};
\draw[->] (B) -| (1.5,1.5,2);

%Draw Retina Layer Cut Overview
\end{tikzpicture}}
	\end{adjustwidth*}
	\caption{OCT volume acquisition: 
		\raisebox{.6pt}{\textcircled{\raisebox{-.9pt} 
				{1}}} is the A-scan axis (single A-scan is marked yellow). 
		Multiple A-scans taken 
		in rapid succession along axis 
		\raisebox{.6pt}{\textcircled{\raisebox{-.9pt} 
				{2}}} form a two-dimensional 
		B-scan (single B-scan is marked blue). The complete OCT volume is 
		formed by repeating this procedure along
		axis \raisebox{.6pt}{\textcircled{\raisebox{-.9pt} 
				{3}}}. A list of retina layers that we expect to find in every 
		A-scan is shown on the left.}
	\label{OCT_Acquisition}
\end{figure}

 To incorporate this knowledge into the 
geometric setting of Section \ref{sec:Assignment-Flow}, we require a smooth 
notion of ordering which allows to compare two probability distributions.
In the following, we assume prototypes $f^{\ast}_{j} \in \mc{F}$, $j \in [n]$ 
in some feature space $\mc{F}$ to be indexed such that ascending label indices 
reflect the physiological order of cell layers.

\begin{definition}[Ordered Assignment Vectors]\label{def:ordered_assignments}
A pair of voxel assignments $(w_i, w_j)\in \mc{S}^2$, $i < j$ within a single A-scan is called 
\emph{ordered}, if $w_j - w_i \in K = \{ By\colon y\in \R^c_+ \}$ with the 
matrix
\begin{equation}
	B = \left(\begin{array}{ccccc}
		-1 &    &        &    & \\
		1  & -1 &        &    & \\
		   & 1  & \ddots &    & \\
		   &    & \ddots & -1 & \\
		   &    &        & 1  & -1
		\end{array}
	\right)	\in \R^{c\times c}\ .
\end{equation} 
\end{definition}

This new continuous ordering of probability distributions is consistent with discrete ordering of layer indices in the following way.

\begin{lemma}
Let $w_i = e_{l_1}$, $w_j = e_{l_2}$, $l_1, l_2\in [c]$ denote two integral 
voxel assignments. Then $w_j-w_i\in K$ if and only if $l_1 \leq l_2$.
\end{lemma}
\begin{proof}
$B$ is regular with inverse
\begin{equation}
	B^{-1} = -Q,\qquad Q_{i,j} = \begin{cases}1 & \text{if } i\geq j\\ 0 & 
\text{else}\end{cases}
\end{equation}
and $w_j-w_i\in K \gdw B^{-1}(w_j-w_i)\in \R^c_+$. It holds
\begin{equation}
	B^{-1}(w_j-w_i) = Qe_{l_1}-Qe_{l_2} = \sum_{k=l_1}^c e_k - \sum_{k=l_2}^c 
e_k
\end{equation}
such that $B^{-1}(w_j-w_i)$ has nonnegative entries exactly if $l_1 \leq l_2$.
\end{proof}

The continuous notion of order preservation put forward in Definition 
\ref{def:ordered_assignments} can be interpreted in terms of a related 
discrete graphical model. Consider a graph consisting of two nodes connected by a 
single edge. The order constrained image labeling problem on this graph can be 
written as the integer linear program
\begin{equation}
	\min_{W\in \{0,1\}^{2\times c}, M\in \Pi(w_i,w_j)} \la W, D\ra + \theta\la 
Q-\II, M\ra
\end{equation}
where $\Pi(w_i,w_j)$ denotes the set of coupling measures for marginals $w_i$, 
$w_j$ and $\theta \gg 0$ is a penalty associated with violation of the 
ordering constraint. By taking the limit $\theta \to \infty$ we find the more tightly 
constrained problem
\begin{equation}\label{eq:tightened_lp}
	\min_{W\in \{0,1\}^{2\times c}, M\in \Pi(w_i,w_j)} \la W, D\ra\qquad 
\text{s.t. }\la Q-\II, M\ra = 0\ .
\end{equation}
Its feasible set has an informative relation to Definition 
\ref{def:ordered_assignments} examined in Proposition \ref{prop:ordered_transport}.
\begin{lemma}\label{lem:pos_weight_redistribution}
Let $M \in \R^{c\times c}$ be an upper triangular matrix with non-negative entries above the diagonal and non-negative marginals
\begin{equation}\label{eq:lem_marginals}
	M\eins_c \geq 0,\qquad M^\top\eins_c \geq 0\ .
\end{equation}
Then there exists a modified matrix $M^1$ with the same properties such that $M^1 \geq 0$.
\end{lemma}
\begin{proof}
\eqref{eq:lem_marginals} directly implies $M_{11}\geq 0$ and $M_{cc}\geq 0$ because $M$ is upper triangular. For row indices $l\neq m$ and column indices $q\neq r$, define the matrix $O^{lm,qr}$ with
\begin{equation}\label{eq:redistribution_matrix}
	O^{lm,qr}_{ij} = \begin{cases}
		-1& \text{ if } (i,j)=(l,q) \lor (i,j)=(m,r)\\
		1 & \text{ if } (i,j)=(l,r) \lor (i,j)=(m,q)\\
		0 & \text{ else}
	\end{cases}\ .
\end{equation}
Then $O^{lm,qr}\eins = (O^{lm,qr})^\top\eins = 0$. Adding a matrix $O^{lm,qr}$ to $M$ does therefore not change its marginals, but it redistributes mass from the positions $(l,q)$ and $(m,r)$ to the positions $(l,r)$ and $(m,q)$. Due to \eqref{eq:lem_marginals}, it is possible to choose scalars $\alpha^k_{lr} \geq 0$ such that
\begin{equation}\label{eq:redistribution_sum}
	M + \sum_{2\leq k\leq c-1}\;\sum_{\substack{l < k\\ r > k}} \alpha^k_{lr} O^{lk,kr} \geq 0\ .
\end{equation}
\end{proof}
\begin{proposition}\label{prop:ordered_transport}
A pair of voxel assignments $(w_i, w_j)\in \mc{S}^2$ within an single A-scan is ordered if and only 
if the set
\begin{equation}\label{eq:ordered_couplings}
	\Pi(w_i,w_j) \cap \{ M\in\R^{c\times c}\colon \la Q-\II, M\ra = 0 \}
\end{equation}
is not empty.
\end{proposition}
\begin{proof}
$"\Leftarrow"$ Suppose there exists a measure $M\in \R^{c\times c}$ with 
marginals $w_i$, $w_j$ and $\la Q-\II, M\ra = 0$. Then
\begin{equation}\label{eq:explicit_y_proof}
	w_j - w_i = By \;\Leftrightarrow\; Q(M-M^\top)\eins = y\ .
\end{equation}
It suffices to show that no entry of $y$ is negative. Define the shorthand $\zeta = (M-M^\top)\eins$. Further,
let $M_{\cdot,k}$ denote the $k$-th column of $M$ and let $M_{k,\cdot}$ denote the $k$-th row of $M$. 
$\zeta$ has entries
\begin{equation}
	\zeta_l = (M-M^\top)\eins|_l = \la M_{l,\cdot} - M_{\cdot,l}, \eins\ra = \sum_{k=l}^c M_{l,k} - \sum_{k=1}^l M_{k,l},\qquad l\in [c]\ .
\end{equation}
By \eqref{eq:explicit_y_proof}, the entries of $y$ read
\begin{equation}
	y_r = \sum_{q = 1}^r \zeta_q\ .
\end{equation}
We can now inductively show that $y_r \geq 0$ for all $r\in[c]$. The cases $r=1$ and $r=c$ are immediate:
\begin{align}
	y_1 &= \zeta_1 = \sum_{k=1}^c M_{1,k} - M_{1,1} = \sum_{k=2}^c M_{1,k} \geq 0\label{eq:induction_start}\\
	y_c &= \la \zeta, \eins\ra = \la M-M^\top, \eins\eins^\top\ra = \sum_{i,j\in [c]} M_{i,j} - \sum_{i,j\in [c]} M^\top_{i,j} = 0\ .
\end{align}
For $r\in\{2,\dotsc,c-1\}$ we make the hypothesis that
\begin{equation}\label{eq:induction_hypothesis}
	y_r = \sum_{q=1}^r \zeta_q = \sum_{k=r+1}^c\left( M_{1,k} + \dotsc + M_{r,k} \right) \geq 0
\end{equation}
which is consistent with the result for $r=1$ in \eqref{eq:induction_start}. It follows
\begin{align}
	y_{r+1} &= \sum_{q=1}^{r+1} \zeta_q\\
		&= \zeta_{r+1} + \sum_{k=r+1}^c\left( M_{1,k} + \dotsc + M_{r,k} \right)\label{eq:use_hypothesis}\\
		&= \sum_{k=r+1}^c M_{r+1,k} - \sum_{k=1}^{r+1} M_{k,r+1} + \sum_{k=r+1}^c\left( M_{1,k} + \dotsc + M_{r,k} \right)\\
		%&= \sum_{k=r+2}^c M_{r+1,k} - \sum_{k=1}^{r} M_{k,r+1} + \sum_{k=r+1}^c\left( M_{1,k} + \dotsc + M_{r,k} \right)\\
		&= \sum_{k=r+2}^c M_{r+1,k} + \sum_{k=r+2}^c\left( M_{1,k} + \dotsc + M_{r,k} \right)\\
		&= \sum_{k=r+2}^c\left( M_{1,k} + \dotsc + M_{r,k} + M_{r+1,k} \right)
\end{align}
where we used \eqref{eq:induction_hypothesis} in \eqref{eq:use_hypothesis}. 
This completes the inductive step and thus shows $y\geq 0$.\\[1em]
$"\Rightarrow"$ Let $(w_i,w_j)$ be ordered. Following Definition 
\eqref{def:ordered_assignments}, it holds 
\begin{align}
	B^{-1}(w_j-w_i) = Q(w_i-w_j) \in \R^c_+. 
\end{align}
We show the existence of a transport plan $M \geq 0$ satisfying 
\begin{align}\label{eq:dis_constr}
	M \eins = w_i,\quad M^\top \eins = w_j
\end{align}
as well as the ordering constraint $\la Q-\II, M\ra = 0$ by direct construction.
For $c = 2$, 
\begin{equation}\label{eq:plan_induction_start}
M = \begin{pmatrix} 
		(w_j)_1 & (w_i)_1-(w_j)_1 \\
		0       & 1-(w_i)_1
	\end{pmatrix}
\end{equation} 
satisfies these requirements. Now, let $c > 2$ and define the mapping
\begin{align}
	C_{1}^{c-1}: \Delta_{c} &\to \Delta_{c-1}\\
						w   &\mapsto \tilde{w} = 
						(w_2,\dots,w_c) + \frac{w_1}{c-1}\eins_{c-1}.
\end{align}
If $(w_i, w_j) \in \Delta_c^2$ is ordered, then 
$(\tilde w_i,\tilde w_j) := (C_{1}^{c-1}(w_i), C_{1}^{c-1}(w_j)) \in \Delta_{c-1}^2$ is ordered as well because
\begin{equation}
	Q(\tilde w_i - \tilde w_j) = Q(\bar w_i - \bar w_j) + \frac{(w_i)_1 - (w_j)_1}{c-1}Q\eins \geq 0
\end{equation}
where $\bar w_i$ denotes the vector $((w_i)_2,\dotsc,(w_i)_c)$. Suppose a transport plan $\tilde{M} \in \R^{(c-1)\times(c-1)}$ exists such that
\begin{align}
	\tilde{M}\eins_{c-1} = \tilde{w}_i \quad \tilde{M}^\top\eins_{c-1} = 
	\tilde{w}_j, \quad \tilde{M} \geq 0.
\end{align}
To complete the inductive step, we consider the matrix 
\begin{align}
	M^0 := \begin{pmatrix}
	(w_j)_1 & s^\top \\
	0       & \tilde{M}-\frac{(w_i)_1}{c-1}I
	\end{pmatrix}, 
	\quad s = \frac{(w_i)_1-(w_j)_1}{c-1}\eins_{c-1}
\end{align}
which satisfies \eqref{eq:dis_constr} as well as $\la Q-\II, M^0\ra = 0$.
By Lemma \ref{lem:pos_weight_redistribution}, $M^0$ can be modified to yield a transport plan with the desired properties.
\end{proof}
Proposition \ref{prop:ordered_transport} shows that transportation plans 
between ordered voxel assignments $w_i$ and $w_j$ exist which do not move mass from 
$w_{i,l_1}$ to $w_{j,l_2}$ if $l_1 > l_2$. This characterizes order preservation for 
non-integral assignments as put forward in Definition 
\ref{def:ordered_assignments}.

\subsection{Ordered Assignment Flow}\label{sec:ordered_af}

Likelihoods as defined in \eqref{schnoerr-eq:def-Li} emerge by lifting 
$-\frac{1}{\rho}D_{\mc{F}}$ regarded as Euclidean gradient of $-\frac{1}{\rho}\la D_{\mc{F}}, W \ra$ 
to the assignment manifold. It is our goal to encode order preservation into a 
generalized likelihood matrix $L_\text{ord}(W)$. To this end, consider the assignment matrix
$W\in\mc{S}^N$ for a single A-scan consisting of $N$ voxels.
We define the related
matrix $Y(W)\in\R^{N(N-1)\times c}$ with rows indexed by 
pairs $(i,j)\in [N]^2$, $i\neq j$ in fixed but arbitrary order. Let the rows of $Y$ be given by
\begin{equation}
	Y_{(i,j)}(W) = \begin{cases}Q(w_j-w_i) & \text{if }i > j\\Q(w_i-w_j) & 
\text{if }i < j\end{cases}\ .
\end{equation}
By construction, an A-scan assignment $W$ is ordered exactly if all entries of the corresponding $Y(W)$ are 
nonnegative. 
This enables to express the ordering constraint on a single A-scan in terms of the energy objective
\begin{equation}\label{eq:a_scan_order_energy}
	E_\text{ord}(W) \;= \sum_{(i,j)\in [N]^2,\;i\neq j} 
	\phi(Y_{(i,j)}(W))\ .
\end{equation}
where $\phi\colon \R^c \to \R$ denotes a smooth approximation of $\delta_{\R^c_+}$. In our numerical
experiments, we choose
\begin{equation}
\label{gamma_par}
	\phi(y) = \left\la\gamma \exp \left( -\frac{1}{\gamma} y \right), 
\eins\right\ra
\end{equation}
with a constant $\gamma > 0$. Suppose a full OCT volume assignment matrix $W\in\mc{W}$ is given and denote
the set of submatrices for each A-scan by $C(W)$. Then order preserving assignments consistent with given distance
data $D_{\mc{F}}$ in the feature space $\mc{F}$ are found by minimizing the energy objective
\begin{equation}\label{eq:full_vol_energy}
	E(W) = \la D_{\mc{F}}, W \ra + \sum_{\mathclap{W_A\in 
			C(W)}} E_\text{ord}(W_A)\ .
\end{equation}
We consequently define the generalized likelihood map
\begin{equation}\label{eq:ordering_likelihood}
	L_\text{ord}(W) = \exp_W\left(-\nabla E(W)\right) = 
	\exp_W\left(-\frac{1}{\rho}D_{\mc{F}} - \sum_{\mathclap{W_A\in 
			C(W)}} \nabla E_\text{ord}(W_A)\right)
\end{equation}
and specify a corresponding assignment flow variant.
\begin{definition}[Ordered Assignment Flow]\label{def:ordered_af}
The dynamical system
\begin{equation}\label{eq:ordered_af}
	\dot W = R_WS(L_\text{ord}(W)),\qquad W(0) = \eins_\mc{W}
\end{equation}
evolving on $\mc{W}$ is called the \emph{ordered assignment flow}.
\end{definition}
By applying known numerical schemes \cite{Zeilmann:2020aa} for approximately integrating the flow \eqref{eq:ordered_af}, we find a class of discrete-time image labeling algorithms which respect the physiological cell layer ordering in OCT data. In chapter \ref{sec:Experimental-Results}, we benchmark the simplest instance of this class, emerging from the choice of geometric Euler integration.
% !TEX root =  ../OCT-3D-AFlow.tex
%%%%%%%%%%%%%%%%%%%%%%%%%%%%%%%%%%

\section{Experimental Results}\label{sec:Experimental-Results}

\subsection{Data, Competing Approaches, Performance Measures}

\subsubsection{OCT-Data}
In the following sections, after introducing key terminology in volumetric OCT 
data we describe experiments performed on a set of OCT volumes
depicting the intensity of light reflection in chorioretinal tissues centered 
around the fovea. The scans were obtained using a spectral domain OCT device 
(Heidelberg Engineering, Germany) for multiple patients
at a variety of resolutions by averaging various registered B-scan which share 
same location to reduce the speckle noise. This is representative of the fact 
that
different resolutions may be desirable in clinical settings at the preference of
medical practitioners. 
In the following, we always assume an OCT volume in question to consist of 
$N_B$ B-scans, each comprising $N_A$ A-scans
with $N$ voxels and use the term surface to represent the set of voxels located 
in the interface of two retina layers. See Figure \ref{OCT_Acquisition} for a 
schematic acquisition 
illustration of retina layers and separating membranes.

 In this publication we 
use OCT volumes of size $(N\times N_A \times N_B) = (498 \times 768 \times 61)$ 
with a approximate resolution of \SI{3.87}{\micro\meter}/voxel along 
$N$,$N_A$ 
direction and with a resolution of \SI{11,11}{\micro\meter}/voxel on $N_B$ 
axis. The 
volume set was divided into training set and the testing set where the latter 
consists of 8 volumes extracted 
from different patients without any observable pathological retina changes. 
Figure 
\eqref{fig:Intensity_Plot} 
provides a 
view on a Bscan located in the OCT volume on which the ordered 
assignment flow is validated. The right plot depicts the noisy signal along an 
A-scan indicated by a yellow vertical line which underpins 
the difficulty of segmenting the underlying data sets.

\begin{figure}[ht!]
	\begin{subfigure}[t]{0.33\textwidth}
		\includegraphics[width=5cm,height=4cm]{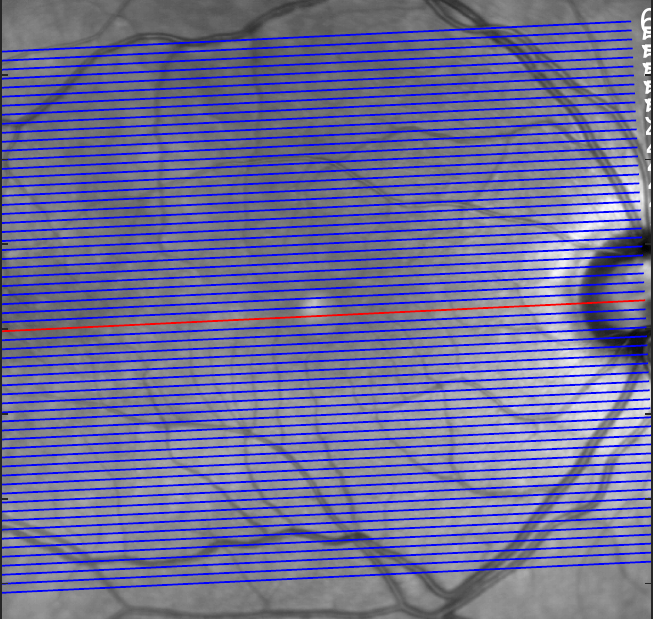}
		\caption{}
	\end{subfigure}
	\begin{subfigure}[t]{0.33\textwidth}
		\includegraphics[width=5cm,height=4cm]{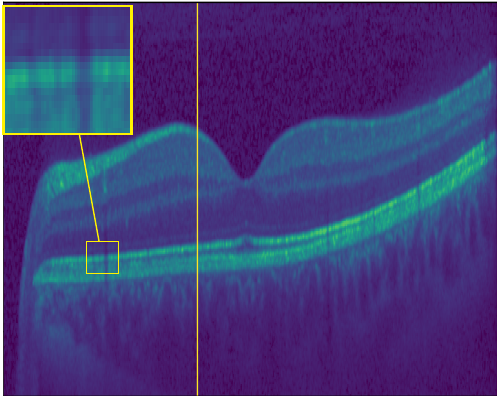}
		\caption{}
	\end{subfigure} 
	\begin{subfigure}[t]{0.32\textwidth}
		\scalebox{1.05}{\includegraphics[width=5cm,height=4cm]{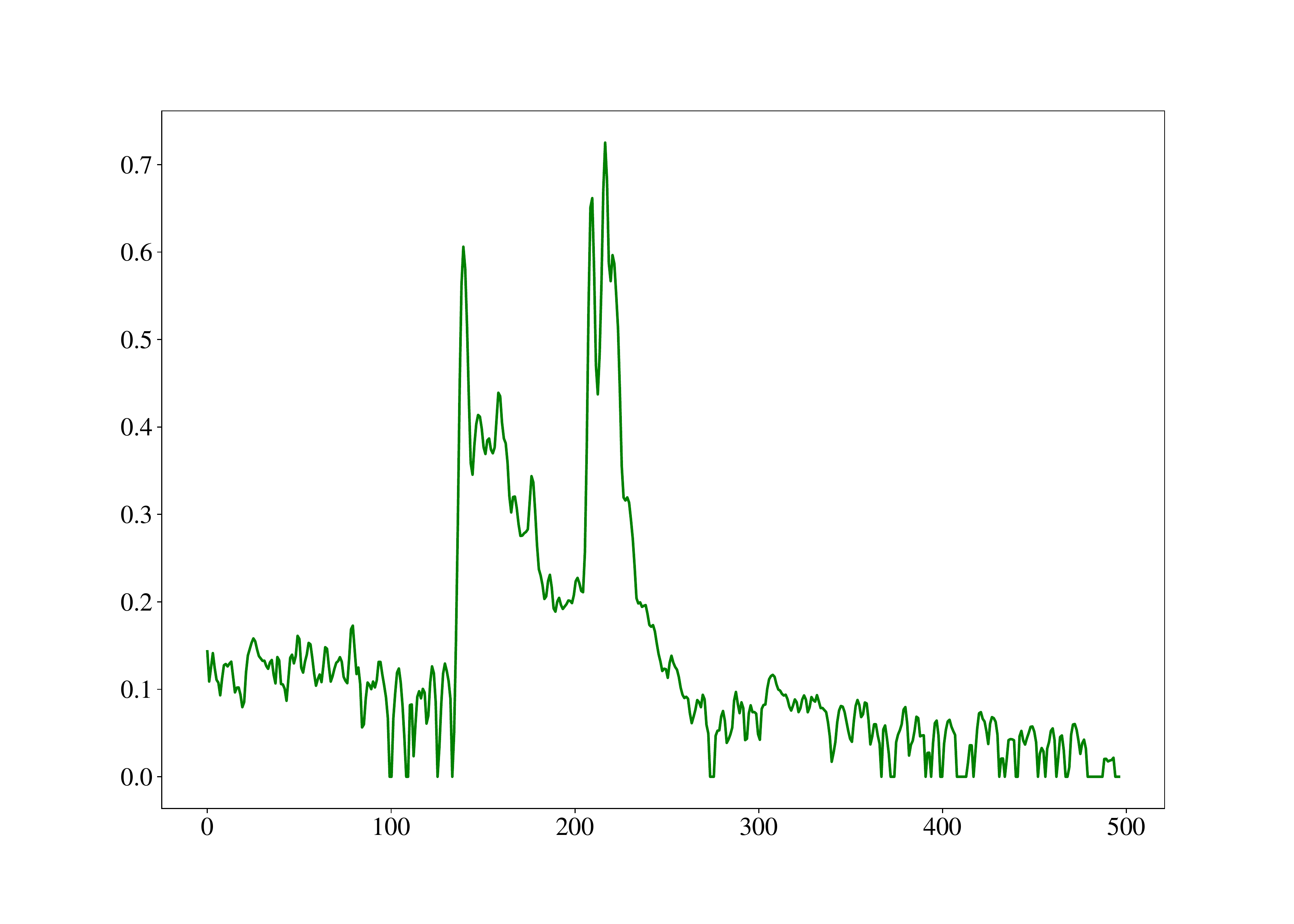}
		}\caption{}
	\end{subfigure} \hfill
	\captionsetup{font=footnotesize}
	\caption{\textbf{Left}: En-face view on the volumetric OCTA data with red 
	line 
		indicating the location of the B-scan shown in the center image. 
		\textbf{Center}: The enlarged view depicts typical artifacts such as 
		shadow regions and speckle noise. 
		\textbf{Right}: The gray value intensity of one vertical A-scan within 
		the 
		B-Scan represented by yellow line in the center image. The noisy 
		intensity plot indicates the difficulty of application to segment the 
		OCT-volume accurately while respecting layer ordering.} 
	\label{fig:Intensity_Plot}
\end{figure}

\subsubsection{Reference Methods}\label{sec:Reference_Methods}
To assess the segmentation performance of our proposed approach we compare 
ourselves to state of the art retina segmentation methods presented in 
Rathke \cite{Rathke2014} and Li \cite{Li:2006aa} which are applicable for both 
healthy and pathological patient data. In particular, we prefer these reference 
methods over \cite{Dufour:2013aa}, \cite{Song:2012aa} and \cite{Garvin:2008aa} 
because available implementations of the latter are limited to the segmentation of 
up to 9 retina layers.\\

\textbf{\textit{IOWA Reference Algorithm:}}
A well-known graph-based approach to segmentation of macular volume data was 
developed by the Retinal Image Analysis Laboratory at the Iowa Institute for 
Biomedical Imaging
\cite{Li:2006aa,Abramow:2010aa,Garvin:2008aa}.
The problem of localizing cell layer boundaries in 3D OCT volumes
is posed and ultimately transformed into a minimum $st$-cut problem on a non-trivially
constructed graph $G$.
To this end, a distance tensor $D_k\in\R^{N_B\times N_A\times N}$ is formed in a feature
extraction step for each boundary $k\in [c-1]$. This encodes $c-1$ separate binary
segmentation problems on a geometric graph $G_k$ spanning the volume. In each instance, 
voxels are to be classified as either belonging to boundary $k$ or not belonging 
to boundary $k$. By utilizing a (directed) neighborhood structure on each $G_k$,
smoothness constraints are introduced and regulated via user-specified stiffness 
parameters.
To model interactions between different boundaries, the graphs $G_k$ are combined 
to a global graph $G$, introducing additional edges between them. The latter set up
constraints on the distance between consecutive boundaries within each 
A-scan which can be used to enforce physiological ordering of cell layers.
On $G$, the problem of optimal boundary localization takes the form of minimal
closed set construction which is in turn transformed into a minimum $st$-cut
problem for which standard methods exist. Their standalone software is freely 
available for research purposes\footnote{see 
\url{https://www.iibi.uiowa.edu/oct-reference}}.\\

\textbf{\textit{Probabilistic Model:}}
Rathke et al. \cite{Rathke2014}
proposed a graph-based probabilistic approach for segmenting OCT volumes for
given data $y$ by leveraging the Bayesian ansatz
\begin{align}
	\label{Graphical_Model}
	p(y,s,b) = p(y|s)p(s|b)p(b)\ . 
\end{align}
Here, the tensor $b \in \R^{N_B\times N_A \times (c-1)}$ contains real-valued 
boundary positions between retina layers and $s$ denotes discrete (voxel-wise)
segmentation.
The appearance terms $p(y|s)$, $p(s|b)$ and $p(b)$ represent data 
likelihood, Markov random field regularizer and global shape prior 
respectively.
In order to approximate the desired posterior
\begin{align}
	p(b,s|y) = \frac{p(y|s)p(s|b)p(b)}{p(y)}\ , 
\end{align}
a variational inference strategy is employed. This aims to find a tractable
distribution $q$ decoupled into
\begin{equation}
	q(b,s) = q_b(b)q_s(s)
\end{equation}
which is close to $p(b,s|y)$ in terms of the relative entropy $\text{KL}(q\,|\,p)$.
The shape prior $p(b)$ is learned offline by maximum likelihood estimation in the space of normal 
distributions using a low-rank approximation of the involved covariance matrix.
Ordering constraints 
\begin{align}
1 \leq s_{1,ij} \leq s_{2,ij} \leq \cdots \leq s_{c-1,ij}, \quad ij \in [N_B] \times [N_A]
\end{align}
are enforced for the discrete segmentation $s$ and are not enforced for the 
continuous boundaries $b$. This is in contrast to the proposed model which 
integrates the ordering of retina layers by adding a cost function 
\eqref{eq:ordered_couplings} penalizing the overall deviation of soft 
assignments  by integrating \eqref{eq:assignment-flow} from the subspace of 
probability distributions satisfying \eqref{def:ordered_assignments}. The 
method comes along with a standalone software free available under \footnote{ 
\url{https://github.com/FabianRathke/octSegmentation}}.

\subsubsection{Performance Measures}

We will evaluate the computed segmentations 
by their direct comparison with manual annotations regarded as gold 
standard which were realized by an medical expert.
Respective metrics are suitable for segmentation tasks that involve multiple tissue 
types \cite{Crum:2006aa}.
Specifically, we report the mean DICE similarity coefficient \cite{Dice:1945aa} 
for each segmented cell layer.

\begin{definition}{\textit{\textbf{(DICE)}}}
Given two sets $A,B$ the \textit{DICE similarity coefficient} is defined as 
\begin{equation}\label{DICE}
	\DSC(A,B) := \frac{2|A \cap B|}{|A|+|B|} = \frac{2TP}{2TP+FP+FN} \in [0,1],
\end{equation} 
where $\{TP,FN,FP\}$ denotes the number of \textit{true 
positives}, \textit{false negatives} and \textit{false positives} respectively.
\end{definition}

The DICE similarity coefficient quantifies the region agreement between computed
segmentation results and manually labeled OCT volumes which serve as ground truth.
High similarity index $\DSC(A,B) \approx 1$ indicates large relative overlap 
between the sets $A$ and $B$.
This metric is well suited for average performance evaluation and 
appears frequently in the literature 
(e.g. \cite{Chiu:2015aa}, \cite{Yazdanpanah:2011aa} and \cite{Novosel:2017aa}). 
It is closely related to the positively
correlated Jaccard similarity measure \cite{Jaccard:1908aa} which in contrast 
to \eqref{DICE} is more strongly influenced by worst case performance.

In addition, we report the mean absolute error (MAE) of computed layer 
boundaries used in \cite{Rathke2014} and \cite{Garvin:2008aa} to make our 
results more directly comparable to these references.
 
\begin{definition}{\textbf{\emph{(Mean Absolute Error)}}}
 For a single A-scan indexed by $ij\in [N_B]\times [N_A]$,
 let $e_{ij} :=| g_{ij}-p_{ij} |$ denote the absolute difference between 
 a layer boundary position $g_{ij}$ in the gold standard segmentation
 and a predicted layer boundary $p_{ij}$. The mean absolute error
 (MAE) is defined as the mean value
 \begin{equation}\label{MAE}
 	\MAE(g,p) = \frac{1}{N_BN_A}\sum_{ij \in [N_B]\times [N_A]} e_i\ .
 \end{equation}
\end{definition}  

\subsection{Feature Extraction}

\subsubsection{Region Covariance Descriptors}\label{sec:experiments_cov_descr}
To apply the geometric framework proposed in Section 
\ref{sec:OCT-Segmentation} we next introduce the \emph{region covariance 
descriptors} \cite{Tuzel:2006aa} which have been widely applied in computer 
vision and medical imaging, see 
e.g.~\cite{Cherian:2016aa,Turaga:2016aa,DEPEURSINGE:2014aa,Korsuk:2015aa}. We 
model the raw intensity data for a given OCT volume by a mapping $I:\mc{D} \to 
\R_+$ where $\mc{D} \subset \R^3$ is the underlying spatial domain.
To each voxel $v \in \mc{D}$, we associate the local feature vector $f \colon 
\mc{D} \to \R^{10}$,
\begin{gather}
\label{Feature_Vector}
f : \mc{D} \to \R^{10} \\
v \mapsto 
( I(v), 
\nabla_x I(v), 
\nabla_y I(v), 
\nabla_z I(v), 
\sqrt{2}\nabla_{xy} I(v),
\sqrt{2}\nabla_{yz} I(v),
\nabla_{xx} I(v), 
\nabla_{yy} I(v), 
\nabla_{zz} I(v)
)^\top\ .
\end{gather}
assembled from the intensity $I(v)$ as well as first- and 
second-order responses of derivative filters capturing information from larger 
scales following \cite{HASHIMOTO:1987aa}. To improve the segmentation accuracy 
we combine the derivative filter 
responses from various scales in an computationally efficient way we first 
normalize the derivatives of the input volume $I(v)$ at every scale $\sigma_s$ 
by convolution each 
dimension with a 1D window:
\begin{align}
	\nabla_x \tilde{I}_{\sigma_s}(v) = \sigma_s^{2}\frac{\partial}{\partial 
	x}\tilde{G}(v,\sigma_s)
\end{align}
where $\tilde{G}(v,\sigma_s)$ is an approximation to Gaussian window 
$\big(G(v,\sigma_s)\ast I\big)(v)$ at scale $\sigma_s$ as in detail described 
in 
\cite{HASHIMOTO:1987aa}. Subsequently we follow the idea presented by 
\cite{Lindeberg:2004aa} by taking local maxima over scales
\begin{align}
	\nabla_x \tilde{I}(v) = \max_{\sigma_s} \nabla_x \tilde{I}_{\sigma_s}(v), 
\end{align}
which are serving for the mapping \eqref{Feature_Vector}. 

By introducing a suitable geometric graph spanning $\mc{D}$, we can associate a 
neighborhood $\mathcal{N}_i$ of fixed size with each voxel $i\in [n]$ as in 
\eqref{eq:def-Si}.
For each neighborhood, we define the regularized \textit{region covariance descriptor}
\begin{align}
\label{Cov_Des_Feature}
S_i := \sum_{j \in \mathcal{N}_i}\theta_{ij} 
(f_j-\overline{f_i})(f_j-\overline{f_i})^T + \epsilon I, \quad 
\overline{f_i} = \sum_{k \in \mathcal{N}_i}\theta_{ik}f_k,
\end{align}
as a weighted empirical covariance matrix with respect to feature vectors 
$f_{j}$. The small value $1 \gg \epsilon > 0$ acts as a regularization 
parameter enforcing positive definiteness of $S_i$. 
Diagonal entries of each covariance matrix $C_i$ are empirical
variances of feature channels in \eqref{Feature_Vector} while the 
off-diagonal entries represent empirical correlations within the region 
$\mathcal{N}_i$.

\subsubsection{Prototypes on $\mathcal{P}^d$}\label{sec:Stein_Div_Prototypes_Expirements}
In view of the assignment flow framework introduced in Section \ref{sec:Assignment-Flow},
we interpret region covariance descriptors \eqref{Cov_Des_Feature} as 
data points in the metric space $\mathcal{P}^d$ of 
symmetric positive definite matrices and model each retina tissue indexed by $l \in [c]$ with a random variable $S_l$ taking values in $\mathcal{P}^d$. 
Suppose we draw $N_l$ samples $\{S_l^k \}_{k = 1}^{N_l}$ from the distribution of $S_l$.
The most basic way to apply assignment flows to data in $\mathcal{P}^d$ is based on computing a prototypical element of $\mathcal{P}^d$ for each tissue layer, e.g. the Riemannian center of mass of $\{S_l^k \}_{k = 1}^{N_l}$. This corresponds to directly choosing $\mathcal{P}^d$ as feature space $\mc{F}$ in \eqref{eq:def-mcF-n}.
We find that superior empirical results are achieved by considering a dictionary of $K_l > 1$ prototypical elements for each layer $l\in [c]$.
This entails partitioning the samples $\{S_l^k \}_{k = 1}^{N_l}$ into $K_l$ disjoint subsets 
$\hat{S}^j_l \subseteq \{S_l^k \}_{k = 1}^{N_l}$, $j\in [K_l]$
with representatives $\tilde{S}^j_l$ determined offline. 

To find a set of representatives which captures the structure of the data, we minimize expected loss measured by the Stein divergence \eqref{Stein_div} leading to the $K$-means like functional 
\begin{align}
\label{K_means}
\EE_{p_l}(\tilde{S}_{l}) = \sum^{K_l}_{j = 1} p(j) \sum_{S_l^i \in \hat{S}_l^j}\frac{p(i|j)}{p(j)} 
D_S(S_l^i,\tilde{S}^{j}_l), \quad p(i,j) = \frac{1}{N_l},p_l(j) = 
\frac{N_j}{N_l}\ .
\end{align}
A hard partitioning is achieved by applying Lloyd's algorithm in conjunction 
with algorithm \ref{Algorithm Riemannian Mean S} for mean retrieval.
We additionally employ the more 
common soft K-means like approach for determining prototypes by employing the 
mixture exponential family model based on Stein divergence to given data  
\begin{align}
\label{EM}
p(S_l^i, \Gamma_l) = \sum_{j = 1}^K \pi_{l}^j p(S_l^i,\tilde{S}_l^j)),  
\end{align}
where the parameters 
\begin{align}
	\Gamma_l 
	= \{(\pi_l^{j}\}_{j = 1}^K,\{\tilde{S}_l^j\}_{j = 1}^K), \quad 
	(\pi_l^1,\cdots, \pi_l^{|J|}) \in S 
\end{align}
have to be adjusted to given data. The prototypes are recovered as mean 
parameters $S_l^{j,T}$ though an 
iterative process commonly refered to as \textit{expectation maximation } (EM) 
defined by alternation of the following iterations 
\begin{align}
	p_l(j|S_l^i,\Gamma^t_l) = 
	\frac{\pi_l^{(j,t)}e^{-D_S(S_l^i,\tilde{S}_l^{(j,t)})}}{\sum_{k = 
	1}\pi_l^{(k,t)}e^{-D_S(S_l^i,\tilde{S}_l^{(k,t)})}}, \hspace{1.9cm} 
	\textbf{\textit{(Expectation 
	step)}}
\end{align} 
followed by updating the marginals at each time step up to final time $T$
\begin{align}
\pi_l^{(j,t+1)} &= \sum_{i = 1}^{N_j} p_l(j|S_l^i,\Gamma^t_l)
\tilde{S}^{j,t}\\
\tilde{S}^{j,t+1}_l &=  \argmin_{S \in \mathcal{P}_d} \Big( \sum_{i =1}^n 
p(j|\Gamma_i^t) 
D_{S}(S^i_l,S)\Big), \qquad \textbf{(\textit{Maximization step)}}.
\end{align}

The decision to approximate the Riemannian metric on $\mc{P}_d$ by the Stein 
divergence \eqref{Stein_div} can be backed up empirically. To this end, we 
randomly select descriptors \eqref{Cov_Des_Feature} representing the nerve 
fibre layer in real-world OCT data and compute their Riemannian mean as well as 
their mean w.r.t. the Log-Euclidian metric \eqref{Log_Euclid} and Stein 
divergence \eqref{Stein_div}. Figure \ref{fig:Running_Time} illustrates that 
Stein divergence approximates the full Riemannian metric more precisely than 
the Log-Euclidian metric while still achieving a significant reduction in 
computational effort. Furthermore to evaluate the classification we extracted a 
dictionary of $200$ prototypes for representing each retina tissue for 
different choice of metric and subsequently evaluated the resulting 
segmentation accuracy by assigning each 
voxel to a class containing the prototype with smallest distance using a 
cropped 
OCT Volume of size $138\times100\times40$ taken from the testing set.

 Figure \ref{fig:Metric_Classification} visualizes the correct classification 
 matches for retina layers ordered by color according to Figure 
 \ref{OCT_Acquisition}. In particular, we inspect a notable gain of correct 
 matches while respecting the 
Riemannian geometry (first column) as 
opposed to Log-Euclidean setting (third column). Regarding the approximation of 
\eqref{eq:def-g-mcPd} by \eqref{Stein_div}, we are observing more effective 
detection of outer Photoreceptor Layer (PR1), Inner Nuclear Layer (INL) and 
Retinal Pigment Epithelium (RPE). Furthermore, taking a closer look at (OPL) 
and (ONL) we note a typical tradeoff between the number of prototypes and 
detection performance indicating superior retina to 
voxel allocation by applying \eqref{Log_Euclid},  whereas the surrogate 
divergence metric \eqref{Stein_div} has the tendency to improve the accuracy 
while increasing the size of evaluated prototypes in contrast to flattening 
curves when relying on \eqref{eq:LE-mean}. 
 
\begin{figure}
	\centering
	\begin{subfigure}[t]{0.32\textwidth}
		\scalebox{0.5}{\begin{tikzpicture}
\begin{axis}[
		nodes near coords,
		every other node near coord/.append style={font=\tiny},
		legend columns=4,height=8cm,width = 10cm,
%		legend entries={{\tiny Random},{\tiny +Cost},{\tiny +FTE},{\tiny 
%		++Cost},{\tiny ++FTE},{\tiny ++Resources}, {\tiny ++Cost$_{tri}$}, 
%		{\tiny ++FTE$_{tri}$}, {\tiny ++Resources$_{tri}$}, {\tiny 
%		++Cost$_{LN}$}, {\tiny ++FTE$_{expo}$}, {\tiny ++Resources$_{expo}$}},
		mark repeat=2,
		xlabel = Number of prototypes,
		ylabel = True positive classification rate,
		xmin=0,
		xmax=200,
		scaled ticks=false,
		/pgf/number format/.cd,
		use comma,
		1000 sep={}
		]
	%\addplot[line join=round,mark = star,color = Layer_3]
	%	coordinates {(05,9035) (20,13959) (40,15868) (60,16159) (80,16912) 
	%	(100,16681) (120,16840) (140,17341) (180,17747) (200,17768)};
	%]
	\addplot[line join=round,mark = star,color = Layer_4]
	coordinates {(05,7947) (20,12906) (40,12967) (60,13432) (80,12778) 
		(100,12935) (120,12585) (140,12816) (180,12896) (200,13216)};
	]			
	\addplot[line join=round,mark = star,color = Layer_5]
	coordinates {(05,6585) (20,11103) (40,14556) (60,15535) (80,15494) 
		(100,15760) (120,15916) (140,15528) (180,15429) (200,16262)};
	]
	\addplot[line join=round,mark = star,color = Layer_6]
	coordinates {(05,7613) (20,9378) (40,9446) (60,9653) (80,9908) (100,9913) 
	(120,9929) (140,10407) (160,9753) 
		(180,10235) (200,10315)};
	]	
	%\addplot[line join=round,mark = star,color = Layer_7]
	%coordinates {(05,13811) (20,32061) (40,34210) (60,32958) 
	%	(80,34231) (100,34503) (120,34158) (140,35486) (180,34852) (200,35263)};
	%]
	%\addplot[line join=round,mark = star,color = Layer_8]
	%coordinates {(05,12811) (20,16590) (40,19910) (60,20523) (80,20769) 
	%	(100,21209) (120, 21705) (140,21717) (180,21722) (200,21832)};
	%]
%	\addplot[line join=round,mark = star,color = Layer_9]
%	coordinates {(05,5496) (20,10595) (40,14192) (60,11519) (80,14222) 
%	(100,16159)
%		(120,14570) (140,15837) (160,15558) (180,15498) (200,17469)};	
	\addplot[line join=round,mark = star,color = Layer_10]
coordinates {(05,10412) (20,12277) (40,12676) (60,9416) (80,8257) (100,8737)
	(120,9009) (140,9295) (160,9313) (180,9493) (200, 9892)};
%		\addplot table[line join=round,col sep=comma, y=PlusCost, 
%		x=Day]{DollarCommitment.csv};
		%\addlegendentry{{\scriptsize +Cost}}
		\end{axis}
\end{tikzpicture}
%		\addplot table[line join=round,col sep=comma, y=PlusFTE, 
%		x=Day]{DollarCommitment.csv};
%		\addlegendentry{{\scriptsize +FTE}}
%		\addplot table[line join=round,col sep=comma, y=PlusPlusCost, 
%		x=Day]{DollarCommitment.csv};
%		\addlegendentry{{\scriptsize ++Cost}}
%		\addplot table[line join=round,col sep=comma, y=PlusPlusFTE, 
%		x=Day]{DollarCommitment.csv};
%		\addlegendentry{{\scriptsize ++FTE}}
%		\addplot table[line join=round,col sep=comma, y=PlusPlusResources, 
%		x=Day]{DollarCommitment.csv};
%		\addlegendentry{{\scriptsize ++Resources}}
%		\addplot table[line join=round,col sep=comma, y=PlusPlusCostTri, 
%		x=Day]{DollarCommitment.csv};
%		\addlegendentry{{\scriptsize ++Cost$_{tri}$}}
%		\addplot table[line join=round,col sep=comma, y=PlusPlusFTETri, 
%		x=Day]{DollarCommitment.csv};
%		\addlegendentry{{\scriptsize ++FTE$_{tri}$}}
%		\addplot table[line join=round,col sep=comma, y=PlusPlusResourcesTri, 
%		x=Day]{DollarCommitment.csv};
%		\addlegendentry{{\scriptsize ++Resources$_{tri}$}}
%		\addplot table[line join=round,col sep=comma, y=PlusPlusCostLN, 
%		x=Day]{DollarCommitment.csv};
%		\addlegendentry{{\scriptsize ++Cost$_{LN}$}}
%		\addplot table[line join=round,col sep=comma, y=PlusPlusFTEExpo, 
%		x=Day]{DollarCommitment.csv};
%		\addlegendentry{{\scriptsize ++FTE$_{expo}$}}
%		\addplot table[line join=round,col sep=comma, y=PlusPlusResourcesExpo, 
%		x=Day]{DollarCommitment.csv};
%		\addlegendentry{{\scriptsize ++Resources$_{expo}$}}
}
	\end{subfigure}
\begin{subfigure}[t]{0.32\textwidth}
	\scalebox{0.5}{\begin{tikzpicture}
\begin{axis}[nodes near coords,
every other node near coord/.append style={font=\tiny},
legend columns=4,height=8cm,width = 10cm,
		legend columns=4,
%		legend entries={{\tiny Random},{\tiny +Cost},{\tiny +FTE},{\tiny 
%		++Cost},{\tiny ++FTE},{\tiny ++Resources}, {\tiny ++Cost$_{tri}$}, 
%		{\tiny ++FTE$_{tri}$}, {\tiny ++Resources$_{tri}$}, {\tiny 
%		++Cost$_{LN}$}, {\tiny ++FTE$_{expo}$}, {\tiny ++Resources$_{expo}$}},
		mark repeat=2,
		xlabel = Number of prototypes,
		ylabel = True positive classification rate,
		xmin=0,
		xmax=200,
		scaled ticks=false,
		/pgf/number format/.cd,
		use comma,
		1000 sep={}
		]
%	\addplot[line join=round,mark = star,color = Layer_3]
%		coordinates {(5, 8609)
%			(20, 11510)
%			(40, 15289)
%			(60, 15643)
%			(80, 15959)
%			(100, 16227)
%			(120, 16173)
%			(140, 16851)
%			(160, 17103)
%			(180, 17228)
%			(200, 17437)};
%	]
	\addplot[line join=round,mark = star,color = Layer_4]
	coordinates {(5, 8627)
		(20, 11880)
		(40, 12708)
		(60, 12083)
		(80, 11790)
		(100, 11824)
		(120, 11607)
		(140, 11752)
		(160, 12128)
		(180, 12190)
		(200, 13553)};
	]			
	\addplot[line join=round,mark = star,color = Layer_5]
	coordinates {(5, 10974)
		(20, 12116)
		(40, 15291)
		(60, 15258)
		(80, 15264)
		(100, 14930)
		(120, 15131)
		(140, 15593)
		(160, 15475)
		(180, 15572)
		(200, 15821)};
	]
	\addplot[line join=round,mark = star,color = Layer_6]
	coordinates {(5, 8125)
		(20, 9301)
		(40, 8177)
		(60, 8607)
		(80, 8614)
		(100, 8217)
		(120, 8406)
		(140, 8640)
		(160, 9028)
		(180, 9102) (200,10201)};
	]
	
%	\addplot[line join=round,mark = star,color = Layer_7]
%	coordinates {(5, 26229) 
%		(20, 27770)
%		(40, 30841)
%		(60, 32965)
%		(80, 33966)
%		(100, 33281)
%		(120, 33979)
%		(140, 34497)
%		(160, 34296)
%		(180, 33068) (200,34194)};
%	]
	
%	\addplot[line join=round,mark = star,color = Layer_8]
%	coordinates {(05,14022) (20,16674) (40,19228) (60,20670) (80,20599) 
%		(100,20706) (120, 20522) (140,20541) (180,20635) (200,20707)};
%	]
%	\addplot[line join=round,mark = star,color = Layer_9]
%	coordinates {(05,9358) (20,9765) (40,8957) (60,13784) (80,15851) 
%	(100,15963)
%		(120,16990) (140,16970) (160,16835) (180,17099) (200,17236)};
	
	\addplot[line join=round,mark = star,color = Layer_10]
coordinates {(05,4491) (20,11025) (40,12111) (60,10517) (80,9890) (100,9956)
	(120,10020) (140,9916) (160,9582) (180,9448) (200, 9193)};
%		\addplot table[line join=round,col sep=comma, y=PlusCost, 
%		x=Day]{DollarCommitment.csv};
	%	\addlegendentry{{\scriptsize +Cost}}
		\end{axis}
\end{tikzpicture}}
\end{subfigure}
\begin{subfigure}[t]{0.32\textwidth}
	\scalebox{0.5}{\begin{tikzpicture}
\begin{axis}[nodes near coords,
every other node near coord/.append style={font=\tiny},
legend columns=4,height=8cm,width = 10cm,
		legend columns=4,
%		legend entries={{\tiny Random},{\tiny +Cost},{\tiny +FTE},{\tiny 
%		++Cost},{\tiny ++FTE},{\tiny ++Resources}, {\tiny ++Cost$_{tri}$}, 
%		{\tiny ++FTE$_{tri}$}, {\tiny ++Resources$_{tri}$}, {\tiny 
%		++Cost$_{LN}$}, {\tiny ++FTE$_{expo}$}, {\tiny ++Resources$_{expo}$}},
		mark repeat=2,
		xlabel = Number of prototypes,
		ylabel = True positive classification rate,
		xmin=0,
		xmax=200,
			scaled ticks=false,
	/pgf/number format/.cd,
	use comma,
	1000 sep={}
		]
%	\addplot[line join=round,mark = star,color = Layer_3]
%		coordinates {(5, 9825)
%			(20, 10945)
%			(40, 15643)
%			(60, 16768)
%			(80, 16879)
%			(100, 16988)
%			(120, 16968)
%			(140, 16058)
%			(160, 17094 )
%			(180, 17133)
%			(200, 17206)};
%	]
	\addplot[line join=round,mark = star,color = Layer_4]
	coordinates {(5, 7896)
		(20, 12433)
		(40, 12378)
		(60, 12818)
		(80, 12184)
		(100, 11979)
		(120, 11959)
		(140, 12225)
		(160, 12366)
		(180, 12338)
		(200, 12360)};
	]			
	\addplot[line join=round,mark = star,color = Layer_5]
	coordinates {(5, 5870)
		(20, 10701)
		(40, 13408)
		(60, 14568)
		(80, 14123)
		(100, 14720)
		(120, 14398)
		(140, 14241)
		(160, 14321)
		(180, 14568)
		(200, 14694)};
	]
	\addplot[line join=round,mark = star,color = Layer_6]
	coordinates {(5, 7624)
		(20, 9581)
		(40, 9485)
		(60, 9634)
		(80, 9133)
		(100, 8958)
		(120, 8816)
		(140, 9115)
		(160, 9298)
		(180, 9346)
		(200, 9277)};
	]
	
%	\addplot[line join=round,mark = star,color = Layer_7]
%	coordinates {(5, 12378)
%		(20, 31535)
%		(40, 34411)
%		(60, 33559)
%		(80, 33107)
%		(100, 33361)
%		(120, 33692)
%		(140, 35194)
%		(160, 34679)
%		(180, 33886)
%		(200, 33503)};
%	]
%	\addplot[line join=round,mark = star,color = Layer_8]
%	coordinates {(5, 20070)
%		(20, 17032)
%		(40, 21916)
%		(60, 21951)
%		(80, 22063)
%		(100, 22039)
%		(120, 22304)
%		(140, 22472)
%		(160, 22399)
%		(180, 22425)
%		(200, 22146)};
%	]
%	\addplot[line join=round,mark = star,color = Layer_9]
%	coordinates {(5, 5542)
%		(20, 16249)
%		(40, 13323)
%		(60, 14673)
%		(80, 14697)
%		(100, 14942)
%		(120, 15699)
%		(140, 15748)
%		(160, 15550)
%		(180, 15653)
%		(200, 15798)};
%]	
	\addplot[line join=round,mark = star,color = Layer_10]
coordinates {(5, 11948)
	(20, 8691)
	(40, 7809)
	(60, 8202)
	(80, 8323)
	(100, 8322)
	(120, 8671)
	(140, 8688)
	(160, 8799)
	(180, 8846)
	(200, 9196)};
%		\addplot table[line join=round,col sep=comma, y=PlusCost, 
%		x=Day]{DollarCommitment.csv};
	%	\addlegendentry{{\scriptsize +Cost}}
		\end{axis}
\end{tikzpicture}}
\end{subfigure}
	\begin{subfigure}[t]{0.32\textwidth}
	\scalebox{0.5}{\begin{tikzpicture}
\begin{axis}[
		nodes near coords,
		every other node near coord/.append style={font=\tiny},
		legend columns=4,height=8cm,width = 10cm,
%		legend entries={{\tiny Random},{\tiny +Cost},{\tiny +FTE},{\tiny 
%		++Cost},{\tiny ++FTE},{\tiny ++Resources}, {\tiny ++Cost$_{tri}$}, 
%		{\tiny ++FTE$_{tri}$}, {\tiny ++Resources$_{tri}$}, {\tiny 
%		++Cost$_{LN}$}, {\tiny ++FTE$_{expo}$}, {\tiny ++Resources$_{expo}$}},
		mark repeat=2,
		xlabel = Number of prototypes,
		ylabel = True positive classification rate,
		xmin=0,
		xmax=200,
		scaled ticks=false,
		/pgf/number format/.cd,
		use comma,
		1000 sep={}
		]
	\addplot[line join=round,mark = star,color = Layer_3]
		coordinates {(05,9035) (20,13959) (40,15868) (60,16159) (80,16912) 
		(100,16681) (120,16840) (140,17341) (180,17747) (200,17768)};
	]
	\addplot[line join=round,mark = star,color = Layer_7]
	coordinates {(05,13811) (20,32061) (40,34210) (60,32958) 
		(80,34231) (100,34503) (120,34158) (140,35486) (180,34852) (200,35263)};
	]
	\addplot[line join=round,mark = star,color = Layer_8]
	coordinates {(05,12811) (20,16590) (40,19910) (60,20523) (80,20769) 
		(100,21209) (120, 21705) (140,21717) (180,21722) (200,21832)};
	]
	\addplot[line join=round,mark = star,color = Layer_9]
	coordinates {(05,5496) (20,10595) (40,14192) (60,11519) (80,14222) 
	(100,16159)
		(120,14570) (140,15837) (160,15558) (180,15498) (200,17469)};	
%		\addplot table[line join=round,col sep=comma, y=PlusCost, 
%		x=Day]{DollarCommitment.csv};
	%	\addlegendentry{{\scriptsize +Cost}}
		\end{axis}
\end{tikzpicture}}
	\end{subfigure}
	\begin{subfigure}[t]{0.32\textwidth}
		\scalebox{0.5}{\begin{tikzpicture}
\begin{axis}[nodes near coords,
every other node near coord/.append style={font=\tiny, inner ysep=0.5pt},
legend columns=4,height=8cm,width = 10cm,
		legend columns=4,
%		legend entries={{\tiny Random},{\tiny +Cost},{\tiny +FTE},{\tiny 
%		++Cost},{\tiny ++FTE},{\tiny ++Resources}, {\tiny ++Cost$_{tri}$}, 
%		{\tiny ++FTE$_{tri}$}, {\tiny ++Resources$_{tri}$}, {\tiny 
%		++Cost$_{LN}$}, {\tiny ++FTE$_{expo}$}, {\tiny ++Resources$_{expo}$}},
		mark repeat=2,
		xlabel = Number of prototypes,
		ylabel = True positive classification rate,
		xmin=0,
		xmax=200,
		scaled ticks=false,
		/pgf/number format/.cd,
		use comma,
		1000 sep={}
		]
	\addplot[line join=round,mark = star,color = Layer_3]
		coordinates {(5, 8609)
			(20, 11510)
			(40, 15289)
			(60, 15643)
			(80, 15959)
			(100, 16227)
			(120, 16173)
			(140, 16851)
			(160, 17103)
			(180, 17228)
			(200, 17437)};
	]
%	\addplot[line join=round,mark = star,color = Layer_4]
%	coordinates {(5, 8627)
%		(20, 11880)
%		(40, 12708)
%		(60, 12083)
%		(80, 11790)
%		(100, 11824)
%		(120, 11607)
%		(140, 11752)
%		(160, 12128)
%		(180, 12190)
%		(200, 13553)};
%	]			
%	\addplot[line join=round,mark = star,color = Layer_5]
%	coordinates {(5, 10974)
%		(20, 12116)
%		(40, 15291)
%		(60, 15258)
%		(80, 15264)
%		(100, 14930)
%		(120, 15131)
%		(140, 15593)
%		(160, 15475)
%		(180, 15572)
%		(200, 15821)};
%	]
%	\addplot[line join=round,mark = star,color = Layer_6]
%	coordinates {(5, 8125)
%		(20, 6830)
%		(40, 8177)
%		(60, 8607)
%		(80, 8614)
%		(100, 8217)
%		(120, 8406)
%		(140, 8640)
%		(160, 8869)
%		(180, 8785) (200,9101)};
	]	
	\addplot[line join=round,mark = star,color = Layer_7]
	coordinates {(5, 26229) 
		(20, 27770)
		(40, 30841)
		(60, 32965)
		(80, 33966)
		(100, 33281)
		(120, 33979)
		(140, 34497)
		(160, 34296)
		(180, 33068) (200,34194)};
	]	
	\addplot[line join=round,mark = star,color = Layer_8]
	coordinates {(05,14022) (20,16674) (40,19228) (60,20670) (80,20599) 
		(100,20706) (120, 20522) (140,20541) (180,20635) (200,20707)};
	]
	\addplot[line join=round,mark = star,color = Layer_9]
	coordinates {(05,9358) (20,9765) (40,8957) (60,13784) (80,15851) 
	(100,15963)
		(120,16990) (140,16970) (160,16835) (180,17099) (200,17236)};
%		\addplot table[line join=round,col sep=comma, y=PlusCost, 
%		x=Day]{DollarCommitment.csv};
		%\addlegendentry{{\scriptsize +Cost}}
		\end{axis}
\end{tikzpicture}}
	\end{subfigure}
	\begin{subfigure}[t]{0.32\textwidth}
		\scalebox{0.5}{\begin{tikzpicture}
\begin{axis}[nodes near coords,
every other node near coord/.append style={font=\tiny},
legend columns=4,height=8cm,width = 10cm,
		legend columns=4,
%		legend entries={{\tiny Random},{\tiny +Cost},{\tiny +FTE},{\tiny 
%		++Cost},{\tiny ++FTE},{\tiny ++Resources}, {\tiny ++Cost$_{tri}$}, 
%		{\tiny ++FTE$_{tri}$}, {\tiny ++Resources$_{tri}$}, {\tiny 
%		++Cost$_{LN}$}, {\tiny ++FTE$_{expo}$}, {\tiny ++Resources$_{expo}$}},
		mark repeat=2,
		xlabel = Number of prototypes,
		ylabel = True positive classification rate,
		xmin=0,
		xmax=200,
			scaled ticks=false,
	/pgf/number format/.cd,
	use comma,
	1000 sep={}
		]
	\addplot[line join=round,mark = star,color = Layer_3]
		coordinates {(5, 9825)
			(20, 10945)
			(40, 15643)
			(60, 16768)
			(80, 16879)
			(100, 16988)
			(120, 16968)
			(140, 16058)
			(160, 17094 )
			(180, 17133)
			(200, 17206)};
	]
%	\addplot[line join=round,mark = star,color = Layer_4]
%	coordinates {(5, 7896)
%		(20, 12433)
%		(40, 12378)
%		(60, 12818)
%		(80, 12184)
%		(100, 11979)
%		(120, 11959)
%		(140, 12225)
%		(160, 12366)
%		(180, 12338)
%		(200, 12360)};
%	]			
%	\addplot[line join=round,mark = star,color = Layer_5]
%	coordinates {(5, 5870)
%		(20, 10701)
%		(40, 13408)
%		(60, 14568)
%		(80, 14123)
%		(100, 14720)
%		(120, 14398)
%		(140, 14241)
%		(160, 14321)
%		(180, 14568)
%		(200, 14694)};
%	]
%	\addplot[line join=round,mark = star,color = Layer_6]
%	coordinates {(5, 9158)
%		(20, 9581)
%		(40, 9485)
%		(60, 9634)
%		(80, 9133)
%		(100, 9020)
%		(120, 9081)
%		(140, 9115)
%		(160, 9298)
%		(180, 9346)
%		(200, 9277)};
%	]
	
	\addplot[line join=round,mark = star,color = Layer_7]
	coordinates {(5, 12378)
		(20, 31535)
		(40, 34411)
		(60, 33559)
		(80, 33107)
		(100, 33361)
		(120, 33692)
		(140, 35194)
		(160, 34679)
		(180, 33886)
		(200, 33503)};
	]
	\addplot[line join=round,mark = star,color = Layer_8]
	coordinates {(5, 20070)
		(20, 17032)
		(40, 21916)
		(60, 21951)
		(80, 22063)
		(100, 22039)
		(120, 22304)
		(140, 22472)
		(160, 22399)
		(180, 22425)
		(200, 22146)};
	]
	\addplot[line join=round,mark = star,color = Layer_9]
	coordinates {(5, 5542)
		(20, 16249)
		(40, 13323)
		(60, 14673)
		(80, 14697)
		(100, 14942)
		(120, 15699)
		(140, 15748)
		(160, 15550)
		(180, 15653)
		(200, 15798)};	
%		\addplot table[line join=round,col sep=comma, y=PlusCost, 
%		x=Day]{DollarCommitment.csv};
	%	\addlegendentry{{\scriptsize +Cost}}
		\end{axis}
\end{tikzpicture}}
	\end{subfigure}
\captionsetup{font=footnotesize}
\caption{\textbf{Top}: Metric classification evaluated on thin layers  
(IPL,INL,OPL,PR2). \textbf{Bottom}: Analogous metric evaluation for 
(GCL,ONL,PR1,RPE). \textbf{From left to right}: The number of true 
outcomes after direct comparison with ground truth, for the choice of the exact 
Riemannian geometry of $\mc{P}_d$, Stein divergence and Log-Euclidean distance 
for geometric mean 
computation. The results of first two columns indicate higher detection 
performance while respecting the Riemannian geometry of a curved manifold.   
 Enlarging the set of prototypical covariance descriptors leads to increased 
 matching accuracy which is in contrast to the observed flattening of matching 
 curves when using the Log-Euclidean distance.
 } 
 %from \eqref{sec:Log_Euclid}}.
\label{fig:Metric_Classification}
\end{figure}
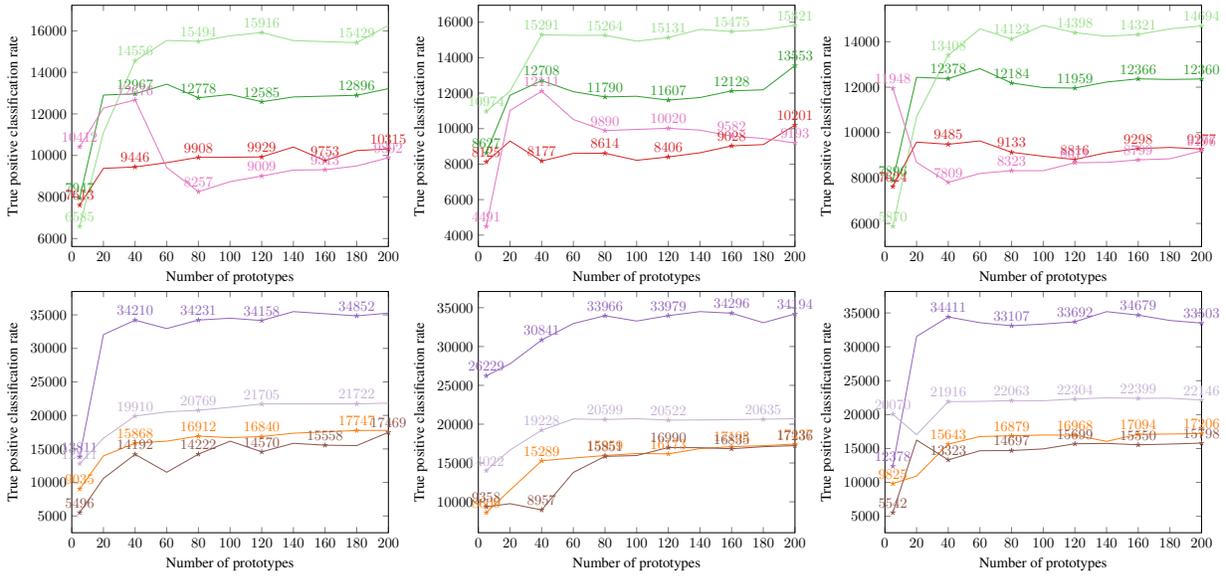

\begin{figure}[ht] 
	\begin{subfigure}[t]{0.48\textwidth}
		\begin{tikzpicture}
		\node[anchor=south,scale=1] 
	at (0,0) 
	{\includegraphics[width=8cm,height=5cm]{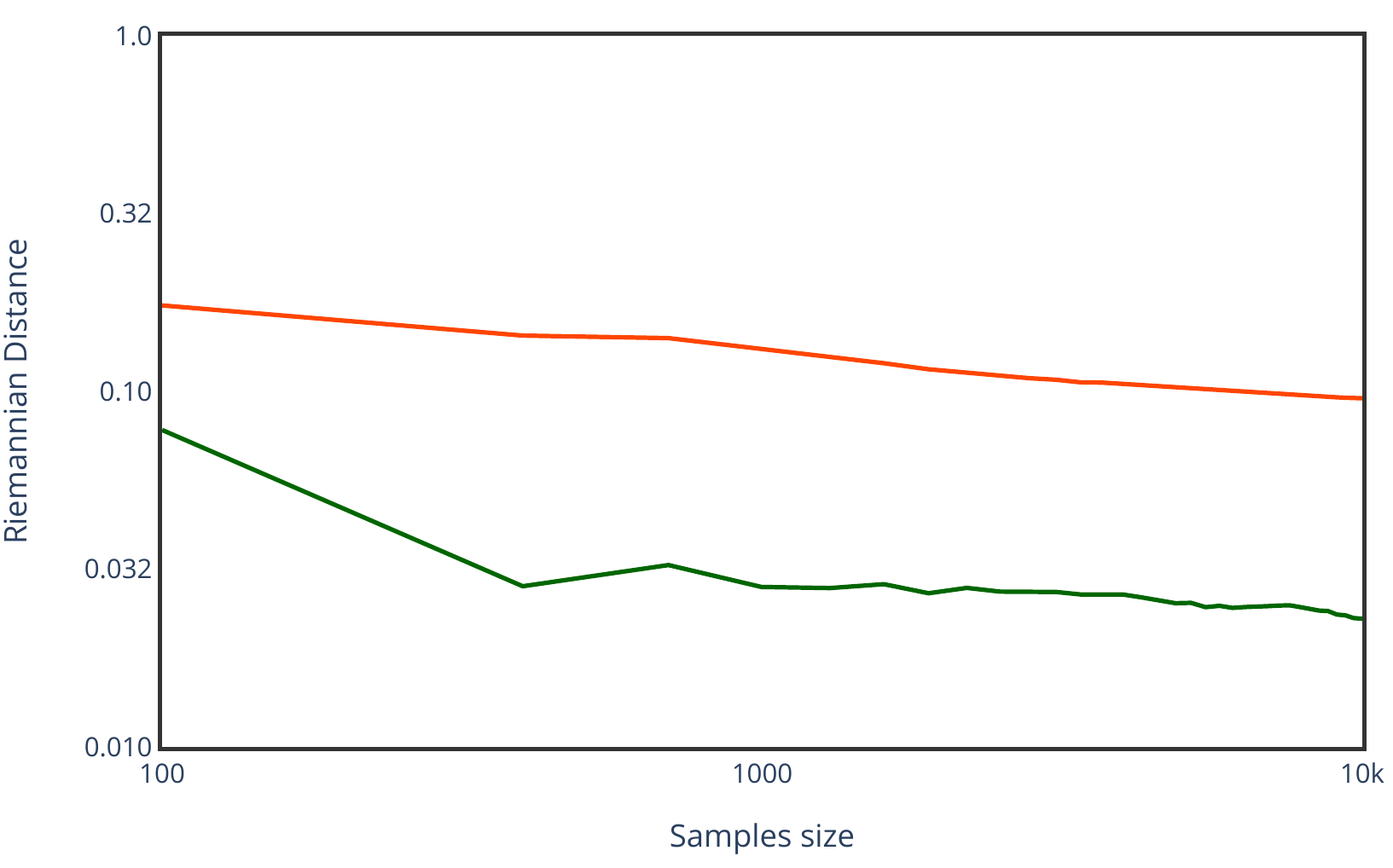}}; 
	\node[anchor=south,scale=0.5] 
	at (2.87,4.38)
	{
		\renewcommand\arraystretch{1}
		\begin{tabular}{|cp{2.8cm}|}
		\hline
		& \normalsize $d_R(M_{R},M_{Stein})$ 
		\\
		& \normalsize $d_R(M_{R},M_{Euclid})$ \\
		\hline
		\end{tabular}
	};
	\path[draw,thick, color = Stein_line] (2.1,4.8) -- (2.25,4.8);
	\path[draw,thick, color = Euc_line]   (2.1,4.58) -- (2.25,4.58);
		\end{tikzpicture}
	\end{subfigure}
	\begin{subfigure}[t]{0.48\textwidth}
		\begin{tikzpicture}
		\node[anchor=south,scale=1] 
		at (5,0) 
		{\includegraphics[width=8cm,height=5cm]{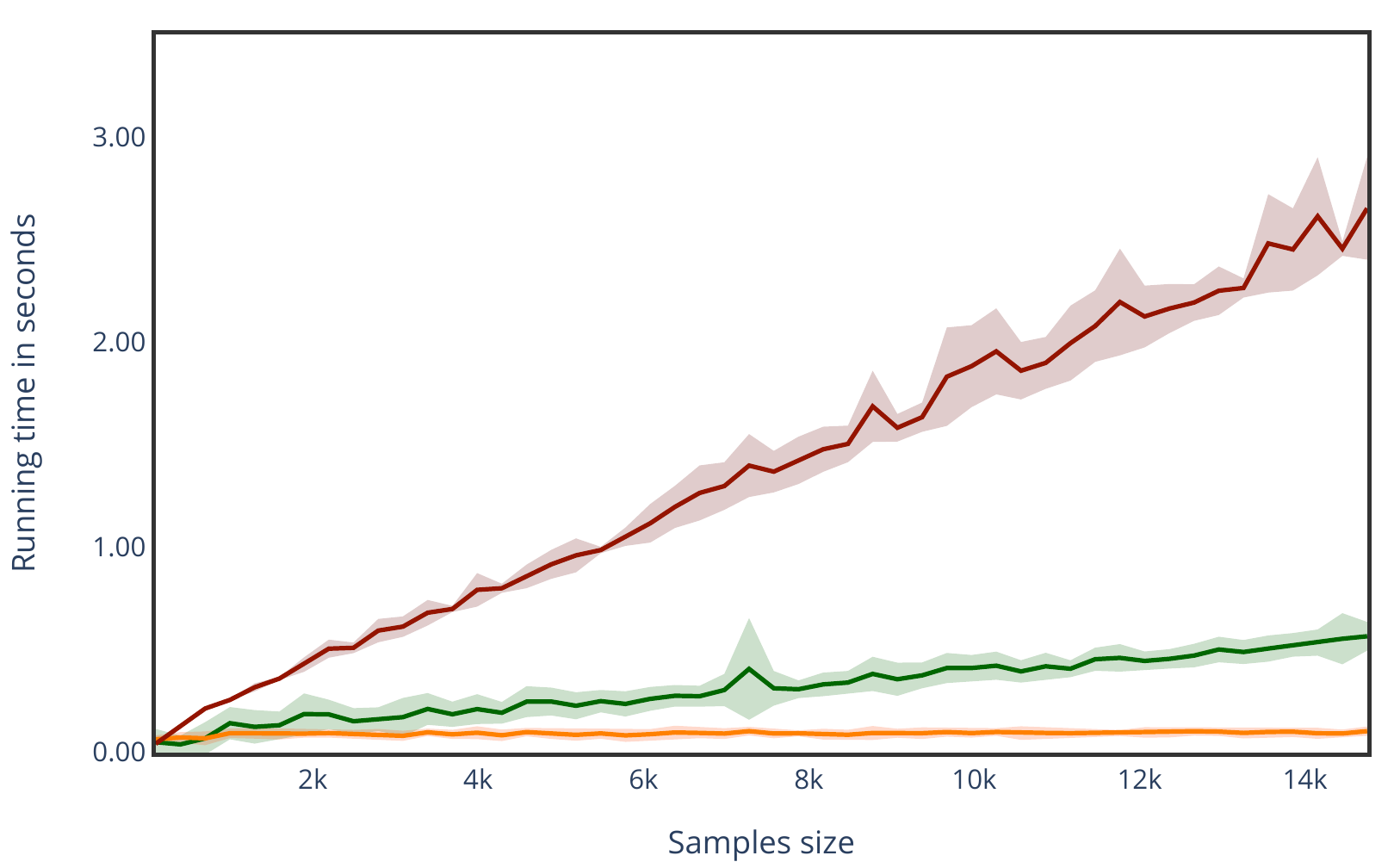}}; 
		\node[anchor=south,scale=0.5] 
		at (7.9,4.15)
		{
			\renewcommand\arraystretch{1}
			\begin{tabular}{|cp{2.8cm}|}
			\hline
			& \normalsize \textit{\hspace{0.5cm} Riemann} 
			\\
			& \normalsize \textit{\hspace{0.5cm} Log-Euclid}  \\
			& \normalsize \textit{\hspace{0.5cm} Stein}  
			\\
			\hline
			\end{tabular}
		};
		\path[draw,thick, color = Riemann] (7.2,4.8) -- (7.5,4.8);
		\path[draw,thick, color = orange]   (7.2,4.58) -- (7.5,4.58);
		\path[draw,thick, color = Stein_line] (7.2,4.36) -- (7.5,4.36);
		\end{tikzpicture}
	\end{subfigure}
	\captionsetup{font=footnotesize}
	\caption{\textbf{Left}: 
		Deviation of the geometric means computed using the Log-Euclidian metric and Stein divergence, respectively, from the true Riemannian mean.
		\textbf{Right}: Runtime for geometric mean computation using the different 
		metrics. All evaluations were performed on a randomly chosen subset of 
		covariance descriptors representing the retinal nerve fibre layer in a 
		real-world OCT scan. Both graphics clearly highlight the advantages of 
		using Stein the divergence in terms of approximation accuracy and efficient 
		numerical computation.}
	\label{fig:Running_Time}
\end{figure}
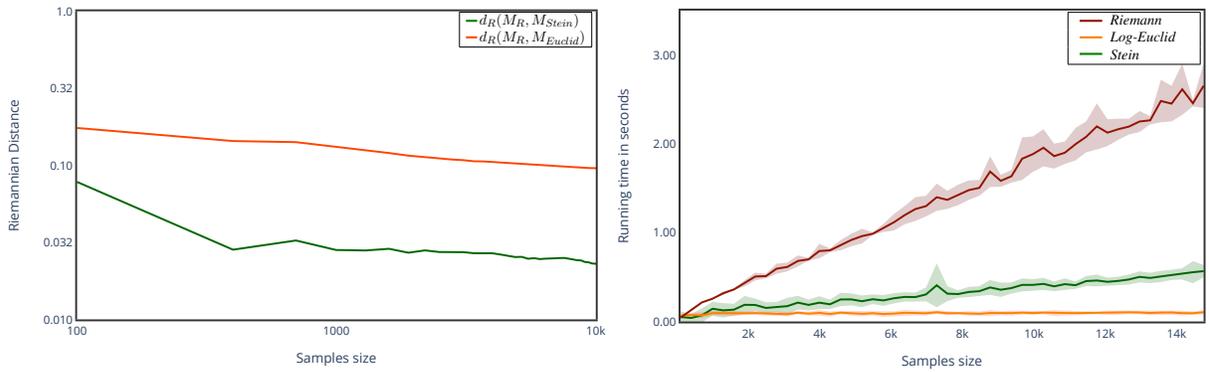

This illustrates a tradeoff between computational effort and labeling 
performance, cf. Figure \ref{fig:Running_Time}. Note that prototypes are 
computed offline, making runtime performance less relevant to medical 
practitioners. However, building a distance matrix involves computing 
$n\sum_{l\in [c]} K_l$ Riemannian distances resp. Stein divergences to 
prototypes. This still leads to a large difference in (online) runtime since 
evaluation of the Riemannian distance \eqref{eq:def-g-mcPd} involves 
generalized eigendecomposition while less costly Cholesky decomposition 
suffices to evaluate the Stein divergence \eqref{Stein_div}. 

Summarizing the discussed results concerning the application of Algorithm 
\ref{Algorithm 
	Riemannian Mean} and Algorithm \ref{Algorithm Riemannian Mean S}, we 
	point out that respecting the Riemannian geometry leads to superior 
	labeling results providing more descriptive prototypes. 
\begin{figure}[ht] 
	\begin{subfigure}[t]{0.49\textwidth}
		\centering
		\includegraphics[width=8.2cm,height=4.1cm]{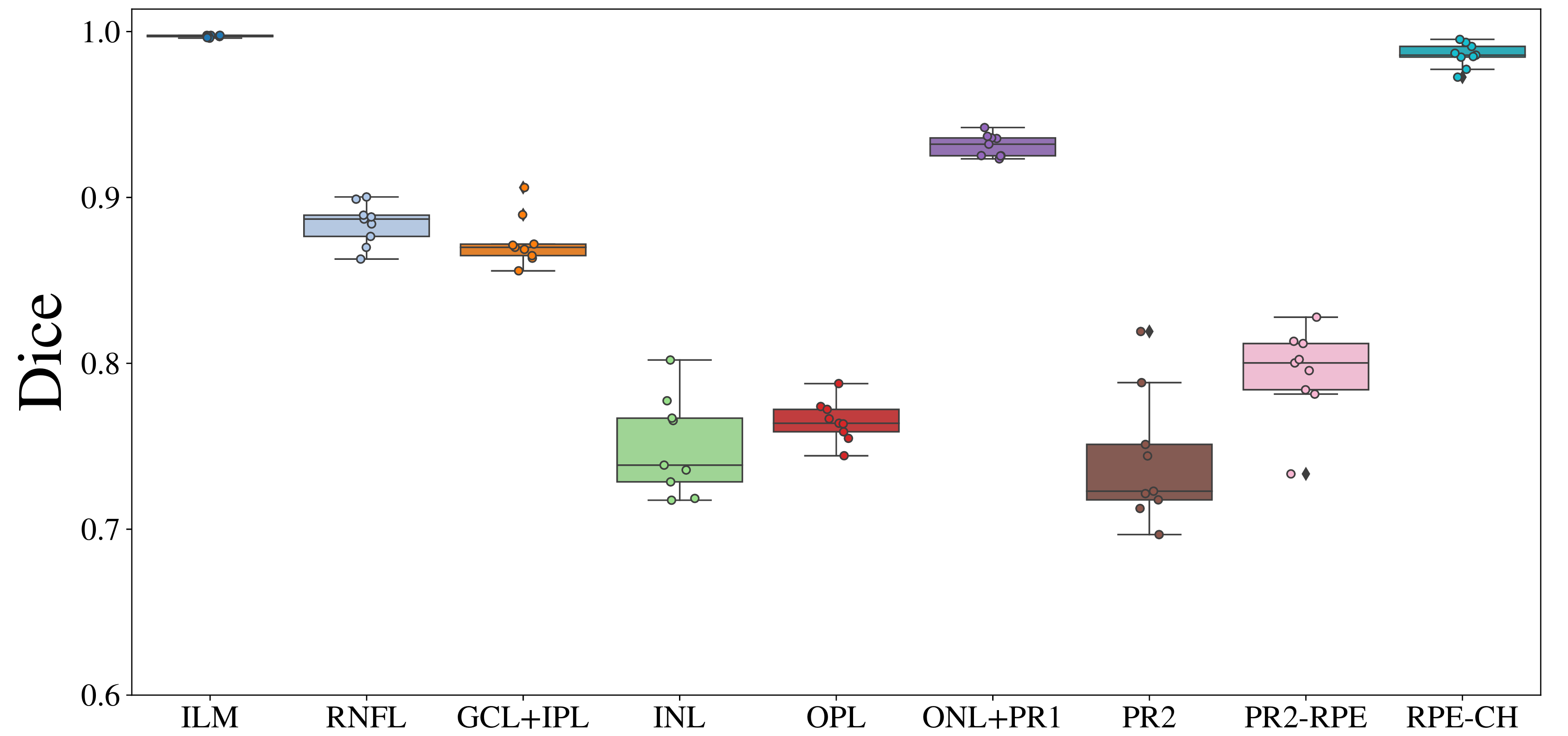}
	\end{subfigure}
	\begin{subfigure}[t]{0.49\textwidth}
		\centering
		\includegraphics[width=8.2cm,height=4.1cm]{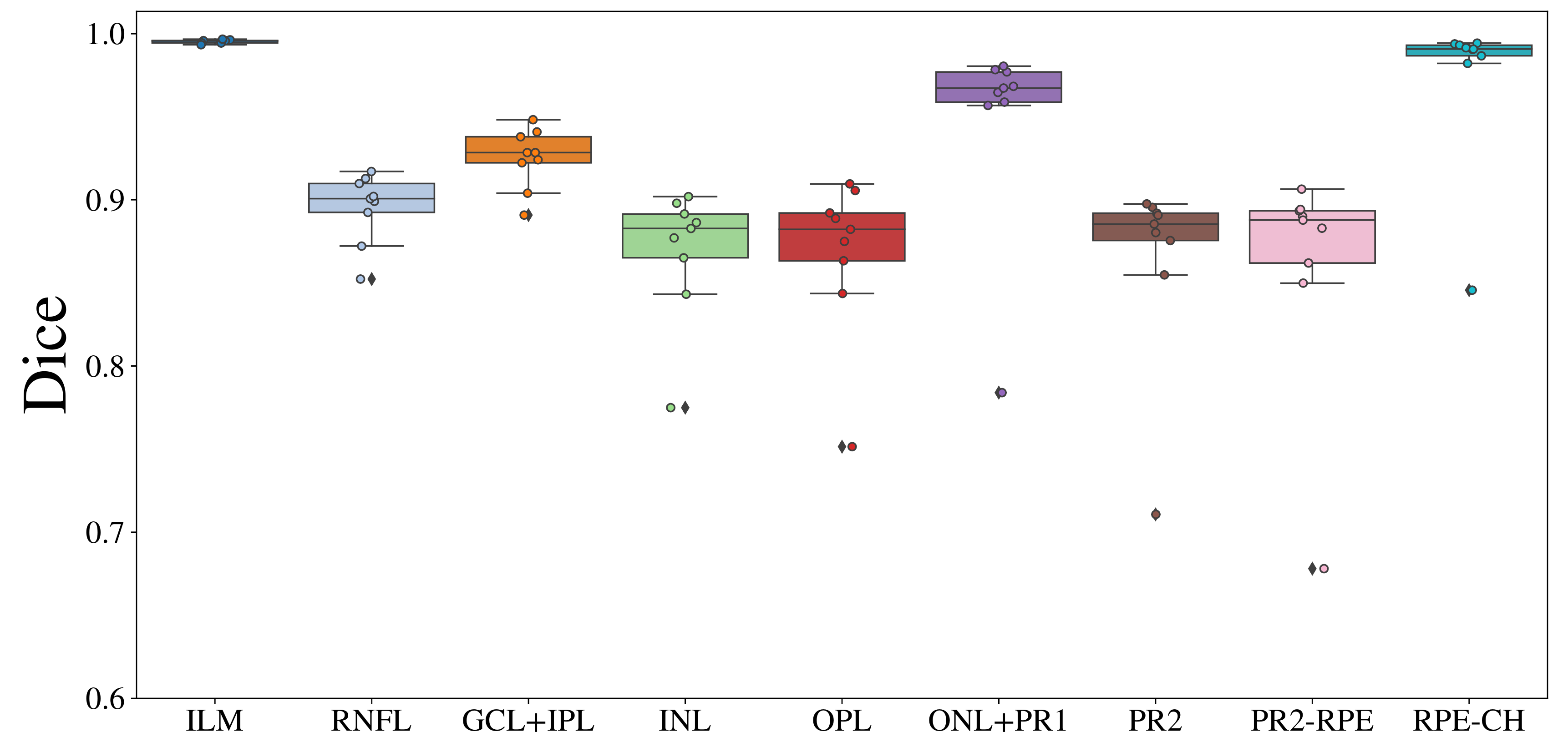}
	\end{subfigure}
	\captionsetup{font=footnotesize}
	\caption{Box plots of DICE similarity coefficients between computed 
	segmentation results and manually labeled ground truth. \textbf{Left}: 
	Probabilistic approach \eqref{Graphical_Model} proposed in 
	\cite{Rathke2014}.
	\textbf{Right}: OAF based on CNN features. See Table \ref{tab:Table_Rathke} 
	for mean and standard deviations. Direct comparison shows a notably higher 
	detection performance for segmenting the intraretinal layers using OAF (B).}
	\label{fig:Rathke_Boxplot}
\end{figure}

\begin{figure}[ht] 
	\begin{subfigure}[t]{0.49\textwidth}
		\centering
		\includegraphics[width=8.2cm,height=4.1cm]{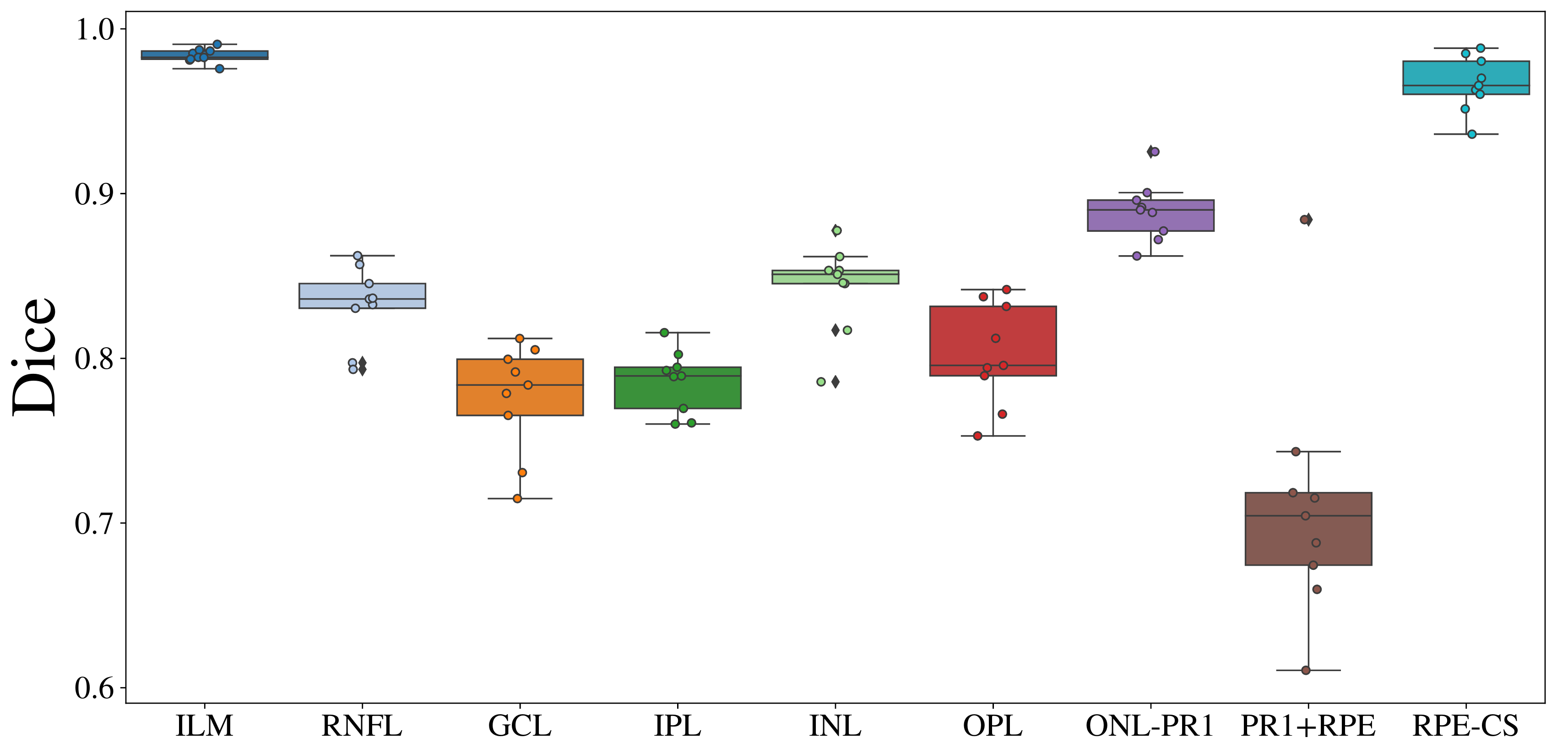}
	\end{subfigure}
	\begin{subfigure}[t]{0.49\textwidth}
		\centering
		\includegraphics[width=8.2cm,height=4.1cm]{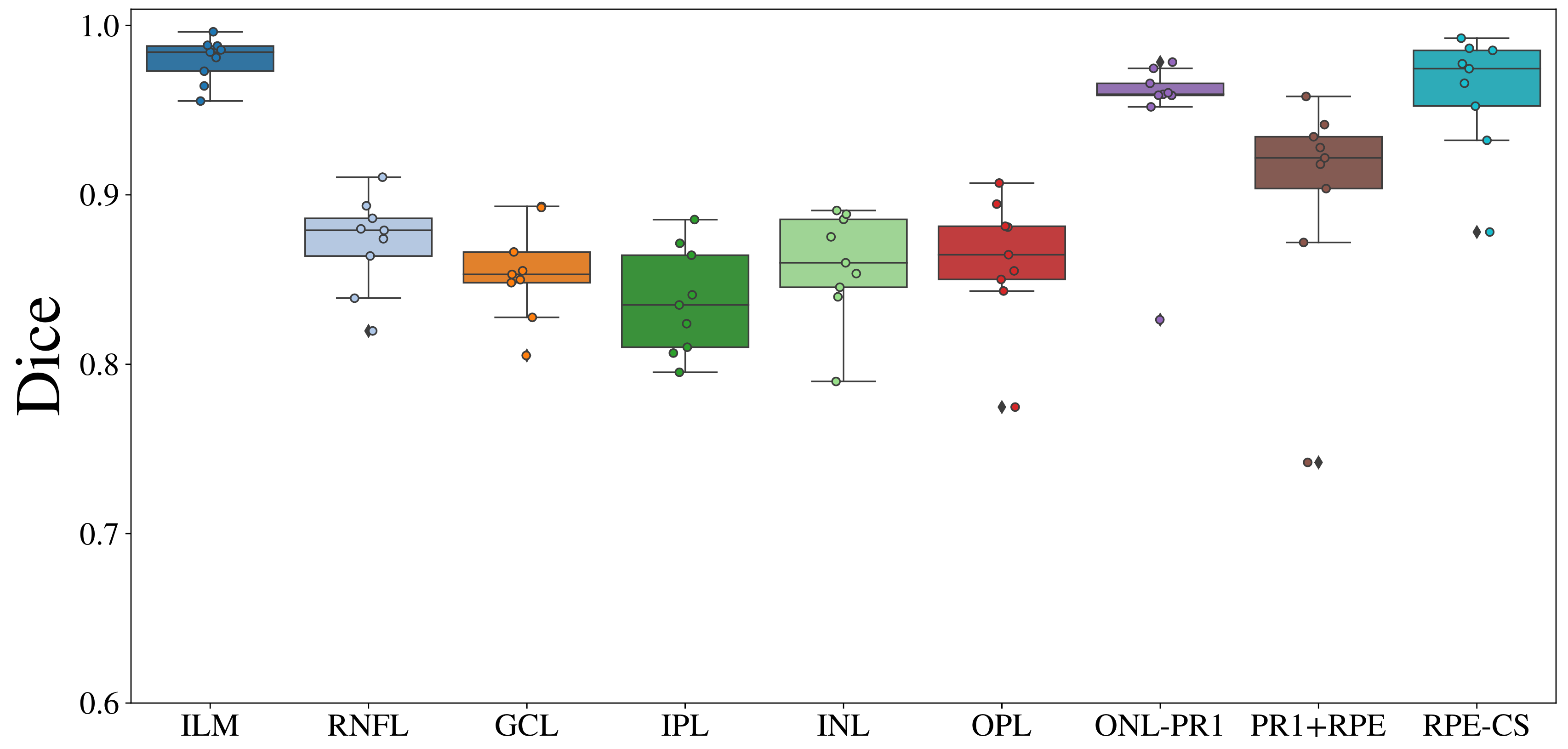}
	\end{subfigure}
    \captionsetup{font=footnotesize}
	\caption{Box plots of DICE similarity coefficients between computed 
	segmentation results and manually labeled ground truth. \textbf{Left}: IOWA 
	reference algorithm \cite{IOWA}.
	\textbf{Right}: OAF based on CNN features. See Table \ref{tab:Table_IOWA} 
	for mean and standard deviations. Exploiting OAF (B) for retina tissue 
	classification results in improved overall layer detection performance, 
	especially for the PR2-RPE region.}\label{fig:IOWA_Boxplot}
\end{figure}

\begin{figure}[ht!] 
	\begin{subfigure}[t]{0.49\textwidth}
		\centering
		\includegraphics[width=8.2cm,height=4.1cm]{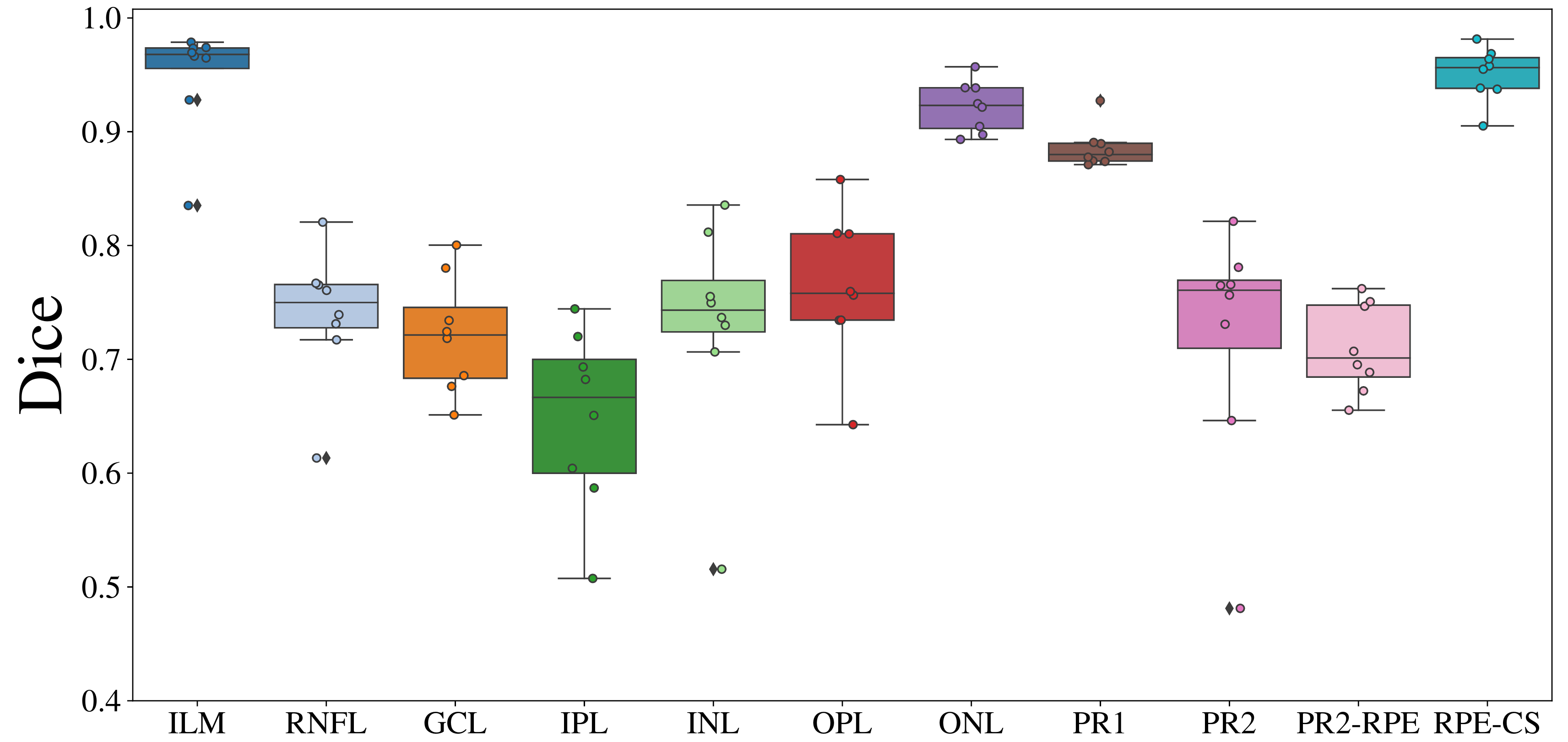}
	\end{subfigure}
	\begin{subfigure}[t]{0.49\textwidth}
		\centering
		\scalebox{1}{\includegraphics[width=8.2cm,height=4.1cm]{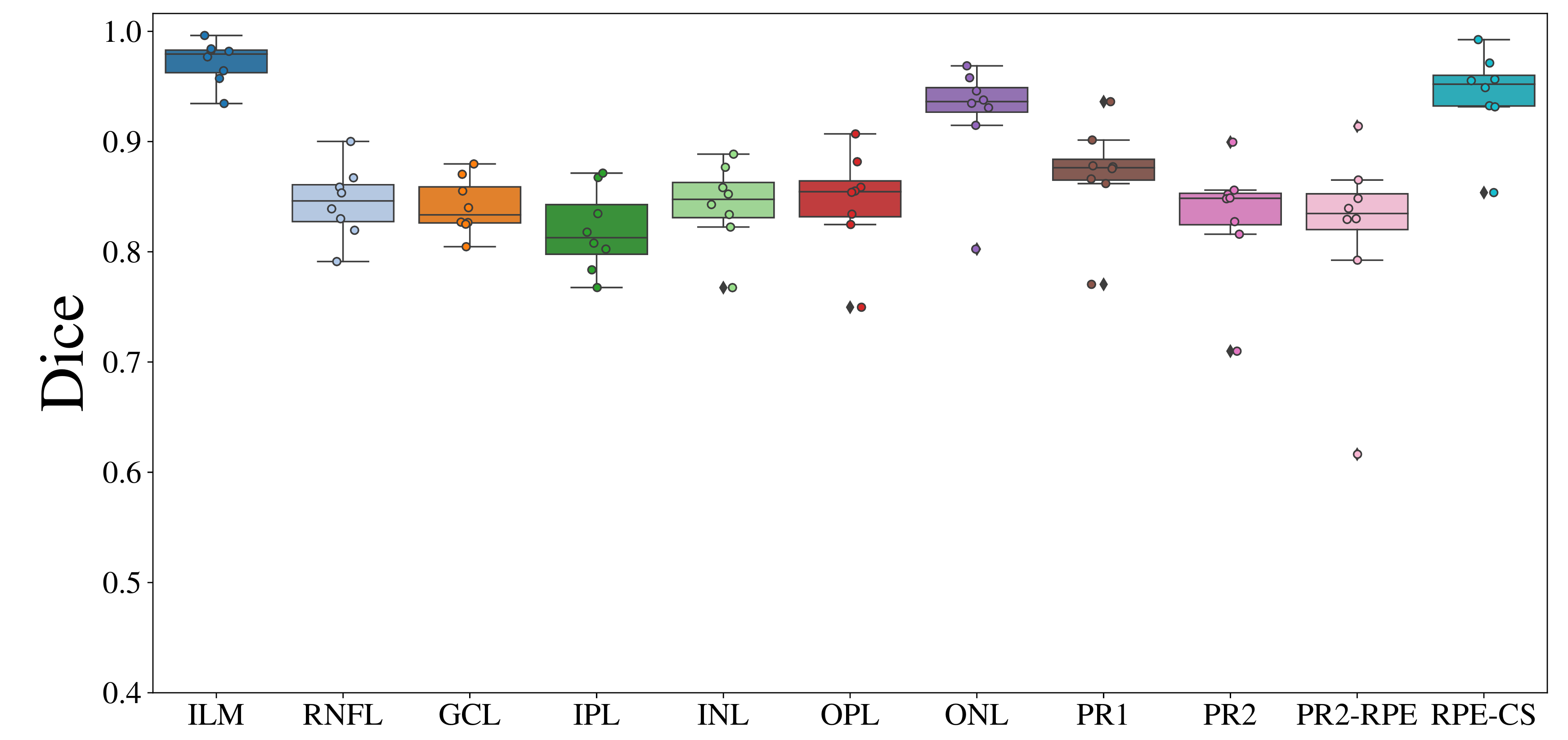}}
	\end{subfigure}
	\captionsetup{font=footnotesize}
	\caption{Box plots of DICE similarity coefficients between computed 
	segmentation results and manually labeled ground truth. \textbf{Left}: OAF 
	(A).
	\textbf{Right}: OAF (B). The OAF based on CNN features yields improved segmentations for all 
	retina layers.}
\label{fig:Cov_Descriptor_Boxplot}
\end{figure}

%\begin{figure}[ht]
%	\begin{subfigure}[t]{0.42\textwidth}
%		\centering
%		\includegraphics[width=6cm,height=6cm]{Graphics/ContourLayer_1.png}
%	\end{subfigure}
%	\begin{subfigure}[t]{0.42\textwidth}
%		\centering
%		\includegraphics[width=6cm,height=6cm]{Graphics/ContourLayer_3.png}
%	\end{subfigure}
%\begin{subfigure}[t]{0.42\textwidth}
%	\centering
%	\includegraphics[width=6cm,height=6cm]{Graphics/Save_Contour_Layer2.png}
%\end{subfigure}
%\begin{subfigure}[t]{0.42\textwidth}
%	\centering
%	\includegraphics[width=6cm,height=6cm]{Graphics/ContourLayer_4.png}
%\end{subfigure}
%	\caption{Retina layers boundaries}
%\end{figure}

\subsubsection{CNN Features}\label{sec:experiments_cnn}
In addition to the covariance features in 
Section \ref{sec:experiments_cov_descr}, we
compare a second approach to local feature extraction based on a convolutional 
neural network
architecture. For each node $i\in [n]$, we trained the network to directly 
predict the correct class in
$[c]$ using raw intensity values in $\mc{N}_i$ as input. As output, we find a 
score for each layer
which can directly be transformed into a distance vector suitable as input to 
the ordered assignment flow \eqref{eq:ordered_af} via 
\eqref{eq:ordering_likelihood}.
The specific network used in our experiments has a ResNet architecture 
comprising four residually 
connected blocks of 3D convolutions and ReLU activation. Model size was 
hand-tuned for different sizes of input neighborhoods,
adjusting the number of convolutions per block as well as corresponding channel 
dimensions. In particular, labeling accuracy is increased for the detection of RPE 
and PR2 
layers, as illustrated in the last row of Figure \ref{Labeling_Cov}.

\subsection{Segmentation via Ordered Assignment}

By numerically integrating the ordered assignment flow \eqref{def:ordered_af} 
parametrized by the distance matrix $D$, an assignment state $W$ is evolved on 
$\mc{W}$ until mean entropy of pixel assignments is low. We specifically use 
geometric Euler integration steps on $T\mc{W}$ with a constant step-length of 
$h=0.1$ (see \cite{Zeilmann:2020aa} for details of this process). 
Geometric averaging with uniform weights leads to local regularization of 
assignments which smooths regions in which the features do not 
conclusively point to any label. More global knowledge about the ordering of 
cell layers is incorporated into $E_\text{ord}$ which addresses more severe 
inconsistencies between local features and global ordering.
In all experiments, the neighborhood of each voxel $i\in [n]$ is choosen as the 
voxel patch of size $5$ x $5$ x $3$ centered at $i$.

\subsection{Evaluation}

To benchmark our novel segmentation approach, we first extract local features 
for each voxel from a raw OCT volume. As described above, 
either region covariance descriptors (Section \ref{sec:experiments_cov_descr}) 
or class scores predicted by a CNN (Section \ref{sec:experiments_cnn}) are 
computed for segmenting the retina layers with ordered assignment flow which we 
in the following abbreviate as \textbf{OAF (A)} and \textbf{OAF (B)} 
respectively. 
In the former case, a set of $k=400$ prototypical cluster centers on the 
positive definite cone \eqref{eq:def-mcPd} has been determined offline for each 
cell layer. These 
are compared to descriptors extracted from the unseen volume by computing the 
pair-wise distance with respect to the metric induced by Stein divergence 
\eqref{sec:S-distance}. The minimum value corresponding to the lowest of all 
the distances for each pair of voxel 
$i\in [n]$ and cell layer $j\in [c]$ is noted as entry $d_{ij}$ of the distance 
matrix $D_\text{cov}$, i.e. for every voxel $i$ the distance to optimal fitting 
representative to layer $j$ is given by  
\begin{align}
	(D_{\text{cov}})_{ij} := \min_{k \in [400]} D_S(S_i,\tilde{S}_{j}^k).
\end{align}
In the latter case, class scores $C\in \R^{n\times c}$ predicted by the 
neuronal network \eqref{sec:experiments_cnn} are transformed into a distance 
matrix $D_\text{cnn} = -C$ simply by switching their sign followed by adjusting 
the parameter \eqref{schnoerr-eq:def-Li} to weight the data relevance in the 
likelihood matrix.

A naive way to segment the volume in accordance to the data is by choosing 
$\arg\min_{j\in [c]} D_{ij}$ for each voxel $i$. However, due to the 
challenging signal-to-noise ratio in real-world OCT data, classes will not 
usually be well-separated in the feature space at hand. The resulting 
uncertainty pertaining to the assignment of classes using exclusively local 
features is encoded into each distance matrix. To assess the segmentation 
performance of our proposed approach, we first compared 
to the state of the art graph-based retina segmentation method of 10 
intra-retinal layers developed by the Retinal Image Analysis Laboratory at the 
Iowa Institute 
for Biomedical Imaging \cite{Li:2006aa,Abramow:2010aa,Garvin:2008aa}, also 
referred to as the IOWA Reference Algorithm. We quantify the region agreement 
with manual segmentation regarded as gold standard. 
%. For each
%trained classifier, we generated segmentations for the 10− N s
%test subjects not used in training
Since both the augmented 
volumes and the compared reference methods determine boundary locations of 
retina layers intersections, we first transfer the retina surfaces to a layer 
mask by rounding to the voxel size and assign to voxels within each A-scan the 
associated layer label, starting from the observed boundary to the location of 
the next detected intersection surface of two neighboring layers. 

Specifically, we calculate 
the DICE similarity coefficient \cite{Dice:1945aa} and the mean absolute error 
for segmented cell layers within the pixel size of \mum{3.87} compared to human 
grader by segmenting 8 OCT volumes consisting of $61$ B-scans. We first 
directly compare the performance accuracy of using the local features given by 
the covariance descriptor \eqref{sec:experiments_cov_descr} by constructing a 
dictionary with $400$ prototypes for each retina layer using the iterative 
clustering with \eqref{EM} against the features extracted from an CNN net 
\eqref{sec:experiments_cnn}.

\begin{figure}[ht]
	\begin{adjustwidth}{-4.5cm}{}
		\centering
		{\raisebox{10mm}{\begin{subfigure}[t]{1.6em}
					\caption[singlelinecheck=off]{}
			\end{subfigure}}\ignorespaces}
		\begin{subfigure}[t]{0.7\textwidth}
			\centering
			\scalebox{1}{\includegraphics[height=0.2\textwidth,width=1.4\textwidth]{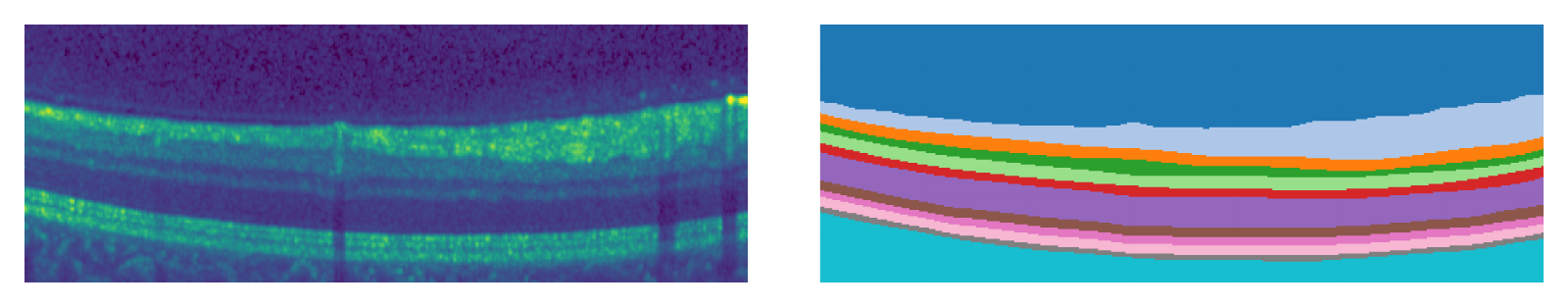}
}		\end{subfigure}\\
{\raisebox{10mm}{\begin{subfigure}[t]{1.6em}
			\caption[singlelinecheck=off]{}
	\end{subfigure}}\ignorespaces}
		\begin{subfigure}[t]{0.7\textwidth}
			\centering
			\scalebox{1}{\includegraphics[height=0.2\textwidth,width=1.4\textwidth]{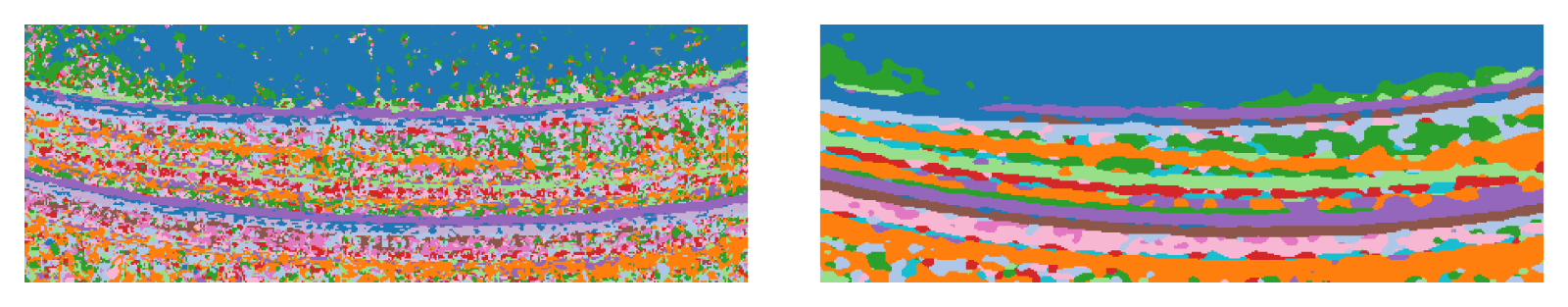}\hfill
}		\end{subfigure}\\
{\raisebox{10mm}{\begin{subfigure}[t]{1.6em}
			\caption[singlelinecheck=off]{}
	\end{subfigure}}\ignorespaces}
		\begin{subfigure}[t]{0.7\textwidth}
			\centering
			\scalebox{1}{\includegraphics[height=0.2\textwidth,width=1.4\textwidth]{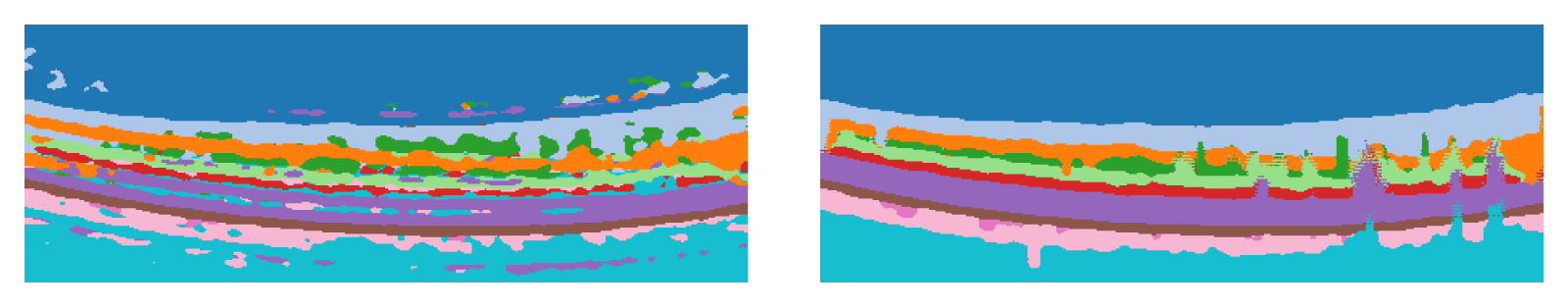}\hfill
}		\end{subfigure}\\
{\raisebox{10mm}{\begin{subfigure}[t]{1.6em}
			\caption[singlelinecheck=off]{}
	\end{subfigure}}\ignorespaces}
		\begin{subfigure}[t]{0.7\textwidth}
			\centering
			\scalebox{1}{\includegraphics[height=0.2\textwidth,width=1.4\textwidth]{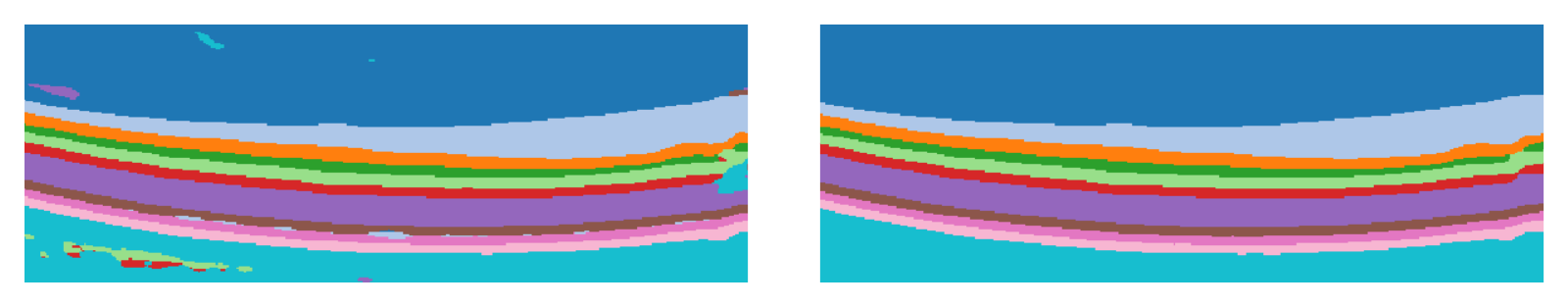}\hfill
}		\end{subfigure}
	\end{adjustwidth}
	\captionsetup{font=footnotesize}
	\caption{\textbf{From top to bottom: Row} \textbf{(a)}: One B-scan from a
	OCT-volume 
	showing the shadow effect,  
		with ground truth plot on the right. \textbf{Row} \textbf{(b)}: Local  
		nearest 
		neighbor 
		assignments based on prototypes by minimizing \eqref{K_means} computed 
		with Stein divergence, with the 
		result of the segmentation returned by the basic  assignment flow 
		(Section \ref{sec:Assignment-Flow}) on the right. \textbf{Row} 
		\textbf{(c)}: 
		Proposed \textit{layer-ordered} volume segmentation 
		based on covariance descriptors. From left to right:  ordered volume 
		segmentation for 
		different $\gamma = 0.5,\gamma = 0.1$ (cf.~Eq.~\eqref{gamma_par}). 
		\textbf{Row} \textbf{(d)}: Local rounding result extracted from 
		Res-Net on the left
		and the result of the ordered assignment flow on the right.}
	\label{Labeling_Cov}
\end{figure}

The experimental results discussed next illustrate the relative influence of 
the covariance descriptors  \eqref{Cov_Des_Feature} and regularization property 
of the ordered assignment flow, respectively.
Throughout, we fixed the grid 
connectivity $\mc{N}_{i}$ for each voxel $i\in\mc{I}$ to $3 \times 5 \times 5$. 
Figure \ref{Labeling_Cov} illustrates real-world labeling performance based 
on extracting a dictionary of $400$ prototypes per layer 
by minimizing \eqref{K_means} and employing
Algorithm  \ref{Algorithm Riemannian Mean S} for mean retrieval. Second row in 
Figure \ref{Labeling_Cov}, illustrates a typical result of nearest 
neighbor 
assignment and the volume segmentation \textit{without} ordering constraints. 
As inspected, the high texture similarity between the choroid and GCL  
layer yields wrong predictions resulting in violation of biological retina 
ordering through the whole volume which cannot be resolved with the based 
assignment flow approach given in Section \ref{sec:Assignment-Flow}. On the 
other side using pairwise correlations captured by covariance matrices provides 
accurate detection of large signal intensity internal limiting membrane (ILM) 
with its characteristic highly reflective 
boundary as well to meaningful segmentation of light rejecting fiber layers 
RNFL, PR1 and RPE. For the particularly challenging inner layers such as 
GCL, INL and ONL that are 
mainly comprised of weakly reflective neuronal cell bodies, regularization by 
imposing \eqref{eq:a_scan_order_energy} is required. In the third 
row of Figure \ref{Labeling_Cov}, we plot the 
\textit{ordered} volume segmentation by stepwise increasing the parameter 
$\gamma$ defined in \eqref{gamma_par}, which controls the ordering 
regularization 
by means of the novel generalized likelihood matrix 
\eqref{eq:ordering_likelihood}. 
The direct comparison with the ground truth  remarkably shows how the 
ordered labelings evolve on the assignment manifold while simultaneously 
giving accurate data-driven detection of RNFL, OPL, INL and the ONL layer. 
For the remaining critical inner layers, the local prototypes extracted by 
\eqref{K_means} fail to segment the retina layers properly and are still 
revealing artifacts due to the presence of vertical shadow 
regions caused by existing blood vessels, which contribute to a loss of the 
interference signal during the scanning process of the OCT-data, as shown in 
Figure 
\ref{Labeling_Cov}.

\begin{figure}[t!]
	\begin{adjustwidth}{-4.5cm}{}
		\centering
		{\raisebox{-10mm}{\begin{subfigure}[t]{1.6em}
					\caption[singlelinecheck=off]{}
			\end{subfigure}}\ignorespaces}
		\begin{subfigure}[t]{0.7\textwidth}
			\centering 
			\scalebox{1}{\includegraphics[height=0.2\textwidth,width=1.4\textwidth,valign=t]{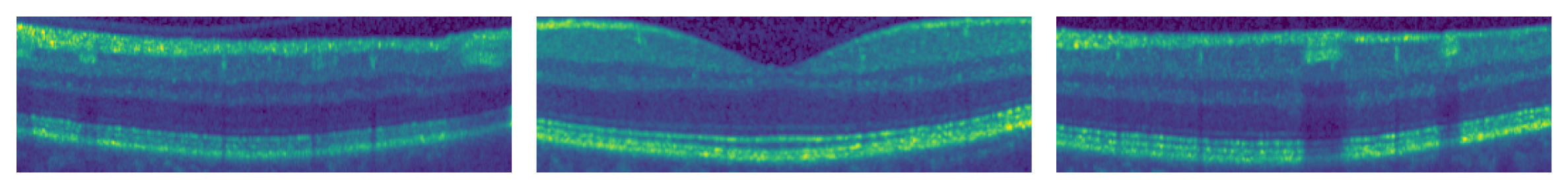}
		}		\end{subfigure}\\
		{\raisebox{10mm}{\begin{subfigure}[t]{1.6em}
					\caption[singlelinecheck=off]{}
			\end{subfigure}}\ignorespaces}
		\begin{subfigure}[t]{0.7\textwidth}
			\centering
			\scalebox{1}{\includegraphics[height=0.2\textwidth,width=1.4\textwidth]{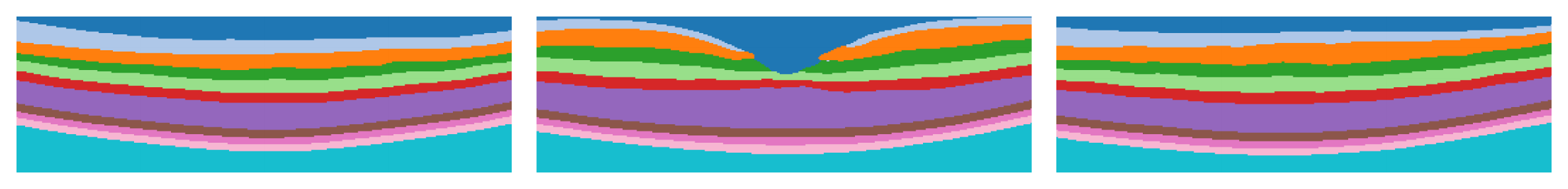}\hfill
		}		\end{subfigure}\\
		{\raisebox{10mm}{\begin{subfigure}[t]{1.6em}
					\caption[singlelinecheck=off]{}
			\end{subfigure}}\ignorespaces}
		\begin{subfigure}[t]{0.7\textwidth}
			\centering
			\scalebox{1}{\includegraphics[height=0.2\textwidth,width=1.4\textwidth]{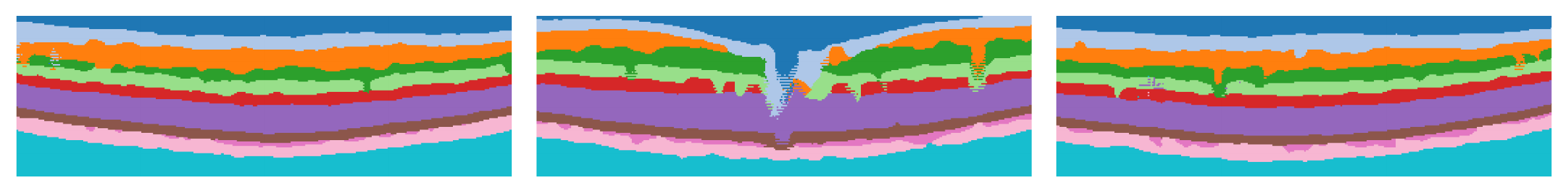}\hfill
		}		\end{subfigure}\\
		{\raisebox{10mm}{\begin{subfigure}[t]{1.6em}
					\caption[singlelinecheck=off]{}
			\end{subfigure}}\ignorespaces}
		\begin{subfigure}[t]{0.7\textwidth}
			\centering
			\scalebox{1}{\includegraphics[height=0.2\textwidth,width=1.4\textwidth]{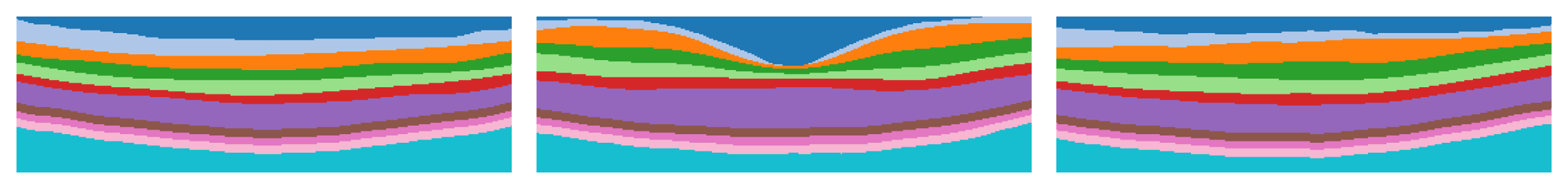}\hfill
		}		\end{subfigure}
	\end{adjustwidth}
	\captionsetup{font=footnotesize}
	\caption{\textbf{From top to bottom: Row (a)}: Three sample B-Scans 
		extracted 
		for different locations from a 
		healthy 
		OCT volume with 61 scans, with the fovea centered OCT scan visualized 
		in 
		the middle 
		column. \textbf{Row (b)}: The associated augmented labeling. 
		\textbf{Row (c)}: OAF (A) segmentation 
		using a dictionary of covariance descriptors determined by \eqref{EM}. 
		\textbf{Row (d)}: 
		OAF (B) segmentation using features determined the CNN network. In 
		contrast to 
		to results achieved by OAF (A), the above visualization indicates more 
		accurate detection of retina boundaries using OAF (B), in particular 
		near 
		the fovea region (middle column).}
	\label{Labeling_11_Layer}
\end{figure}

 After segmentation of the test data set, the mean and  
standard deviation were calculated for better assessment of the retina layer 
detection accuracy of the proposed segmentation method, according to the 
performance measures \eqref{MAE} and \eqref{DICE}. The evaluation results for 
each retina tissue as depicted in Figure \ref{OCT_Acquisition}, are
detailed in Table \ref{tab:Table}. The first row of Figure \ref{fig:Cov_Deep_Barplot} clearly shows the superior detection accuracy of 
utilizing the Ordered Assignment Flow for the first outer retina layers 
(RNFL, GCL, IPL, INL) and the (PR2-RPE) region in connection with local features 
extracted from an CNN net \eqref{sec:experiments_cnn}. Nonetheless, the 
covariance descriptor achieves comparable results for characterization of the 
outer plexiform layer (OPL) and exhibits increased retina detection regarding 
the photoreceptor region (PR1,PR2) and outer nuclear region (ONL). Additionally, 
Table \ref{tab:Table} includes the evaluation based on 
DICE similarity which is less sensitive to outliers and serves as an 
appropriate metric for calculating the performance measures across large 3D 
volumes. To obtain a consistent and clear comparability between the involved 
features on which we rely to tackle the specific problem of retina layer 
segmentation, the corresponding results are visualized in Figure 
\ref{fig:Cov_Descriptor_Boxplot}. The graphic illustrates higher Dice 
similarity and relative small standard deviation when incorporating features 
\eqref{sec:experiments_cnn} as input to 
our model, which characterizes their superior informative content. According to 
the left plot, the covariance descriptor 
performs well for detecting the prototypical textures of the internal limiting 
membrane (ILM), the (ONL) and (PR1) layers as well as the RPE boundary to the 
choroid section. Especially this highlights the ability of using gradient 
based features for accurate detection of retina tissues indicating sharp 
contrast between the neighboring layers, as is the case for ONL and PR1.

\begin{figure}[ht!]
	\begin{adjustwidth}{-4.5cm}{}
		\centering
		{\raisebox{10mm}{\begin{subfigure}[t]{1.6em}
					\caption[singlelinecheck=off]{}
			\end{subfigure}}\ignorespaces}
		\begin{subfigure}[t]{0.7\textwidth}
			\centering
			\scalebox{1}{\includegraphics[height=0.2\textwidth,width=1.4\textwidth]{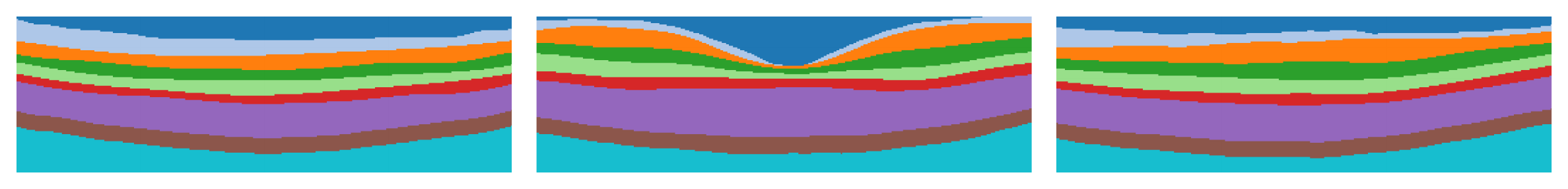}
		}		\end{subfigure}\\
		{\raisebox{10mm}{\begin{subfigure}[t]{1.6em}
					\caption[singlelinecheck=off]{}
			\end{subfigure}}\ignorespaces}
		\begin{subfigure}[t]{0.7\textwidth}
			\centering
			\scalebox{1}{\includegraphics[height=0.2\textwidth,width=1.4\textwidth]{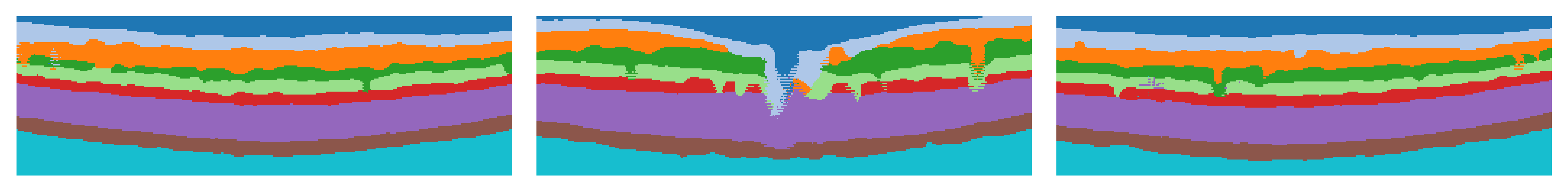}\hfill
		}		\end{subfigure}\\
		{\raisebox{10mm}{\begin{subfigure}[t]{1.6em}
					\caption[singlelinecheck=off]{}
			\end{subfigure}}\ignorespaces}
		\begin{subfigure}[t]{0.7\textwidth}
			\centering
			\scalebox{1}{\includegraphics[height=0.2\textwidth,width=1.4\textwidth]{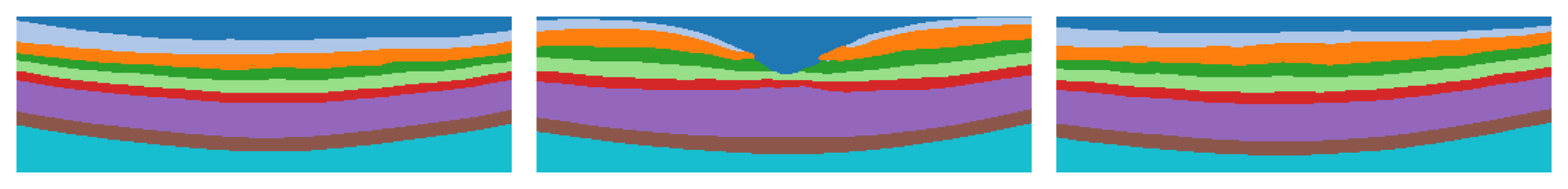}\hfill
		}		\end{subfigure}\\
		{\raisebox{10mm}{\begin{subfigure}[t]{1.6em}
					\caption[singlelinecheck=off]{}
			\end{subfigure}}\ignorespaces}
		\begin{subfigure}[t]{0.7\textwidth}
			\centering
			\scalebox{1}{\includegraphics[height=0.2\textwidth,width=1.4\textwidth]{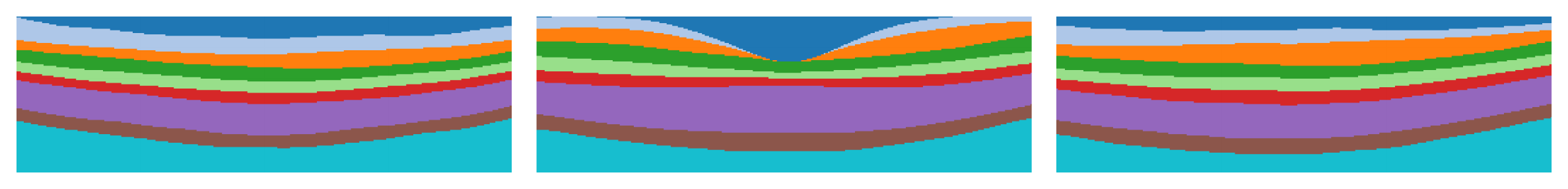}\hfill
		}		\end{subfigure}
	\end{adjustwidth}
	\captionsetup{font=footnotesize}
	\caption{Illustration of retina layer segmentation results listed in Table 
		\eqref{tab:Table_IOWA}. \textbf{Row (a)}:  Ground truth 
		labeling. 
		\textbf{Row (b-c)}: Labeled retina tissues using the proposed 
		approach 
		based on 
		covariance descriptors and CNN features,  respectively. \textbf{Row (d):} 
		The resulting 
		segmentation obtained using the IOWA reference algorithm. Visual comparison 
		with manually annotated retina layers (a) shows that segmentation with OAF (B) 
		leads to more reliable layer thickness of regions concentrated near the 
		fovea region (middle column), as opposed to IOWA reference method.}
	\label{Labeling_IOWA}
\end{figure}  

 On the contrary, the decrease of effective detection regarding 
the remaining layers reflects the present high texture variability, which can by 
tackled by increasing the size of prototypes on the cost of efficient 
computation of the generalized likelihood \eqref{eq:ordering_likelihood}. In 
general, the more robust retina detection features from an CNN net can be 
attributed to the underlying manifold geometry of symmetric positive definite 
matrices where the data partition is performed linearly by hyperplanes. This 
further indicates the nonlinear structure of the acquired volumetric OCT data. 
Figure \ref{Labeling_11_Layer} presents typical labelings of a B-scan for 
different locations in the segmented healthy OCT-volume obtained with 
the proposed approach. Direct comparison with the ground truth, as depicted in 
row (b), demonstrate higher accuracy and smoother boundary transitions by using 
CNN features instead of covariance descriptors. 
In particular, for the challenging 
segmentation of the ganglion cell layer (GCL) with a typical thinning near the 
macular region (middle scan), we report a Dice index of \mum{0.8373 +- 0.0263}
as opposed to the result \mum{0.6657 +- 0.1909}. The remaining numerical experiments are 
focusing on the validation of OAF against the retina segmentation methods serving as reference, as
summarized in Section 
\ref{sec:Reference_Methods}.

To access a quantitative direct comparison with the IOWA reference algorithm,  
the tested OCT volumes were imported into OCTExplorer 3.8.0 and segmented 
using the predefined Macular-OCT IOWA software after properly adjusting the 
resolution parameters. Additionally, we preprocessed each volume by removing 2 
B-scans from each side to get rid of boundary artifacts. We calculated and 
compared the segmentation results for layers consistent with available OCT 
volumes which were augmented by a medical expert. As before, we use the mean 
average error and the Dice index after segmenting the 8 volumes with no 
observable intraretinal diseases which are reported in Table 
\eqref{tab:Table_IOWA}. Figure \ref{fig:IOWA_Boxplot} provides a statistical 
illustration of the Dice index which reveals the high performance accuracy for 
methods which is in accordance with the mean average error shown in the last 
row of Figure \ref{fig:Cov_Deep_Barplot}. In particular, we observe a notable 
increase of performance using the OAF for detection of the ganglion cell layer 
with overall accuracy of \mum{0.8546 +- 0.0281}, see Figure \ref{Labeling_IOWA} for visualized segmentations of 3 B-scans.  

\begin{figure}[ht!]
	\begin{adjustwidth}{-4.5cm}{}
		\centering
		{\raisebox{10mm}{\begin{subfigure}[t]{1.6em}
					\caption[singlelinecheck=off]{}
			\end{subfigure}}\ignorespaces}
		\begin{subfigure}[t]{0.7\textwidth}
			\centering
			\scalebox{1}{\includegraphics[height=0.2\textwidth,width=1.4\textwidth]{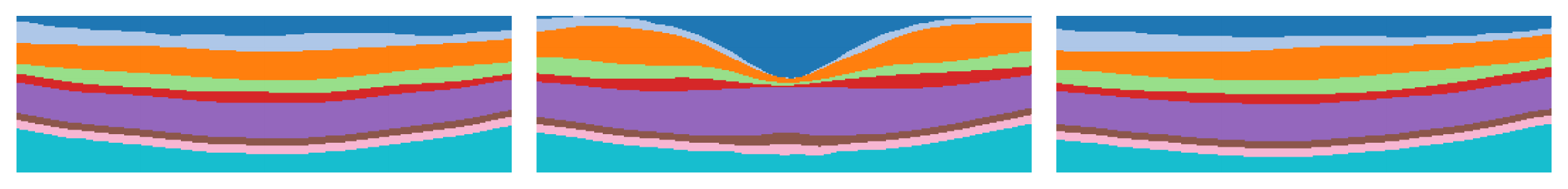}
		}		\end{subfigure}\\
		{\raisebox{10mm}{\begin{subfigure}[t]{1.6em}
					\caption[singlelinecheck=off]{}
			\end{subfigure}}\ignorespaces}
		\begin{subfigure}[t]{0.7\textwidth}
			\centering
			\scalebox{1}{\includegraphics[height=0.2\textwidth,width=1.4\textwidth]{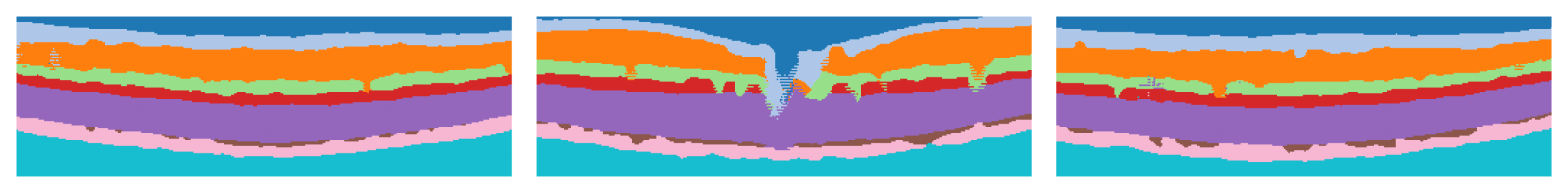}\hfill
		}		\end{subfigure}\\
		{\raisebox{10mm}{\begin{subfigure}[t]{1.6em}
					\caption[singlelinecheck=off]{}
			\end{subfigure}}\ignorespaces}
		\begin{subfigure}[t]{0.7\textwidth}
			\centering
			\scalebox{1}{\includegraphics[height=0.2\textwidth,width=1.4\textwidth]{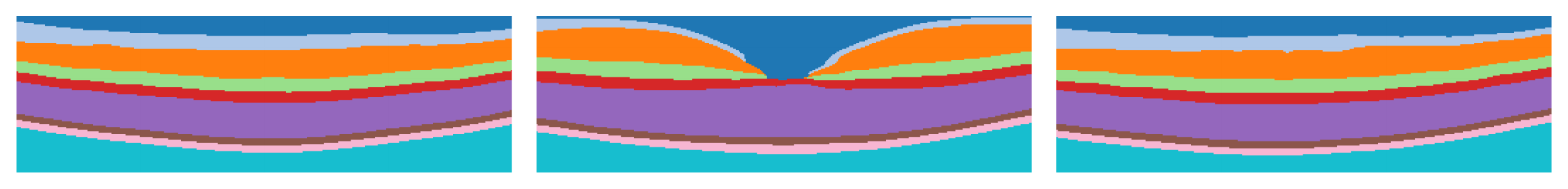}\hfill
		}		\end{subfigure}\\
		{\raisebox{10mm}{\begin{subfigure}[t]{1.6em}
					\caption[singlelinecheck=off]{}
			\end{subfigure}}\ignorespaces}
		\begin{subfigure}[t]{0.7\textwidth}
			\centering
			\scalebox{1}{\includegraphics[height=0.2\textwidth,width=1.4\textwidth]{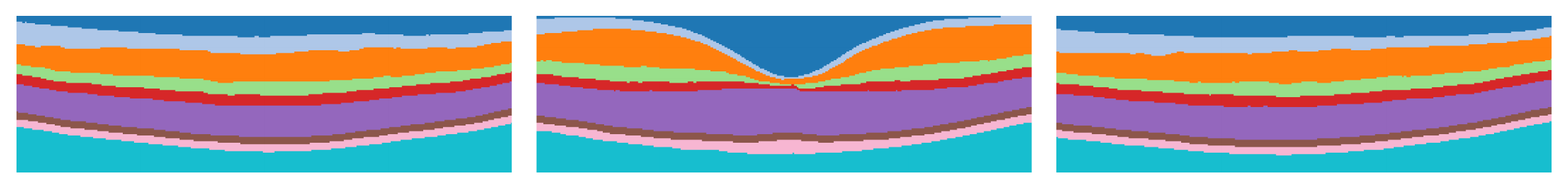}\hfill
		}		\end{subfigure}
	\end{adjustwidth}
	\captionsetup{font=footnotesize}
	\caption{\textbf{From top to bottom: Row (a)}: Ground truth for  
		the augmented retina 
		layer corresponding to Table \ref{tab:Table_Rathke}. \textbf{Row (b) 
			and (c)}: 
		Segmentation results of the OAF based on manifold valued features and on CNN features, respectively. 
		\textbf{Row (d)}: Segmentation results achieved by the 
		probabilistic 
		graphical model approach \cite{Rathke2014}. Both methods provide extraordinary performance for flat retina 
		detection whereas our method is more accurate regarding the 
		photoreceptor layers (below PR1).}
	\label{Labeling_Rathke}
\end{figure}

Next, we provide a visual and statistical comparison  of the 
proposed approach and the probabilistic state of the art retina segmentation 
approach \cite{Rathke2014} underlying Eq.~\eqref{Graphical_Model}. As before, to achieve a direct 
comparison with the proposed approach, we first adopted the OCT 
volumes to match the shape and parameters given in \cite{Rathke2014}. 
Subsequently, we removed the boundary between the GCL and IPL layers to obtain one 
characteristic layer which has to be detected. Figure \ref{Labeling_Rathke} 
displays the 
labeling accuracy.  
Both methods perform well by accurately segmenting flat shaped retina tissues, as shown in the first and last columns. However, closer inspection 
of the second column reveals a more 
accurate detection of layer thickness for the (PR2-RPE) and (INL) regions below 
the concave curved fovea region by using OAF(B). This is mainly due to the 
connectivity 
constraints imposed on boundary detection in \cite{Rathke2014}. By contrast, 
our method is capable to deal with rapidly decreasing layer thickness near the 
fovea region, as observed for  GCL and IPL layers in the middle column of 
Figure \ref{Labeling_Rathke} after visual comparison against the manual 
delineations (first row). This observation is supported by missing 
connectivity constraints for retina boundaries of the proposed method, as 
opposed to the Gaussian shape prior model in \cite{Rathke2014}, and only relying on 
layer ordering entirely included within the Fisher Rao geometry that underlies the 
assignment manifold. Therefore, the OAF 
is amenable for extension to layer detection on pathological volumes with 
vanishing retina boundaries, as for vitreomacular traction or diabetic macular 
edema. Figure \ref{fig:Exp_Surfaces_3D} additionally provides a 3D view on 
detected retina surfaces for each evaluated reference method used in this 
publication. The corresponding performance measures given in Table \ref{tab:Table_Rathke} underpin the  notably higher Dice similarity for 
(PR2-RPE) and (INL) layers with overall accuracy \SI{0.8606 \pm 
0.0706}{\micro\meter} and \SI{0.8690 \pm 0.0396}{\micro\meter}, respectively. 
The statistical plots for the mean 
average error and the Dice similarity index are given in Figures 
\eqref{fig:Rathke_Boxplot} and \eqref{fig:Cov_Deep_Barplot}, clearly showing the 
superiority of OAF (B) with both the higher Dice index and the mean average errors for all 
layers. In particular, following Table \ref{tab:Table_Rathke}, small error 
rates are observed among all the segmented layers, except for the 
(ILM) boundary which is detected by all methods with  high accuracy.  
We point out that in general our method is not limited to any number of  
segmented layers, if ground 
truth is available. 

Concluding the validation, both methods accurately detect the RNFL layer width 
whereas for the remaining retina tissues the layer extraction with the ordered 
assignment flow indicates the smallest mean absolute 
error supported by the highest Dice similarity index. This demonstrates the 
superior performance of order preserving labeling regarding accuracy and 
robustness, in view of segmenting the retina layer for classifying 
volumetric OCT data.

\begin{figure}[h] 
	{\raisebox{10mm}{\begin{subfigure}[t]{1.6em}
				\caption[singlelinecheck=off]{}
		\end{subfigure}}\ignorespaces}
	\begin{subfigure}[t]{0.32\textwidth}
		\centering
		\includegraphics[width=5cm,height=3cm]{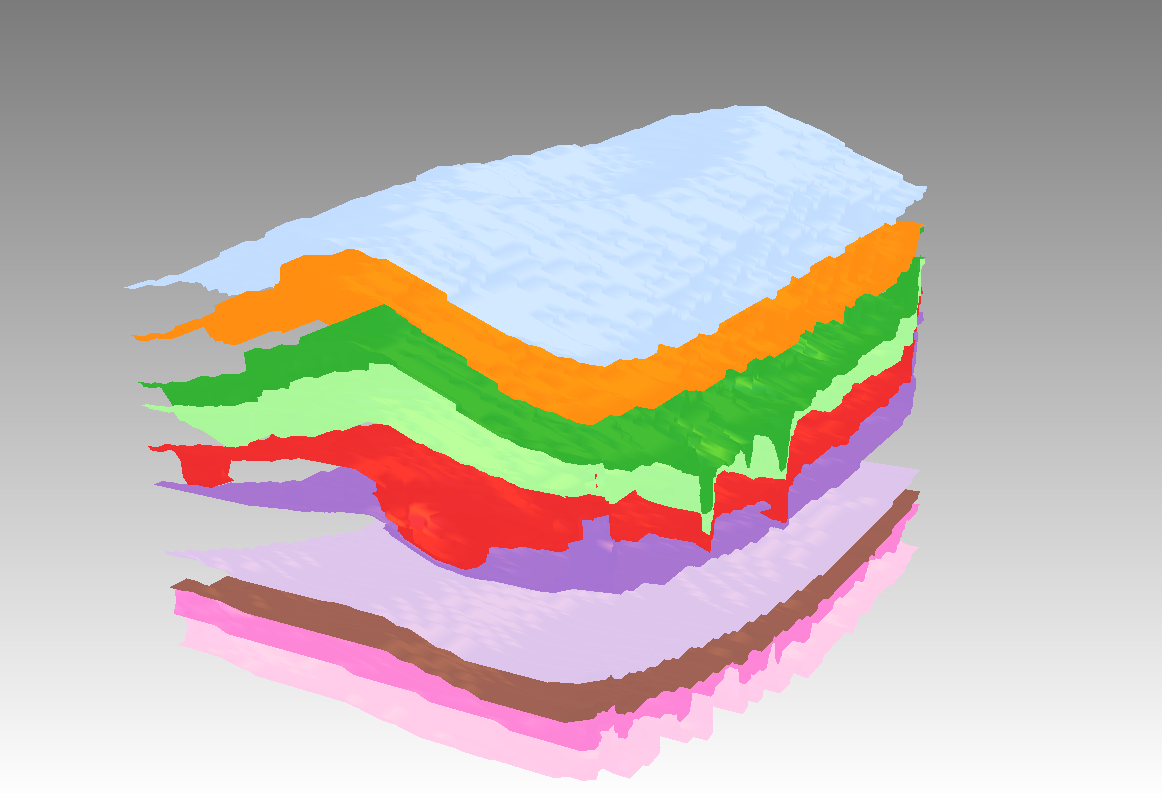}
	\end{subfigure}
	\begin{subfigure}[t]{0.31\textwidth}
		\centering
		\includegraphics[width=5cm,height=3cm]{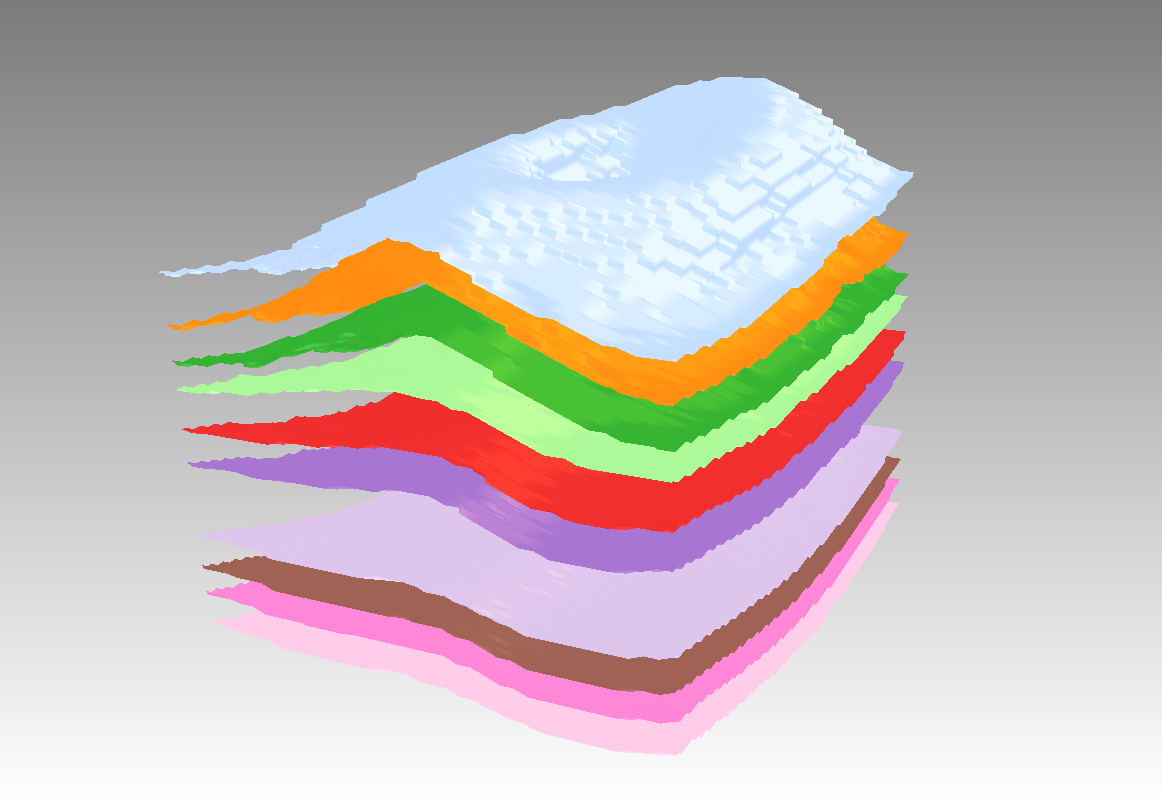}
	\end{subfigure}
	\begin{subfigure}[t]{0.31\textwidth}
		\centering
		\includegraphics[width=5cm,height=3cm]{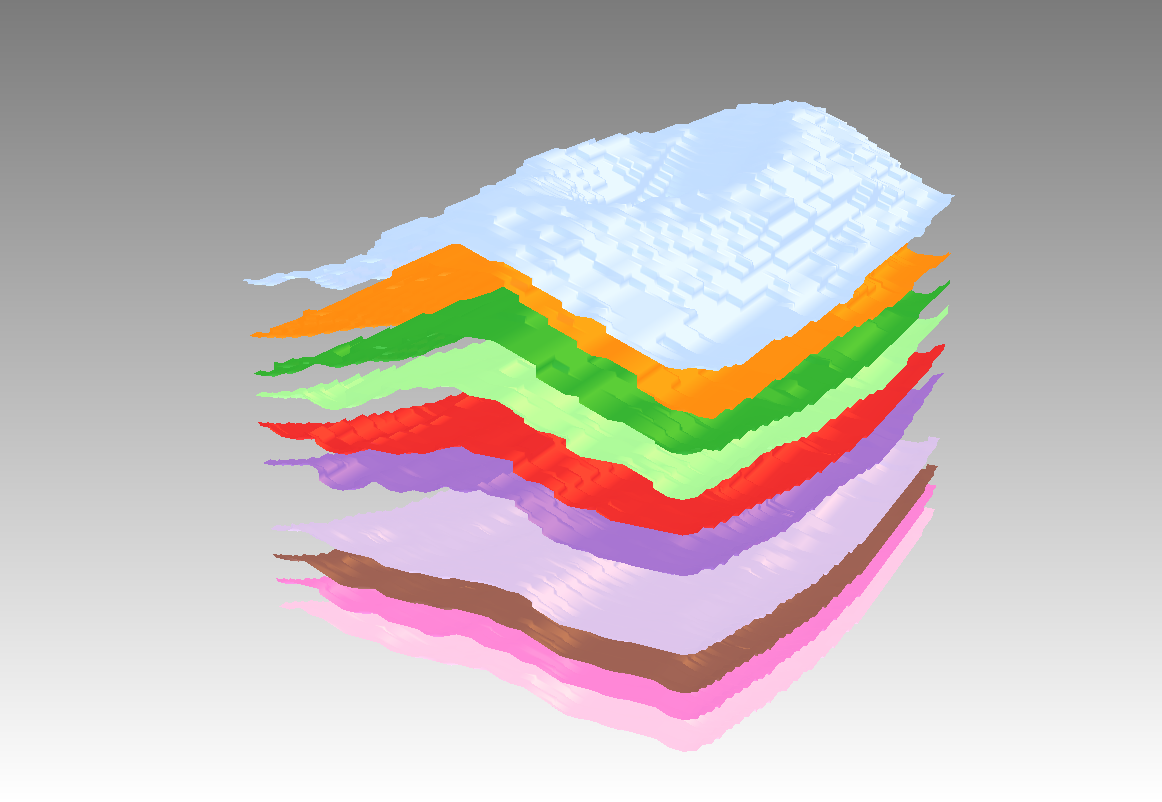}
		%\caption{Ground truth segmentation}	
	\end{subfigure}\\
{\raisebox{10mm}{\begin{subfigure}[t]{1.6em}
		\caption[singlelinecheck=off]{}
\end{subfigure}}\ignorespaces}
	\begin{subfigure}[t]{0.32\textwidth}
		\centering
		\includegraphics[width=5cm,height=3cm]{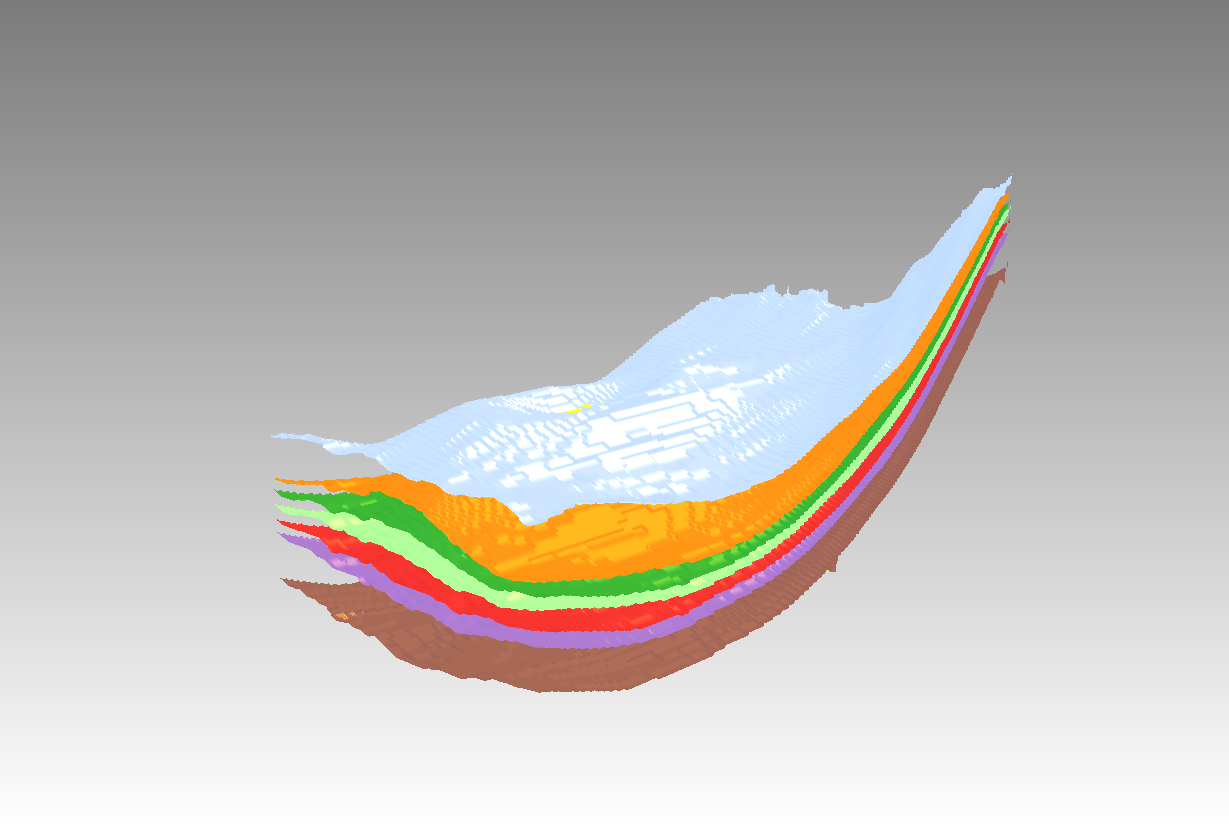}
	\end{subfigure}
	\begin{subfigure}[t]{0.31\textwidth}
	\centering
	\includegraphics[width=5cm,height=3cm]{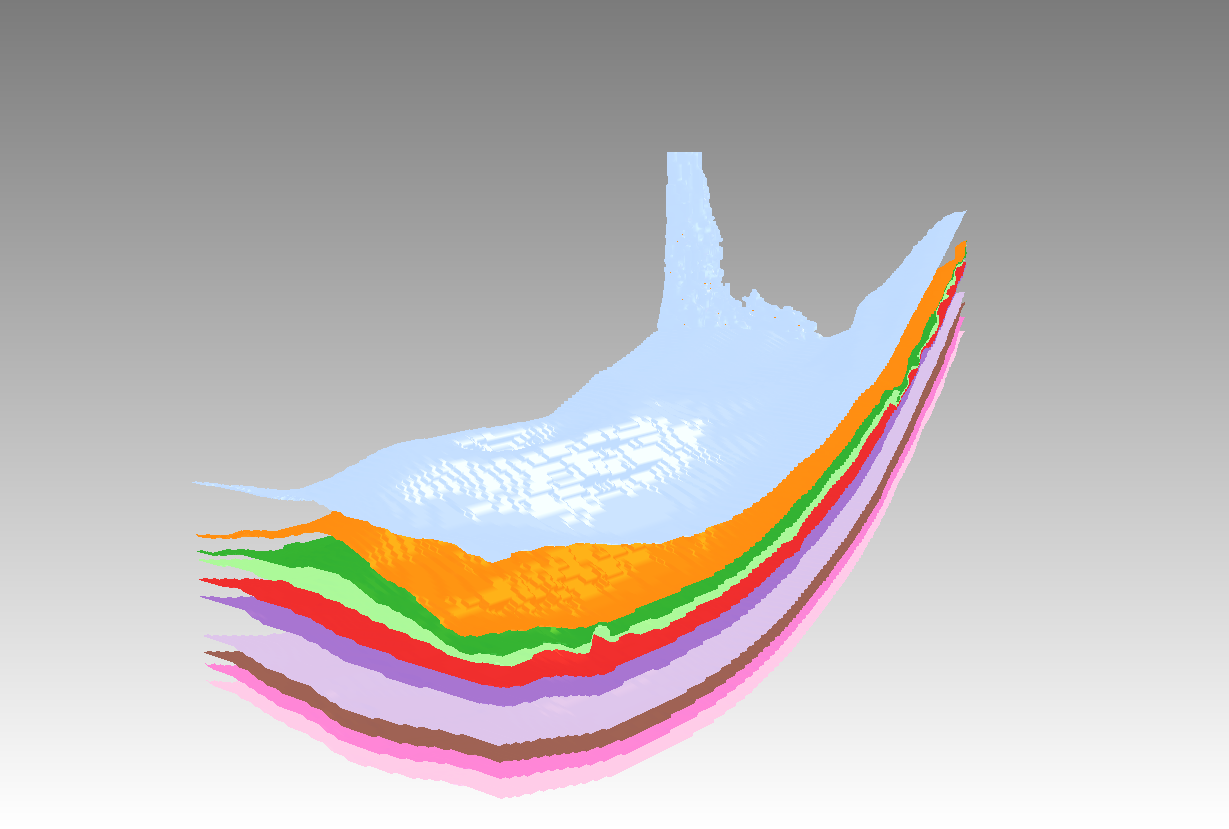}
\end{subfigure}
	\begin{subfigure}[t]{0.31\textwidth}
	\centering
	\includegraphics[width=5cm,height=3cm]{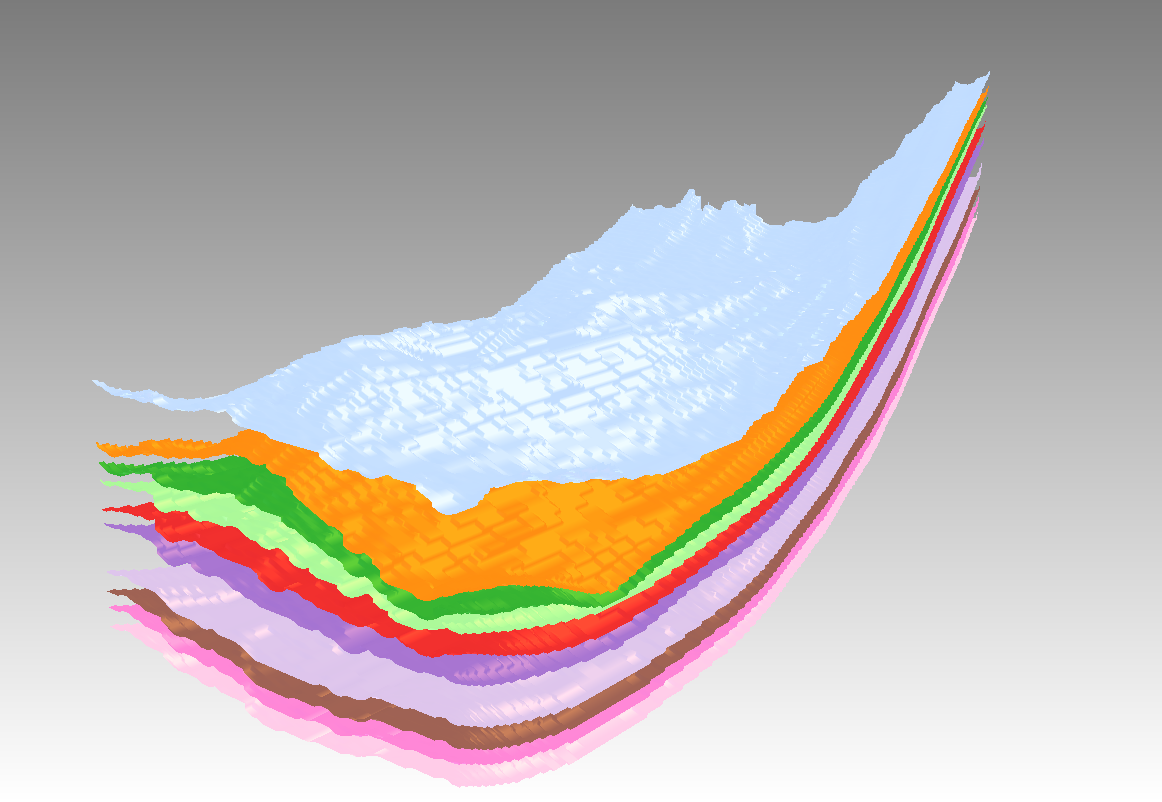}
\end{subfigure}
\\
{\raisebox{10mm}{\begin{subfigure}[t]{1.6em}
			\caption[singlelinecheck=off]{}
	\end{subfigure}}\ignorespaces}
\begin{subfigure}[t]{0.32\textwidth}
	\centering
	\includegraphics[width=5cm,height=3cm]{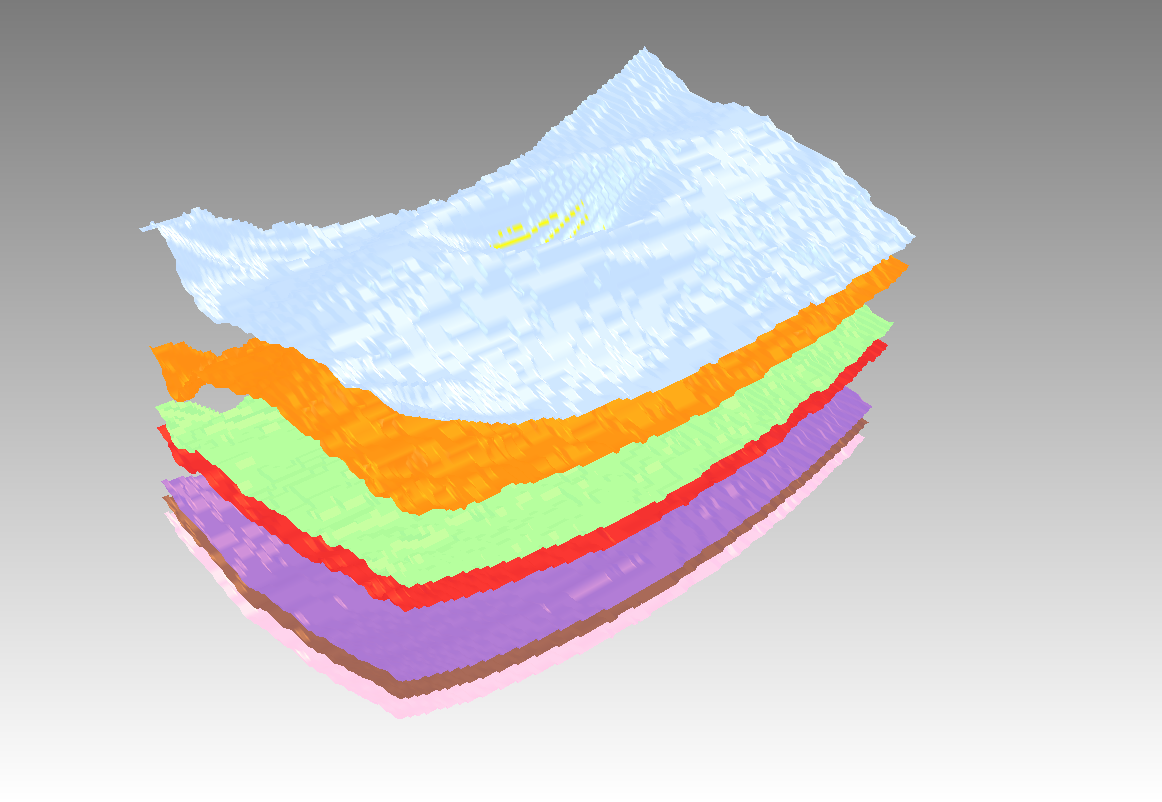}
\end{subfigure}
\begin{subfigure}[t]{0.31\textwidth}
	\centering
	\includegraphics[width=5cm,height=3cm]{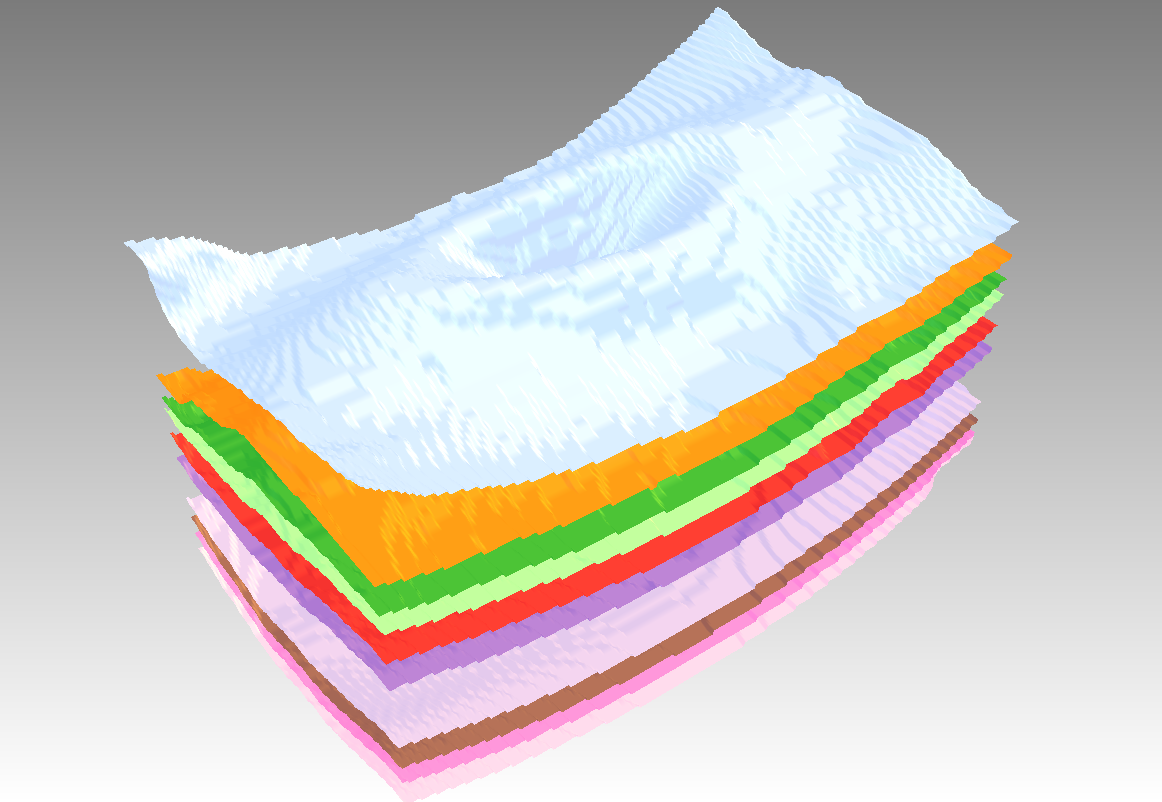}
\end{subfigure}
\begin{subfigure}[t]{0.31\textwidth}
	\centering
	\includegraphics[width=5cm,height=3cm]{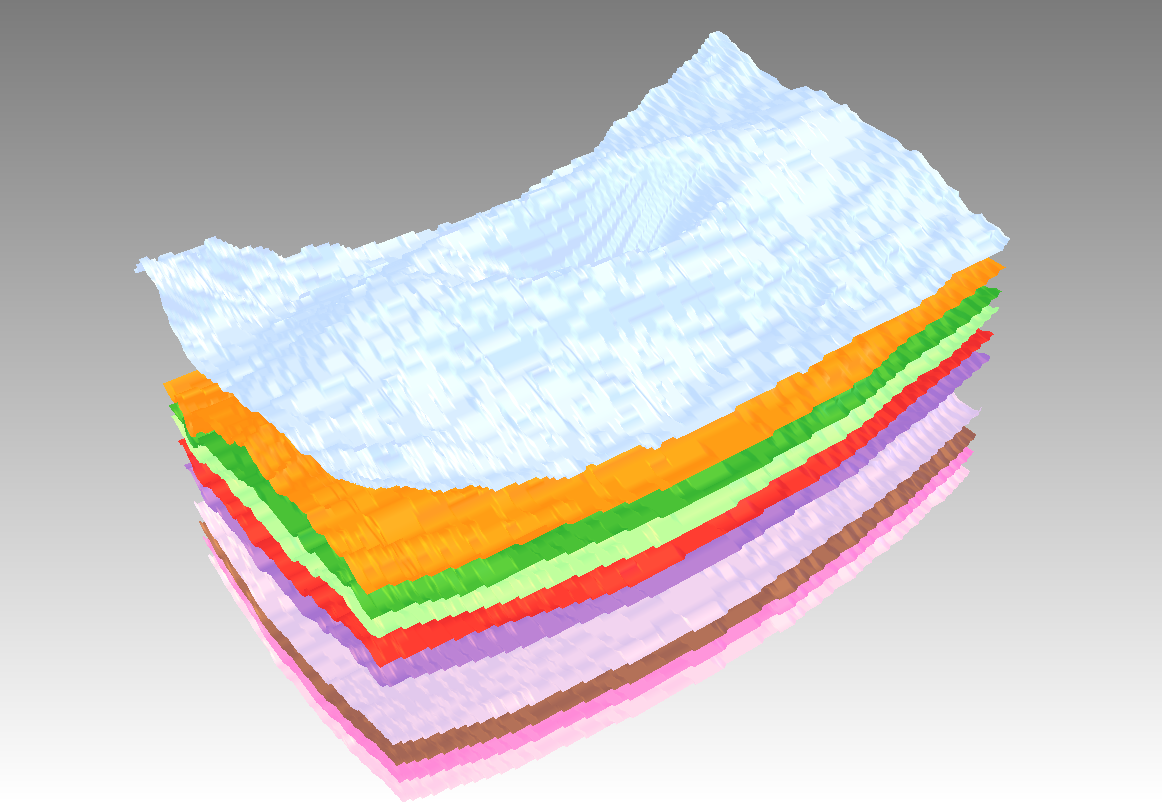}
\end{subfigure}
	\captionsetup{font=footnotesize}
	\caption{\textbf{Row (a)}: \textbf{From left to right:} 3D retinal surfaces 
	determined  
	using OAF (A) (left column) and OAF (B) (middle column). The 
	last column depicts ground truth. \textbf{Row (b)}: \textbf{From 
	left to 
	right:} Segmentation of retinal tissues with the  IOWA reference algorithm (left column) 
	with the proposed approach (middle column). \textbf{Row (c):} Visual 
	comparison of the probabilistic method  \cite{Rathke2014} (left column)  
	left and the OAF (B) (middle column). Our approach OAF (B) leads 
	to accurate retina layer segmentation with smooth layer boundaries, as 
	observed in the middle column.} 
	\label{fig:Exp_Surfaces_3D}
\end{figure}

\clearpage 
\begin{figure}[htb!]
	\begin{adjustwidth}{4.9em}{1em} 
		\centering
		\begin{subfigure}[t]{0.35\textwidth}
			\scalebox{.6}{\begin{tikzpicture}
		\begin{axis}[
		height=8cm,width = 10cm,
		ytick={0.5,1,...,3.5},xmin=0, xmax=7,
		xticklabels={0,0,RNFL,GCL,IPL,INL,OPL,ONL},
		ylabel=Mean average error,
		xlabel= {Layers},
		legend style={
			at={(0,1)},
			anchor=north west, ymax = 3.5,ymin = 0
		},
		nodes near coords,
		every node near coord/.append style={font=\tiny},
		ybar=0pt,
		bar width=11pt
		]
		\addplot[black!20!black,fill=Mea_Deep_Color!80!Mea_Deep_Color]
		coordinates {(01,1.3591) (02,2.5423)
			(03,3.0182) (04,2.61605) (05,1.60803) (06,1.6342) };
		\addplot[black!20!black,fill=Mea_Cov_Color!80!Mea_Cov_Color]
		coordinates {(01,0.8856) (02,1.47666)
			(03,1.60818) (04,1.50036) (05,1.62195) (06,1.8853)}; 
		\legend{OAF Cov., OAF CNN}
		\end{axis}
\end{tikzpicture}}
		\end{subfigure}
		\hfill
		\begin{subfigure}[t]{0.5\textwidth}
			\scalebox{.6}{\begin{tikzpicture}
\begin{axis}[
      height=8cm,width = 10cm,
      ytick={0.5,1,...,3.5},
      xlabel= {Layers},xmin=0, xmax=5,
      xticklabels={0,0,PR1,PR2,PR2-RPE,RPE-CS},
      legend style={
      	at={(0,1)},
      	anchor=north west,
      },ymax = 3.5,ymin = 0,
      nodes near coords,
      every node near coord/.append style={font=\tiny},
      ybar=0pt,
      bar width=11pt
      ]
      \addplot[black!20!black,fill=Mea_Deep_Color!80!Mea_Deep_Color]
        coordinates {(01,0.6995) (02,0.63196)
          (03,1.72436) (04,2.1354)};
      \addplot[black!20!black,fill=Mea_Cov_Color!80!Mea_Cov_Color]
        coordinates {(01,0.750033) (02,0.8458) (03,1.285) (04,2.8613)};
      \legend{Measured, Model}
    \end{axis}
\end{tikzpicture}}
		\end{subfigure}
	\begin{subfigure}[t]{0.35\textwidth}
		\scalebox{.6}{\begin{tikzpicture}
\begin{axis}[
      height=8cm,width = 10cm,
      ytick={0.5,1,...,6},
      ylabel=Mean average error,
      xlabel= {Layers},xmin=0, xmax=5,
      xticklabels={0,0,RNFL,GCL+IPL,INL,OPL},
      legend style={
      	at={(0,1)},
      	anchor=north west,
      },ymax = 6,ymin = 0,
      nodes near coords,
      every node near coord/.append style={font=\tiny},
      ybar=0pt,
      bar width=11pt
      ]
      \addplot[black!20!black,fill=Mea_Deep_Color!80!Mea_Deep_Color]
        coordinates {(01,1.3080) (02,2.9184)
          (03,5.1853) (04,4.8488)};
      \addplot[black!20!black,fill=Mea_Cov_Color!80!Mea_Cov_Color]
        coordinates {(01,2.026) (02,1.9267) (03,3.766) (04,3.701)};
      \legend{Rathke, OAF CNN}
    \end{axis}
\end{tikzpicture}}
	\end{subfigure}
	\hfill
	\begin{subfigure}[t]{0.5\textwidth}
		\scalebox{.6}{\begin{tikzpicture}
\begin{axis}[
      height=8cm,width = 10cm,
      ytick={0.5,1,...,8},
      xlabel= {Layers},xmin=0, xmax=4,
      xticklabels={0,0,,ONL+PR1, ,PR2,,PR2-RPE},
      legend style={
      	at={(0,1)},
      	anchor=north west,
      },ymax = 6,ymin = 0,
      nodes near coords,
      every node near coord/.append style={font=\tiny},
      ybar=0pt,
      bar width=11pt
      ]
      \addplot[black!20!black,fill=Mea_Deep_Color!80!Mea_Deep_Color]
        coordinates {(01,4.139) (02,5.7280)
          (03,5.276)};
      \addplot[black!20!black,fill=Mea_Cov_Color!80!Mea_Cov_Color]
        coordinates {(01,2.7209) (02,3.2814) (03,4.6269)};
      \legend{Rathke, OAF CNN}
    \end{axis}
\end{tikzpicture}}
	\end{subfigure}
	\begin{subfigure}[t]{0.35\textwidth}
	\scalebox{.6}{\begin{tikzpicture}
\begin{axis}[
      height=8cm,width = 10cm,
      ytick={0.5,1,...,6},
      ylabel=Mean average error,
      xlabel= {Layers},xmin=0, xmax=6,
      xticklabels={0,0,RNFL,GCL,IPL,INL,OPL},
      legend style={
      	at={(0,1)},
      	anchor=north west,
      },ymax = 6,ymin = 0,
      nodes near coords,
      every node near coord/.append style={font=\tiny},
      ybar=0pt,
      bar width=11pt
      ]
      \addplot[black!20!black,fill=Mea_Deep_Color!80!Mea_Deep_Color]
        coordinates {(01,2.7799) (02,2.0561) (03,3.1970) (04,2.7583) 
        (05,3.0330)};
    \addplot[black!20!black,fill=Mea_Cov_Color!80!Mea_Cov_Color]
    coordinates {(01,2.8470) (02,1.9683)
    	(03,3.400) (04,3.2094) (05,3.1217)};
      \legend{IOWA, OAF CNN}
    \end{axis}
\end{tikzpicture}}
\end{subfigure}
\hfill
\begin{subfigure}[t]{0.5\textwidth}
\scalebox{.6}{\begin{tikzpicture}
\begin{axis}[
      height=8cm,width = 10cm,
      ytick={0.5,1,...,10},
      xlabel= {Layers},xmin=0, xmax=4,
      xticklabels={0, , ,ONL+ELM+PR1, , ,PR2-RPE},
      legend style={
      	at={(0,1)},
      	anchor=north west,
      },ymax = 8,ymin = 0,
      nodes near coords,
      every node near coord/.append style={font=\tiny},
      ybar=0pt,
      bar width=11pt
      ]
      \addplot[black!20!black,fill=Mea_Deep_Color!80!Mea_Deep_Color]
        coordinates {(01,4.4292)
          (2.5,7.3738)};
      \addplot[black!20!black,fill=Mea_Cov_Color!80!Mea_Cov_Color]
        coordinates {(01,2.3843) (2.5,3.2300)};
      \legend{IOWA, OAF CNN}
    \end{axis}
\end{tikzpicture}
       }
\end{subfigure}
	\end{adjustwidth}
\captionsetup{font=footnotesize}
\caption{Performance measures per layer in terms of  the mean average error based on the segmentation of 10 healthy OCT volumes. \textbf{Top row}: Error bars 
for retina layers separated by 
the external limiting membrane (ELM) corresponding to  OAF (A) and OAF (B). 
\textbf{Middle row}: Comparison of the mean errors of 
OAF (B) and the probabilistic method \cite{Rathke2014}. \textbf{Bottom row}: 
Comparison of mean average errors of OAF (B) and the  
the IOWA reference algorithm.}
\label{fig:Cov_Deep_Barplot}
\end{figure}
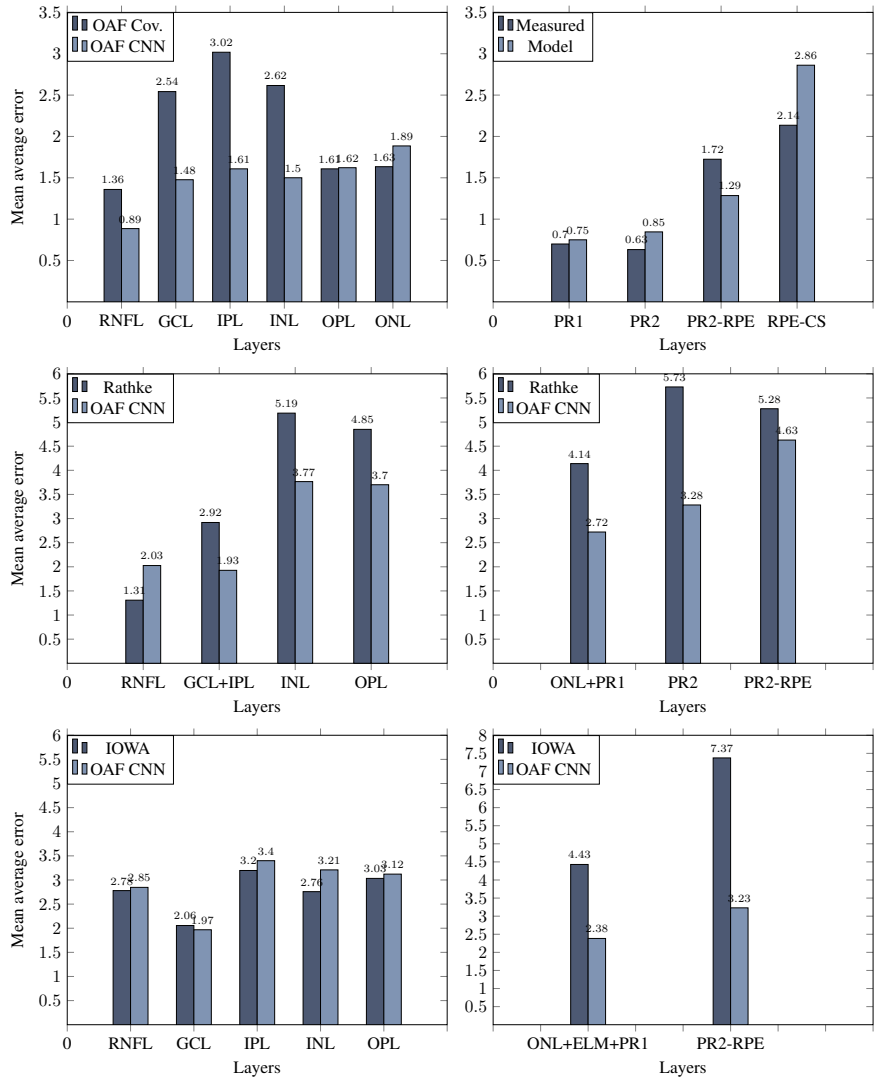
 \clearpage
\footnotesize

\begin{table}[]
\footnotesize{
\centering
\begin{tabular}{@{}lll@{}}
\toprule
OAF Cov & DICE index & Mean absolute error\\ \midrule
ILM     & $0.8837 \pm 0.2564$ &  $-$ \\
RNFL    & $0.6963 \pm 0.1998$ & $1.3590 \pm 0.4114$ \\ 
GCL     & $0.6657 \pm 0.1909$ & $2.5426 \pm 0.7819$ \\ 
IPL     & $0.5853 \pm 0.1773$ & $3.0183 \pm 1.0682$ \\ 
INL     & $0.6671 \pm 0.1773$ & $2.6160 \pm 1.1294$ \\ 
OPL     & $0.7018 \pm 0.2013$ & $1.6080 \pm 0.5120$ \\ 
ONL     & $0.8575 \pm 0.2523$ & $1.6342 \pm 0.7174$ \\
PR1     & $0.8199 \pm 0.2407$ & $0.6995 \pm 0.2467$ \\
PR2     & $0.6787 \pm 0.1976$ & $0.6320 \pm 0.2442$ \\
PR2-RPE & $0.6313 \pm 0.1821$ & $1.7244 \pm 0.6038$ \\
RPE-CS  & $0.8606 \pm 0.2469$ & $2.1354 \pm 1.0836$ \\ \bottomrule
\end{tabular}
\hspace{1em}
\begin{tabular}{@{}lll@{}}
\toprule
OAF CNN  & DICE index  & Mean absolute error\\ \midrule
ILM      & $0.9739 \pm 0.0189$ &  $-$ \\
RNFL     & $0.8842 \pm 0.0313$ & $0.8856 \pm 0.3513$  \\ 
GCL      & $0.8373 \pm 0.0263$ & $1.4767 \pm 0.5589$ \\ 
IPL      & $0.8151 \pm 0.0367$ & $1.6082 \pm 1.5291$ \\ 
INL      & $0.8414 \pm 0.0035$ & $1.5004 \pm 0.8652$  \\ 
OPL      & $0.8442 \pm 0.0437$ & $1.6220 \pm 1.0786$ \\ 
ONL      & $0.9254 \pm 0.0486$ & $1.8853 \pm 1.3951$ \\
PR1      & $0.8717 \pm 0.0441$ & $0.7500 \pm 0.3216$ \\
PR2      & $0.8330 \pm 0.0516$ & $0.8458 \pm 0.4914$ \\
PR2-RPE  & $0.8213 \pm 0.0835$ & $1.2850 \pm 1.3660$ \\
RPE-CS   & $0.9445 \pm 0.0488$ & $2.8613 \pm 2.5612$ \\ \bottomrule
\end{tabular}
\vspace{.5em}
\captionsetup{font=footnotesize}
\caption{Mean and standard deviations of the Dice index and mean absolute errors in pixels (1 pixel $=$ \mum{3.87}). \textbf{Left}: Errors of the proposed approach based on covariance descriptors (OAF (A)).
\textbf{Right}: Errors of the proposed approach based on CNN features (OAF (B)).}
\label{tab:Table}}
\end{table}

\begin{table}[]
\footnotesize{
\centering
\begin{tabular}{@{}lll@{}}
\toprule
\cite{Rathke2014} & DICE index & Mean absolute error\\ \midrule
ILM     & $0.9972 \pm 0.0006$ & $-$ \\
RNFL    & $0.8841 \pm 0.0125$ & $1.3080 \pm 0.6039$ \\ 
GCL+IPL & $0.8735 \pm 0.0152$ & $2.9180 \pm 1.0303$ \\ 
INL     & $0.7501 \pm 0.0292$ & $5.1853 \pm 1.3642$ \\ 
OPL     & $0.7651 \pm 0.0124$ & $4.8489 \pm 1.5898$ \\ 
ONL+PR1 & $0.9312 \pm 0.0068$ & $4.1490 \pm 1.2310$ \\ 
PR2     & $0.7416 \pm 0.0395$ & $5.7281 \pm 1.5411$ \\
PR2-RPE & $0.7945 \pm 0.0271$ & $5.2757 \pm 1.6452$ \\
RPE-CS  & $0.9858 \pm 0.0073$ & $-$ \\ \bottomrule
\end{tabular}
\hspace{.5em}
\begin{tabular}{@{}lll@{}}
\toprule
OAF CNN & DICE index  & Mean absolute error\\ \midrule
ILM     & $0.9953 \pm 0.0011$ & $-$ \\
RNFL    & $0.8954 \pm 0.0208$ & $2.0256 \pm 0.7660$ \\
GCL+IPL	& $0.9250 \pm 0.0180$ & $1.9267 \pm 0.7975$ \\
INL     & $0.8690 \pm 0.0396$ & $3.7660 \pm 2.3101$ \\
OPL     & $0.8680 \pm 0.0048$ & $3.7010 \pm 2.3561$ \\
ONL+PR1 & $0.9485 \pm 0.0622$ & $2.7209 \pm 2.3594$ \\
PR2     & $0.8647 \pm 0.0592$ & $3.2810 \pm 2.0854$ \\
PR2+RPE & $0.8606 \pm 0.0706$ & $4.6270 \pm 3.1891$ \\			
RPE-CS  & $0.9743 \pm 0.0484$ & $-$ \\ \bottomrule
\end{tabular}
\vspace{1em}
\captionsetup{font=footnotesize}
\caption{Mean and standard deviations of the Dice index and mean absolute errors in pixels (1 pixel $=$ \mum{3.87}).  \textbf{Left}: Errors of the probabilistic approach \cite{Rathke2014}. \textbf{Right}: Errors of the proposed approach OAF (B). These numbers demonstrate the superior performance of our novel order-preserving labeling approach.}
\label{tab:Table_Rathke}}
\end{table}
\begin{table}[]
	\footnotesize{
\centering
\begin{tabular}{@{}lll@{}}
\toprule
IOWA & DICE index & MAE\\ \midrule
ILM         & $0.9837 \pm 0.0043$ & $-$ \\
RNFL        & $0.8323 \pm 0.0236$ & $2.7799 \pm 0.9485$ \\ 
GCL         & $0.7757 \pm 0.0334$ & $2.0561 \pm 0.4978$ \\ 
IPL         & $0.7860 \pm 0.0189$ & $3.1970 \pm 1.1408$ \\ 
INL         & $0.8434 \pm 0.0269$ & $2.7583 \pm 1.3776$ \\ 
OPL         & $0.8024 \pm 0.0311$ & $3.0330 \pm 1.2837$ \\ 
ONL+ELM+PR1 & $0.8893 \pm 0.0182$ & $4.4292 \pm 1.5052$ \\
PR2-RPE     & $0.7120 \pm 0.0756$ & $7.3738 \pm 3.2031$ \\
RPE-CS      & $0.9667 \pm 0.0167$ & $-$  \\ \bottomrule
\end{tabular}
\hspace{.5em}
\begin{tabular}{@{}lll@{}}
\toprule
OAF CNN & DICE index  & MAE\\ \midrule
ILM         & $0.9795 \pm 0.0130$ & $-$ \\
RNFL        & $0.8717 \pm 0.0277$ & $2.8470 \pm 1.0758$ \\ 
GCL         & $0.8546 \pm 0.0281$ & $1.9683 \pm 0.6678$ \\ 
IPL         & $0.8370 \pm 0.0313$ & $3.4000 \pm 1.5535$ \\ 
INL         & $0.8587 \pm 0.0320$ & $3.2094 \pm 1.6950$ \\ 
OPL         & $0.8613 \pm 0.0387$ & $3.1217 \pm 1.7875$ \\ 
ONL+ELM+PR1 & $0.9482 \pm 0.0465$ & $2.3842 \pm 1.6869$ \\
PR2+RPE     & $0.9021 \pm 0.0648$ & $3.2296 \pm 1.9627$ \\			
RPE-CS      & $0.9605 \pm 0.0362$ & $-$ \\ \bottomrule
\end{tabular}
\vspace{1em}
\captionsetup{font=footnotesize}
\caption{Mean and standard deviations of the Dice index and mean absolute errors in pixels (1 pixel $=$ \mum{3.87}). \textbf{Left}: Errors of the IOWA reference algorithm \cite{Li:2006aa}. \textbf{Right}: Errors of the proposed approach OAF (B). These numbers demonstrate the superior performance of 
our novel order-preserving labeling approach.}
\label{tab:Table_IOWA}}
\end{table}

\normalsize

% We should see good improvement in DICE from uniform averaging alone, but ordering combats MAE (the kind of artefacts that drive MAE are largely eliminated).

% Rathke can do subpixel accurate boundary positions in priciple.

% When comparing MAE, note that out approach does not directly predict boundary positions. Therefore they are extracted via a simple heuristic (highest pixel of layer i marks boundary) which leads to this performance measure being sensitive to some sorts of artefacts.

% In particular, some examples have < 1 pixel MAE.
% !TEX root =  ../OCT-3D-AFlow.tex
%%%%%%%%%%%%%%%%%%%%%%%%%%%%%%%%%%%

\section{Discussion}\label{sec:discussion}
We discuss additional aspects pertaining to the data used for training feature extractors
as well as the locality of extracted features.

\subsection{Ground Truth Generation}\label{sec:gt_discussion}
The training and evaluation of supervised models for feature extraction requires a sizeable amount 
of high-quality labeled ground truth data. This presents a commonly encountered challenge in 3D OCT 
segmentation \cite{Dufour:2013aa,Li:2006aa}, because the process of manually labeling every voxel of a 3D
volume is extremely laborious. The desire to account for inter-observer variability in manual 
segmentations further compounds this problem.
OCT volumes used for testing purposes in the present paper were initially segmented by an
automatic procedure based on hand-crafted features. In a subsequent step, each B-scan segmentation
was manually corrected by a medical practitioner.
The automatic method used for initial segmentation only explicitly regularizes on each
individual B-scan, leading to irregularity between consecutive B-scans (see Figure \ref{fig:3d_inconsistency}).
% Note that this effect is not present in all automatically segmented scans from HDE, some are clearly
% smoothed in 3D. Apparently, the method changed somewhere along the way.
\begin{figure}[h]
	\includegraphics[width=\textwidth,height=3.5cm]{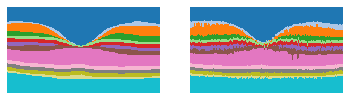}
	\captionsetup{font=footnotesize}
	\caption{\textbf{Left}: Initial automatic segmentation of individual B-scan 
	based on hand-crafted features.
		\textbf{Right}: Section of the same automatically segmented volume 
		orthogonal to each B-scan.}
	\label{fig:3d_inconsistency}
\end{figure}
Manual correction of initial automatic segmentations leads to a noticeable reduction of irregularity but does
not completely remove it. We therefore cannot rule out that a small bias towards the initial automatic segmentation
based on hand-crafted features may still be present in the ground truth segmentations that we used to quantify segmentation
performance of novel methods as well as baseline methods in this paper.
During feature extraction, deep learning models may be capable of discovering the specific hand-crafted features
used for initial automated segmentation which may in turn lead to exploitation of any bias towards them. In contrast, 
because the reference methods are not trained on the same data, they can not exploit any such bias, putting them at
a possible disadvantage.

%To avoid such limitations and to take the advantage 
%of our segmentation approach we recommend to smooth the ground truth data by 
%applying the assignment flow \eqref{eq:assignment-flow} before acquiring 
%prototypes from \eqref{Feature_Vector} or training the CNN network. 
% <- this is not the way current models were trained 

\subsection{Feature Locality}\label{sec:locality_discussion}
The ordered assignment flow segmentation approach can work with data from any metric space and is hence completely
agnostic to the choice of preliminary feature extraction method. 
In this paper, we chose to limit the field of view of deep networks such that 
features with local discriminative 
information are extracted. This 
makes empirical results directly comparable between features based on covariance descriptors and features extracted
by these networks.
In addition, we conjecture that local features may generalize better to unseen pathologies. Specifically,
if a pathological change in retinal appearance pertains to the global shape of cell layers, local features are
largely uneffected and therefore significant for detecting  
irregularly shaped retina boundaries on unhealthy OCT data. In this way, we 
expect segmentation performance to be relatively consistent on real-world data.
Conversely, widening the field of view in feature extraction should be accompanied by a well-considered training 
procedure in order to achieve similar generalization behavior, by employing 
e.g. extensive data augmentation. While raw OCT volume data has become 
relatively plentiful in clinical settings,
large volume datasets with high-quality gold-standard segmentation are not 
widely available at the time of writing. Therefore, by representing a given OCT 
scan \emph{locally} as 
opposed to incorporating global context at every stage, it is our next 
hypothesis that superior 
generalization can be achieved in the face of limited data availability. 
Similarly, although based on local features, the method proposed by 
\cite{Rathke2014}
combines local knowledge in accordance with a \emph{global} shape prior. This 
again limits the methods ability to generalize to unseen data if large 
deviation from the expected global shape seen in training is present.

% There is another point which could be made here: large field of view models can be used together with an ordered
% flow, but the AF regularization should work best if local assignment insecurity is properly represented in the
% distance data i.e. if the entropy of initial likelihoods L = exp_1(-D) is a meaningful measure for how clearly the
% data points to any label at each pixel. This may not be the case for all deep learning models with large
% field of view, because these models may make a singular, global decision about the location of layer boundaries and
% subsequently draw them into final prediction scores, leading to mostly uniform uncertainty on the whole image domain.
% The way to combat this (in my view) is to either train end-to-end, which we are still working on, or to
% try to modify the training procedure purely on the feature extraction side, which requires a new approach.

% !TEX root =  ../OCT-3D-AFlow.tex
%%%%%%%%%%%%%%%%%%%%%%%%%%%%%%%%%%%

\section{Conclusion}\label{sec:Conclusion}

In this paper we presented a novel, fully automated and purely data driven 
approach for retina segmentation in OCT-volumes. Compared to methods 
\cite{Li:2006aa} \cite{Dufour:2013aa} and \cite{Rathke2014} that have proven to 
be particularly effective on tissue classification with a priory known retina 
shape orientation, our ansatz merely relies on local features and yields 
ordered labelings which are directly enforced through the underlying geometry 
of statistical manifold \eqref{schnoerr-eq:def-mcW}. To address the task of 
leveraging 3D-texture information, we proposed two different feature selection 
processes by means of region covariance descriptors  \eqref{Cov_Des_Feature} and 
the output obtained by training a CNN network \eqref{sec:experiments_cnn}, 
which are both based on the interaction between local feature responses.

 As opposed to other machine learning methods developed for segmenting human 
 retina from volumetric OCT data, the proposed method only takes the pairwise 
 distance between voxels and prototypes \eqref{eq:def-mcF-ast} as input. As a 
direct consequence our approach can be applied in connection with broader range 
of features living in any metric space and additionally provides the 
incorporation of outputs from trained neuronal convolution networks interpreted 
as image features, where a particular instance of such type was demonstrated in 
Section \eqref{sec:experiments_cnn}. Even in view of the moderate result  
achieved after 
segmentation using OAF (A) in connection with covariance descriptors, we 
observe the 
importance of our automatic algorithm by its high level of regularization. 
Compared to the approach presented in \cite{Chiu:2015aa} which  
employs a higher number of input features but still requires postprocessing 
steps to yield order preserving labeling, our approach provides a way to perform 
this tasks simultaneously.

 Using locally 
adapted features for handling volumetric OCT data sets from patients 
with observable pathological retina changes is in particular valuable to 
suppress wrong layer boundaries predictions caused by prior assumptions on 
retinal layer thicknesses typically made by graphical model approaches as in 
\cite{Dufour:2013aa} and \cite{Song:2012aa}. Our 
method overcomes this limitation by mainly avoiding any bias towards using 
priors to global retina shape and instead only relies on the natural  
biological layer ordering, which is accomplished by restricting the 
assignment manifold to probabilities that satisfy the ordering constraint 
presented in Section \eqref{sec:Ordered-Segmentation}. The experimental results 
 reported in Section \ref{sec:Experimental-Results},  and the direct 
comparison to the state of the art segmentation techniques \cite{IOWA} and 
\cite{Rathke2014} by using common validation metrics, underpin a notable 
performance and robustness of the geometric  segmentation 
approach introduced in Section \ref{sec:Assignment-Flow}, that we extended to
order-preserving labeling in Section \ref{sec:Ordered-Segmentation}.
Furthermore, the results indicate that the ordered assignment flow successfully  
tackles problems in the field of retinal tissue classification on 3D-OCT data   
which are typically corrupted by speckle noise, with achieved performance 
comparable to manual graders which makes it to a method of choice for medical 
image applications and extensions therein. We point out that our approach 
consequently differs from common deep learning methods which 
explicitly aim to incorporate global context into the feature extraction 
process. In particular, throughout the 
experiments we observed higher regularization resulting in smoother transitions 
of layer boundaries along the B-scan acquisition axis similar to the effect in 
\cite{Rathke2014} where the used smooth global Gaussian prior leads to 
limitations for pathological applications.

 To reduce the reliance of manually segmented ground truth for extracting 
 dictionaries of prototypes, our method can easily be extended to unsupervised 
 scenarios in the context of \cite{Zisler:2020aa}. To deal with highly variable 
 layer boundaries another possible extension of our method is to predict 
 weights for geometric averaging 
\eqref{eq:weights-Omega-i} in an optimal control theoretic way, to cope with the 
linearized dynamics of the assignment flow \cite{Zeilmann:2020aa} as in detail 
elaborated in \cite{Huhnerbein:2020aa}. Consequently, by building on the 
feasible concept of spatially regularized 
 assignments \cite{Schnorr:2019aa}, the ordered flow \eqref{def:ordered_af} 
 possesses the potential to be extended towards the detection of 
pathological retina changes and vascular vessel structure. We expect that the 
joint interaction of retina tissues and blood vessels during the segmentation 
with the assignment flow will lead to more effective layer detection, which is the 
objective of our current research. 
% Furthermore, due to the coherent structure of 
%the assignment manifold where the segmentation task is driven by spatial 
%regularized decisions with initial data represented by 
%\eqref{eq:def-distance-vector}, allows to support the vessel tracking task 
%originating from OCTA volume by simultaneous layer segmentation for more 
%accurate extraction of 3D vascular networks avoiding artifact caused by shadow 
%effects.   

%%%%%%%%%%%%%%%%%%%%%%%%%%%%
\subsection*{Acknowledgement} We thank Dr.~Stefan Schmidt and Julian Weichsel 
for sharing with us their expertise on OCT sensors, data acquisition and 
processing. In addition, we thank Fred Hamprecht and Alberto Bailoni for their 
guidance in training deep networks for feature extraction from 3D data.\\ 

This work is supported by Deutsche Forschungsgemeinschaft (DFG) under Germany’s 
Excellence Strategy EXC-2181/1 - 390900948 (the Heidelberg STRUCTURES 
Excellence Cluster).

\section*{ORCID iDs}

Dmitrij Sitenko        \orcidaffil{0000-0002-0022-3891}

Bastian Boll           \orcidaffil{0000-0002-3490-3350}

Christoph Schn\"{o}rr  \orcidaffil{0000-0002-8999-2338}

%%%%%%%%%%%%%%%%%%%%%%%%%%%%
\bibliographystyle{amsalpha}
\bibliography{./OCT-3D-AFlow}
\newpage

\end{document}